\pdfoutput=1

\documentclass[12pt]{colt2016_mod}
\usepackage{times}

\usepackage{epsf}
\usepackage{verbatim}
\usepackage{makeidx}
\usepackage{graphicx}
\usepackage{latexsym}
\usepackage{amsbsy}

\newtheorem{assumption}{Assumption}

\newenvironment{myexampleplain}{\begin{example}\normalfont}{$\square$\end{example}}
\newenvironment{myexample}[1]{\begin{example}[#1]\normalfont}{$\square$\end{example}}

\newcommand{\cA}{\ensuremath{\mathcal A}}

\newcommand{\cF}{\ensuremath{\mathcal F}}

\newcommand{\cU}{\ensuremath{\mathcal U}}

\newcommand{\cX}{\ensuremath{\mathcal X}}
\newcommand{\cY}{\ensuremath{\mathcal Y}}
\newcommand{\cZ}{\ensuremath{\mathcal Z}}

\newcommand{\reals}{{\mathbb R}}

\newcommand{\naturals}{{\mathbb N}}

\newcommand{\Exp}{\ensuremath{\text{\rm E}}}

\newcommand{\kl}{\mathrm{KL}}

\newcommand{\loss}{\ensuremath{\text{\sc loss}}}

\newcommand{\volume}{\ensuremath{V_{\text{p}}}}

\usepackage{enumitem}
\usepackage{MnSymbol}
\DeclareSymbolFont{Symbols}{OMS}{cmsy}{m}{n}
\SetSymbolFont{Symbols}{bold}{OMS}{cmsy}{b}{n}
\DeclareMathSymbol{\Setminus}{\mathbin}{Symbols}{"6E}
\usepackage{xcolor}
\usepackage{hyperref}
\hypersetup{colorlinks=true}
\usepackage{nicefrac} 

\def\glossaryON{}

\ifdefined\glossaryON
\usepackage[makeindex,section,numberedsection=nameref]{glossaries}

\fi

\newcommand{\nocomparator}{}

\ifdefined\glossaryON
\newcommand{\glsaddcond}[1]{\glsadd{#1}}
\else
\newcommand{\glsaddcond}[1]{#1}
\newcommand{\glsadd}[1]{#1}
\fi

\newcommand{\estim}[1]{{\hat{#1}}}

\DeclareMathOperator*{\pipes}{\|}
\DeclareMathOperator*{\convhull}{conv}
\DeclareMathOperator*{\opmin}{\wedge}
\DeclareMathOperator*{\opmax}{\vee}
\DeclareMathOperator*{\sign}{sign}

\DeclareMathOperator*{\argmin}{arg\,min}

\newcommand{\stochleq}{\leqclosed}

\makeatletter
\DeclareRobustCommand{\qed}{%
  \ifmmode 
  \else \leavevmode\unskip\penalty9999 \hbox{}\nobreak\hfill
  \fi
  \quad\hbox{\qedsymbol}}
\newcommand{\qedsymbol}{\BlackBox}
\newenvironment{Proof}[1][\proofname]{\par
  \normalfont
  \topsep6\p@\@plus6\p@ \trivlist
  \item[\hskip\labelsep\bfseries
    #1]\ignorespaces
}{%
  \qed\endtrivlist
}
\newcommand{\proofname}{Proof}
\makeatother

\newcommand{\twopart}[1]{\ensuremath{\ddot{#1}_{\text{\sc 2-p}}}}

\newcommand{\hellp}[1]{{H}_{f,#1}}

\newcommand{\dhel}{{d}}

\newcommand{\ind}[1]{\mathop{{\bf 1}_{\{#1\}}}}

\newcommand{\nrenb}[1]{\ensuremath{D_{\eta \eta}}}

\newcommand{\Pt}{P} 
\newcommand{\E}{\operatorname{\mathbf{E}}}

\newcommand{\Expann}[1]{\Exp^{\textsc{ann}(#1)}}
\newcommand{\Exphel}[1]{\Exp^{\textsc{he}(#1)}}

\newcommand{\rv}[1]{\underline{#1}}
\newcommand{\funky}[1]{\rv{\tilde{\genaction}}}
\newcommand{\funkyc}[1]{\rv{\tilde{\lambda}}}

\newcommand{\fstar}{f^*}
\newcommand{\fopt}{\fstar}

\renewcommand{\loss}{\ensuremath{\ell}}
\newcommand{\mixloss}[2]{m^{#1}_{#2}}
\newcommand{\grip}[1]{\mixloss{#1}{\cF}}
\newcommand{\gripgen}[2]{\mixloss{#1}{#2}}
\newcommand{\gripbase}[1]{g^{#1}_{\cF}}
\newcommand{\gripbasegen}[2]{g^{#1}_{#2}}

\newcommand{\oldgriplosseta}{\loss_{g_\eta}}
\newcommand{\xslosslong}[1]{\loss_{#1} - \loss_{\fopt}}
\newcommand{\xslossatlong}[2]{\loss_{#1}(#2) - \loss_{\fopt}(#2)}
\newcommand{\xsloss}[1]{L_{#1}}
\newcommand{\xslossat}[2]{L_{#1}(#2)}
\newcommand{\xsrisk}[1]{\E [ \xsloss{#1} ]} 
\newcommand{\xsrisklong}[1]{\E [ \xslosslong{#1} ]}

\newcommand{\dol}{\ensuremath{\Pi}} 
\newcommand{\dolest}{\dol_|} 
\newcommand{\Prior}{\ensuremath{\Pi_0}} 
\newcommand{\postdens}{\ensuremath{\pi}} 
\newcommand{\KL}{\text{\sc KL}}
\newcommand{\Hell}{\text{\sc H}}

\newcommand{\genaction}{\ensuremath{\theta}}

\newcommand{\allactionset}{\bar{\mathcal{F}}}

\newcommand{\smtuple}{(\Pt,\loss,\model)}

\renewcommand\Exp{\ensuremath{\mathbf E}}

\newcommand{\midb}{|}

\newcommand{\model}{\mathcal{F}}

\newcommand{\bound}{\textsc{bound}}

\newcommand{\rsc}{\mathrm{IC}} 

\renewcommand{\volume}{\mathrm{vol}}

\ifdefined\glossaryON

\newglossaryentry{general}{
  name={\textbf{General notation}},
  description={\nopostdesc},
  sort=A,
  nonumberlist
}

\newglossaryentry{Zn}
{
  name={$Z^n$},
  description={i.i.d. sample; \,\, $Z^n = (Z_1, Z_2, \ldots, Z_n) \sim P^n$},
  sort=A01
}

\newglossaryentry{P}
{
  name={$P$},
  description={Probability distribution over $\cZ$},
  sort=A02
}

\newglossaryentry{fhat}
{
  name={$\hat{f}$},
  description={Deterministic estimator or learning algorithm; \,\, $\hat{f} \equiv \hat{f}(Z^n)$},
  sort=A03
}

\newglossaryentry{learning-problem}
{
  name={$\smtuple$},
  description={Learning problem for distribution $P$, loss function $\loss$, and model $\cF$},
  sort=A04
}

\newglossaryentry{loss-of-f}
{
  name={$\loss_f$},
  description={Loss of hypothesis $f$; \,\, $\loss_f(z) \equiv \loss(f, z)$ and $\loss_f \equiv \loss_f(Z)$},
  sort=A05
}

\newglossaryentry{fopt}
{
  name={$\fopt$},
  description={Risk minimizer within $\cF$},
  sort=A07
}

\newglossaryentry{excess-loss}
{
  name={$\xsloss{f}$},
  description={Excess loss (w.r.t.~$\fopt$) of $f$; \,\, $\xslossat{f}{z} \equiv \xslossatlong{f}{z}$ and $\xsloss{f} \equiv \xslossat{f}{Z}$},
  sort=A08
}

\newglossaryentry{gen-bayes}
{
  name={$\Pi^B_n$ (and $\pi^B_n$)},
  description={$\eta$-generalized Bayesian posterior (and its density relative to $\Prior$)},
  sort=A10
}

\newglossaryentry{ESI}
{
  name={$\stochleq_{\eta}$},
  description={Exponential stochastic inequality (E.S.I.)},
  sort=A11
}

\newglossaryentry{dolest}
{
  name={$\dolest$},
  description={Randomized estimator or learning algorithm; \,\, $\dolest: \bigcup_{n=0}^\infty \cZ^n \rightarrow
\Delta(\model)$},
  sort=A12
}

\newglossaryentry{dol-n}
{
  name={$\dol_n$},
  description={Output of algorithm $\dolest$ based on sample $Z^n$; \,\, $\dol_n \equiv \dol \mid Z^n$},
  sort=A13
}

\newglossaryentry{prior}
{
  name={$\Prior$},
  description={Prior; \,\, $\dol_0 \equiv \dol \mid \{\}$},
  sort=A14
}

\newglossaryentry{mu}
{
  name={$\mu$},
  description={Common dominating measure for $\{p_f\}_{f \in \cF}$ in the case of log loss},
  sort=A15
}

\newglossaryentry{gen-two-part}
{
  name={$\twopart{f}$},
  description={$\eta$-generalized two-part MDL estimator for prior $\Prior$ at sample size $n$},
  sort=A18
}

\newglossaryentry{det-rand}
{
  name={$(\estim{f},\Prior)$},
  description={Deterministic estimator $\estim{f}$ viewed as randomized estimator},
  sort=A19
}

\newglossaryentry{hellinger-E}
{
  name={$\Exphel{\eta} \left[U \right] $},
  description={Hellinger-transformed expectation; \,\, $\Exphel{\eta} \left[U \right] 
= \frac{1}{\eta} \left(1- \E \left[e^{-\eta U} \right]\right)$},
  sort=A20
}

\newglossaryentry{annealed-E}
{
  name={$\Expann{\eta} \left[U \right]$},
  description={Annealed expectation; \,\, $\Expann{\eta} \left[U \right] = -\frac{1}{\eta} \log \E \left[e^{-\eta U} \right]$},
  sort=A21
}

\newglossaryentry{IC}
{
  name={$\rsc_{n,\eta}(\nocomparator \dolest)$},
  description={Information complexity},
  sort=A22
}

\newglossaryentry{entropified}
{
  name={$p_{f,\eta}$},
  description={entropified loss; \,\, $p_{f,\eta}(z) = p(z) \frac{\exp(- \eta \xsloss{f}(z))}{\E[\exp(-\eta \xsloss{f}(Z))]}$},
  sort=A23
}

\newglossaryentry{mm}
{
  name={$\dhel_{\bar\eta}(\cdot,\cdot)$},
  description={misspecification metric},
  sort=A24
}

\newglossaryentry{covering-number}
{
  name={$\mathcal{N}(\mathcal{A}, \|\cdot\|, \epsilon)$},
  description={$\varepsilon$-covering number of $(\mathcal{A}, \|\cdot\|)$},
  sort=A240
}

\newglossaryentry{divergences}{
  name={\textbf{Divergences}},
  description={\nopostdesc},
  sort=A241,
  nonumberlist
}

\newglossaryentry{standard-KL}
{
  name={$\kl(\cdot \pipes \cdot)$},
  description={Standard Kullback-Leibler divergence},
  sort=A245
}

\newglossaryentry{standard-hellinger}
{
  name={$\Hell_{1/2}(\cdot \pipes \cdot)$},
  description={standard (squared) Hellinger distance},
  sort=A25
}

\newglossaryentry{gen-hellinger}
{
  name={$\Hell_\eta(\cdot \pipes \cdot)$},
  description={$\eta$-generalized Hellinger divergence},
  sort=A26
}

\newglossaryentry{renyi}
{
  name={$D_{\alpha}(p \| q)$},
  description={R\'enyi divergence of order $\alpha$; \,\, $D_{\alpha}(p \| q) = \frac{1}{\alpha-1} \log  \int p^{\alpha} q^{1- \alpha} d \mu$},
  sort=A265
}

\newglossaryentry{pseudo}{
  name={\textbf{Pseudo-predictors}},
  description={\nopostdesc},
  sort=A27,
  nonumberlist
}

\newglossaryentry{allactionset}
{
  name={$\allactionset$},
  description={enlarged action space $\allactionset \supseteq \cF$ that also contains pseudo-predictors},
  sort=A28
}

\newglossaryentry{f-star-epsilon}
{
  name={$f^*_\epsilon$},
  description={pseudo-predictor, defined via its loss by $\ell_{f^*_{\epsilon}}(z) = \ell_{f^*}(z) - \epsilon$ for all $z \in \cZ$},
  sort=A29
}

\newglossaryentry{pseudoprobs}
{
  name={$\mathcal{E}_{\cF,\eta}$},
  description={set of pseudoprobability densities; \,\, $\mathcal{E}_{\cF,\eta} = \left\{ e^{-\eta \loss_f} : f \in \cF \right\}$ },
  sort=A30
}

\newglossaryentry{mix-pseudoprob}
{
  name={$\xi_Q$},
  description={mixture of pseudoprobability densities; \,\, $\xi_Q = \E_{\rv{f} \sim Q} [ e^{-\eta \loss_{\rv{f}}} ]$},
  sort=A31
}

\newglossaryentry{grip}
{
  name={$\grip{\eta}$ or $\oldgriplosseta$},
  description={GRIP; \,\, $\E [ \grip{\eta} ] 
= \inf_{Q \in \Delta(\cF)} \E \left[ -\frac{1}{\eta} \log \E_{\rv{f} \sim Q} \bigl[ e^{-\eta \loss_{\rv{f}}} \bigr] \right]$},
  sort=A32
}

\newglossaryentry{mixloss-Q}
{
  name={$\mixloss{\eta}{Q}$},
  description={mix loss for $Q \in \Delta(\cF)$; \,\, $\mixloss{\eta}{Q} = -\frac{1}{\eta} \log \E_{\rv{f} \sim Q} \bigl[ e^{- \eta \loss_{\rv{f}}} \bigr]$},
  sort=A33
}

\newglossaryentry{gripgen}
{
  name={$\gripgen{\eta}{A}$},
  description={generalized GRIP w.r.t.~$A \subseteq \allactionset$; \,\, $\displaystyle \E \bigl[ \gripgen{\eta}{A} \bigr] = \inf_{Q \in \Delta(A \cup \{\fopt\})} \E \bigl[ \mixloss{\eta}{Q} \bigr]$},
  sort=A34
}

\newglossaryentry{mini-grip}
{
  name={$\gripgen{\eta}{f}$},
  description={mini-grip w.r.t.~$f$; \,\, $\displaystyle \E [ \gripgen{\eta}{f} ] = \inf_{\alpha \in [0, 1]} \E \left[ -\frac{1}{\eta} \log \left( (1 - \alpha) e^{-\eta \loss_{\fopt}} + \alpha e^{-\eta \loss_f} \right) \right]$},
  sort=A36
}

\newglossaryentry{pseudo-actions-for-grips}
{
  name={$\gripbase{\eta}$ and $\gripbasegen{\eta}{f}$},
  description={pseudo-actions for GRIP losses $\grip{\eta}$ and $\gripgen{\eta}{f}$ respectively \qquad \qquad \qquad \qquad},
  sort=A37
}

\newglossaryentry{conditions}{
  name={\textbf{Conditions}},
  description={\nopostdesc},
  sort=C,
  nonumberlist
}

\newglossaryentry{bernstein}
{
  name={$(\beta, B)$-Bernstein},
  description={$\E [ \xsloss{f}^2 ] \leq B \left( \xsrisk{f} \right)^\beta$ \,\, for all $f \in \cF$},
  sort=CB1
}

\newglossaryentry{strong-central}
{
  name={strong $\bar{\eta}$-central},
  description={$\exists \fopt \in \cF$ s.t.~$\loss_{\fopt} - \loss_f \stochleq_{\bar\eta} 0$ \,\, for all $f \in \cF$
},
  sort=CC1
}

\newglossaryentry{central-up-to-eps}
{
  name={$\eta$-central up to $\varepsilon$},
  description={$\exists \fopt \in \cF$ s.t. $\loss_{\fopt} - \loss_f \stochleq_{\eta} \epsilon$ \,\, for all $f \in \cF$},
  sort=CC2
}

\newglossaryentry{v-central}
{
  name={$v$-central},
  description={for all $\varepsilon \geq 0$, $\exists \fopt \in \cF$ s.t. $\loss_{\fopt} - \loss_f \stochleq_{v(\varepsilon)} \epsilon$ \,\, for all $f \in \cF$},
  sort=CC3
}

\newglossaryentry{PPC-up-to-eps}
{
  name={$\eta$-PPC up to $\varepsilon$},
  description={$\exists \fopt \in \cF$ s.t. $\E_{Z \sim \Pt} \left[ \loss_{\fopt} - \grip{\eta}\right] \leq \epsilon$},
  sort=CC4
}

\newglossaryentry{v-PPC}
{
  name={$v$-PPC},
  description={for all $\varepsilon \geq 0$, $\exists \fopt \in \cF$ s.t. $\E_{Z \sim \Pt} \left[ \loss_{\fopt} - \grip{v(\varepsilon)}\right] \leq \epsilon$},
  sort=CC5
}

\newglossaryentry{u-c-witness}
{
  name={$(u,c)$-witness},
  description={$\E \left[ (\xslosslong{f}) \cdot \ind{\xslosslong{f} \leq u} \right] \geq c \xsrisklong{f}$ for all $f \in \cF$},
  sort=CW1
}

\newglossaryentry{tau-c-witness}
{
  name={$(\tau,c)$-witness},
  description={generalized version of $(u,c)$-witness condition (see Definition~\ref{def:witness})},
  sort=CW2
}

\newglossaryentry{advanced-witness}
{
  name={witness w.r.t.~$\phi$},
  description={$(u,c)$-witness condition with dynamic comparator (see Assumption~\ref{ass:adv-emp-witness-badness})},
  sort=CW3
}

\newglossaryentry{weak-advanced-witness}
{
  name={weak witness w.r.t.~$\phi$},
  description={weakened version of the previous condition (see Assumption~\ref{ass:adv-emp-witness-badness})},
  sort=CW4
}

\newglossaryentry{unif-exp-tail}
{
  name={unif.~exp.~upper tail},
  description={$U_f$ (for $f \in \cF$) has condition if $\exists \kappa \in (0, \infty)$ s.t.~$\sup_{f \in \cF} \E \left[ e^{\kappa U_f} \right] < \infty$},
  sort=CZ1
}

\newglossaryentry{small-ball}
{
  name={small-ball assumption},
  description={$\exists \kappa > 0$ and $\epsilon \in (0, 1)$ s.t.~$\forall f, h \in \cF$, \,\, $\Pr \left( |f - h| \geq \kappa \|f - h\|_{L_2(P)} \right) \geq \varepsilon$},
  sort=CZ2
}

\newglossaryentry{conv-lucky}
{
  name={convex luckiness},
  description={(for squared loss); \,\, $\argmin_{f \in \cF} \E [ \loss_f ] = \argmin_{f \in \convhull(\cF)} \E [ \loss_f ]$},
  sort=CZ3
}

\makeglossaries
\fi

 \coltauthor{\Name{Peter D. Gr\"unwald} \Email{pdg@cwi.nl}\\
 \addr Centrum Wiskunde \& Informatica, Amsterdam, The Netherlands\\
           Leiden University, Mathematical Institute, Leiden, The Netherlands\\
 \AND
 \Name{Nishant A. Mehta} \Email{nmehta@uvic.ca}\\
 \addr Department of Computer Science, University of Victoria\\
           Victoria, Canada
 }

\begin{document}
\hypersetup{colorlinks=true,citecolor=blue,linkcolor=blue} 
\title{Fast Rates for General Unbounded Loss Functions: \\ From ERM to Generalized Bayes}

\maketitle

\begin{abstract}
We present new excess risk bounds for
  general unbounded loss functions including log loss and squared
  loss, where the distribution of the losses may be heavy-tailed. The
  bounds hold for general estimators, but they are optimized when applied
  to $\eta$-generalized Bayesian, MDL, and empirical risk minimization estimators. In the case of log loss, the bounds imply convergence rates for generalized
  Bayesian inference under misspecification in terms of a
  generalization of the Hellinger metric as long as the learning rate $\eta$ is set correctly. For general loss functions, our
  bounds rely on two separate conditions: the $v$-GRIP (generalized reversed information projection) conditions, which control
  the lower tail of the excess loss; and the newly introduced
  witness condition, which controls the upper
  tail. The parameter $v$ in the $v$-GRIP conditions determines the
  achievable rate and is akin to the exponent in the Tsybakov margin
  condition and the Bernstein condition for bounded losses, which the
  $v$-GRIP conditions generalize; favorable $v$ in combination with
  small model complexity leads to $\tilde{O}(1/n)$ rates. The
  witness condition allows us to connect the excess risk to an
  ``annealed'' version thereof, by which we generalize several previous
  results connecting Hellinger and R\'enyi divergence to KL
  divergence.
\end{abstract}

\begin{keywords}
  statistical learning theory, fast rates, PAC-Bayes, misspecification, generalized Bayes.
\end{keywords}

\section{Introduction}
Much of statistical learning theory has operated under the restrictive
assumption that the loss suffered for any prediction falls into some
finite interval, which to say that the losses are bounded. In
addition, much of this theory for deterministic estimators and even
more so for randomized estimators only yields ``slow'' convergence
rates of the risk of the predictor to the minimum risk achievable via
the model in use; these are the best rates possible in the face of a
worst case distribution.  Faster rates of convergence are often
possible under various, practically-applicable conditions on the
learning problem, and showing such improvements is important as they
can translate to drastic reductions on the number of examples needed
to achieve a fixed level of error.  We provide a novel theory of
excess risk bounds for deterministic and randomized estimators in
settings with general unbounded loss functions which may have
heavy-tailed distributions --- important applications include
regression in situations with heavy-tailed noise and density
estimation with log loss without assuming boundedness of likelihood
ratios.  These bounds have implications for two different areas: in
statistical learning, they establish that with unbounded losses, under
weak conditions, one can obtain estimators with \emph{fast}
convergence rates of their risk --- such conditions previously were
only well understood in the bounded case (earlier work on
generalization bounds for unbounded loss functions such as
\citep{MeirZ03,CortesGM19} typically needs much stronger conditions to
obtain fast rates).  In density estimation under misspecification, the
new bounds imply convergence rates for $\eta$-generalized Bayesian
posteriors, in which the likelihood is raised to a power $\eta$ not
necessarily equal to $1$, under surprisingly weak conditions.
Finally, the bounds highlight the close similarity between
PAC-Bayesian and $\eta$-generalized Bayesian learning methods under
misspecification; these methods usually are studied within different
communities.  We now consider these applications in turn:
 
\paragraph{1. Statistical Learning} In Statistical Learning Theory
\citep{vapnik1995nature} the goal is to learn an action or predictor
$\hat{f}$ from some set of actions, or \emph{model}, $\cF$ based on
i.i.d.~data \glsadd{Zn} $Z^n \equiv Z_1, Z_2, \ldots, Z_n \sim P$,
where \glsadd{P} $P$ is an unknown probability distribution over a
sample space $\cZ$.  One hopes to learn an $\hat{f}$ with small risk,
i.e., expected loss $\E [\loss_{\hat{f}}(Z)]$, for some given loss
function $\loss$.  Here, $\E$ denotes expectation under $P$, and
\glsadd{fhat} $\hat{f} \equiv \hat{f}(Z^n)$ is a function from $\cZ^n$
to $\cF$ that represents a learning algorithm; a prototypical example
is empirical risk minimization (ERM). Thus, as is common, with some
abuse of notation a learning algorithm is really a function, i.e., we
do not insist it to be computable; and, in statistical contexts, we sometimes
refer to learning algorithms as {\em estimators}, simply because this
is common usage.  A \emph{learning problem} can thus be summarized as
a tuple \glsadd{learning-problem} $\smtuple$. Well-known special cases
include classification (with $\loss$ the 0-1 loss or some convex
surrogate thereof) and regression (with $\loss$ the squared loss).
As is customary (see e.g.~\citep{bartlett2005local}
and \citep{mendelson2014learning}), in most of our results we assume existence of an optimal
\glsadd{fopt} $\fopt \in \cF$ achieving $\E[\loss_{\fopt}(Z)] = \inf_{f \in \cF}
\E[\loss_f(Z)]$, and we define the excess loss of $f$ as
\glsadd{excess-loss} $\xsloss{f} = \loss_f - \loss_{\fopt}$. 

When the losses are almost surely bounded under $P$, there exists a well-established theory that gives optimal convergence rates of the excess risk $\xsrisk{\hat{f}}$ of estimator $\hat{f}$ in terms of sample size $n$. 
Broadly speaking, in the bounded case the optimal rate is
usually of order
\begin{align} \label{eq:convrate}
O\left( \left(\frac{\textsc{comp}_n}{n}\right)^\gamma \right) ,
\end{align} 
where $\textsc{comp}_n$ is a measure of model complexity such as
the Vapnik-Chervonenkis (VC) dimension or the log-cardinality of an optimally chosen $\epsilon$-net
over $\cF$, among others. For the models usually studied in statistics, such complexity measures are sublinear in $n$, and for ``simple'' models (often called parametric models, like those of finite VC dimension in classification) are finite or logarithmic in $n$. 
The exponent $\gamma$, which is in the range $[1/2, 1]$ in practically all cases of interest, reflects the \emph{easiness} of a learning problem by depending on both geometric and statistical properties of $\smtuple$. 
This exponent is equal to $1/2$ in the worst case but can be larger, allowing for faster rates, if the loss $\loss$ has sufficient curvature, e.g., if it is
\emph{exponentially concave (exp-concave)} or \emph{mixable} \citep{CesaBianchiL06}, or if
$\smtuple$ satisfies ``easiness'' conditions such as the \emph{Tsybakov margin condition} \citep{tsybakov2004optimal}, 
a \emph{Bernstein condition} 
\citep{audibert2004pac,bartlett2006empirical}, or \emph{(stochastic) exp-concavity} \citep{juditsky2008learning}. Because these conditions and the others on which this paper centers can allow for learning at faster rates, when any of the conditions hold a learning problem is intuitively easier. We thus call all such conditions \emph{easiness conditions} throughout this work. 
In this literature, one often calls (\ref{eq:convrate}) with $\gamma = 1/2$ the \emph{slow rate} and  \eqref{eq:convrate} with $\gamma =1$ the \emph{fast rate}.
We note, however, that the terminology ``fast rate'' is somewhat imprecise, as there are special cases for which rates even faster than $n^{-1}$ are possible \citep{audibert2007fast}. A more precise term may be ``optimistic rate'' (see \citep{mendelson2017aggregation} for a lucid discussion), as this is the rate obtainable in the optimistic situation where an easiness condition holds. We opt for ``fast'' primarily for historical reasons.

\Citet{erven2015fast} showed that, in the case when the excess losses are bounded\footnote{\Citet{erven2015fast} actually assume that the losses are bounded, but inspection of the results therein reveals that all that is needed is in fact bounded \emph{excess} losses.},
all the ``easiness'' conditions above are subsumed by what they term the
$v$-central condition, where $v$ is a function that effectively modulates $\gamma$.
While \Citet{erven2015fast} do show connections between such conditions for
unbounded excess losses as well, they left open the question of whether the conditions still imply fast rates in that case. Thus, the first main
target of the present paper is to extend this ``fast rate theory'' to the unbounded and heavy-tailed excess loss case. A main
consequence of our bounds is that under $v$-\emph{GRIP} conditions (``GRIP'' stands for \emph{generalized reversed information projection}), which consist of the $v$-central condition and a weakening thereof, 
and an additional \emph{witness} condition, the obtainable rates remain the same as in the bounded case.

\paragraph{2. Density Estimation under Misspecification} Letting $\cF$
index a set of probability densities $\{ p_f : f \in \cF \}$ and
setting the loss $\loss$ to the log loss, $\loss_f(z) = - \log
p_f(z)$, we find that the statistical learning problem becomes
equivalent to density estimation, the excess risk becomes equal to the
\emph{generalized Kullback-Leibler (KL) divergence}
\begin{align*}
D(\fopt \pipes \hat{f}) = \E _{Z \sim P}[\log (p_{\fopt}(Z)/p_{\hat{f}}(Z))] ,
\end{align*}
and ERM becomes maximum likelihood estimation. 
We call a model $\cF$ {\em well-specified\/} if it is correct, i.e., if  
$p_{\fopt}$ is the density of the true distribution $P$; in that case 
$D(\fopt \pipes \hat{f})$ becomes the standard KL divergence. 
In this setting, our results thus automatically become
convergence bounds of estimators $\hat{f}$ to the KL-optimal density
within $\cF$, where the convergence itself is in terms of KL
divergence rather than more usual, weaker metrics such as Hellinger
distance. Here, our results vastly generalize earlier results on KL
bounds which typically rely on strong conditions such as boundedness
of likelihood ratios or exponential tail conditions
\citep{birge1998minimum,yang1998asymptotic,wong1995probability,sason2016f};
in this work, the much weaker witness condition suffices.

We also provide bounds that are more similar to the standard
Hellinger-type bounds and that hold without the witness condition,
having a generalization of squared Hellinger distance (suitable for
misspecification) rather than KL divergence on the left.  Our bounds
also allow for estimators that output a distribution $\dol$ on $\cF$
rather than a single $\hat{f}$ and are particularly well-suited for
$\eta$-generalized Bayesian posteriors, in which the likelihood in the
prior-posterior update is raised to a power $\eta$; standard Bayes
corresponds to $\eta = 1$.  We thus can compare our rates to classical
results on Bayesian rates of convergence in the well-specified case,
such as in the influential paper \Citep*{GhosalGV00} (GGV from now on). 
In this case, we generally obtain rates comparable to those of {GGV}, but under weaker conditions, as long as we take $\eta$ (arbitrarily close to but) smaller than $1$, a fact already noted for $\eta$-generalized Bayes by \cite{zhang2006epsilon,martin2017empirical,WalkerH02}. 
In contrast to earlier work, however, our results remain valid in the misspecified case, although $\eta$ has to be adjusted there to get convergence at all; moreover, the rates obtained are with respect to a new ``misspecification metric'' and hence are not always comparable to those obtained in the well-specified case. 
The optimal $\eta$ depends on the ``best'' parameter $v$ for which a $v$-GRIP condition holds.  \cite{GrunwaldO17} give a simple example
  which shows that taking $\eta=1$ (standard Bayes) in regression
  under misspecification can lead to results that are dramatically
  worse than taking the right $\eta$, thus showing that our results do
  have practical implications.

\paragraph{3. $\boldsymbol{\eta}$-generalized Bayes and PAC-Bayes}
The $\eta$-generalized Bayesian posterior can be further generalized:
for general loss functions $\loss$, we can define ``posteriors''
$\dol_n^B$ with densities given by \glsadd{gen-bayes}
\begin{equation}\label{eq:bayes}
\frac{ d \dol_n^B}{d \Prior} (f) \equiv \postdens^B_n (f) 
\equiv \postdens^B(f \mid z_{1}, \ldots, z_{n})  :=  \frac
{\exp\left(-\eta \sum_{i=1 }^n \loss_f(z_i) \right)}
{\int_\model
  \exp\left(-\eta \sum_{i= 1 }^n \loss_h(z_i)\right) \cdot d\Prior(h)},
\end{equation}
for some ``prior'' distribution $\Prior$ on $\cF$. 
This idea goes back at least to \cite{vovk1990aggregating} and 
is central in the PAC-Bayesian approach to statistical
learning \citep{McAllester02}. Recently, it has also been embraced
within the Bayesian community \citep{bissiri2016general,miller2018robust}. Nevertheless, the communities studying frequentist convergence of
Bayesian methods under misspecification and PAC-Bayesian analysis are
still largely separate; yet, the present paper shows that the
approaches can be analyzed using the very same machinery and that it
is fruitful to do so. To wit, \emph{all} our results are based on
an existing lemma due to T. Zhang
(\citeyear{zhang2006information,zhang2006epsilon}) which provides
convergence bounds in terms of an ``annealed'' pseudo-excess risk for
general estimators; these bounds are optimized if one plugs in
$\eta$-generalized Bayesian estimators of the general form above. 
Zhang's bound is
itself based on earlier works in the information theory literature 
(in particular, the Minimum Description Length (MDL) literature) 
\citep{barron1991minimum,li1999estimation}) and the PAC-Bayesian
literature \citep{catoni2003pac,audibert2004pac}. Of course, the
technique also has some disadvantages, to which we return in the Discussion (Section~\ref{sec:related-work}).

\subsection{Overview and Main Insights of the Paper}
Section~\ref{sec:extended-intro} formalizes the
setting; Section~\ref{sec:related-work} discusses additional related work and potential future work and provides discussion. The paper ends with appendices containing all long proofs, technical details
concerning infinities, and some additional examples.
The main results are in Sections~\ref{sec:zhang}--\ref{sec:grip}:

\paragraph{Section~\ref{sec:zhang}: Zhang's Bound; Information Complexity}
In Section~\ref{sec:zhang}, for which we do not claim any novelty, we
present Lemma~\ref{lem:rascbound-simple}; this lemma is T.~Zhang's
(\citeyear{zhang2006information,zhang2006epsilon}) result that bounds 
a pseudo-excess risk of estimator $\hat{f}: \cZ^n \rightarrow \cF$ 
in terms of the \emph{information complexity} $\rsc_{n,\eta}$. 
A very simplified form of this lemma is
\begin{equation}\label{eq:protozhang}
\Expann{\eta}_{Z \sim \Pt} \left[ \xsloss{\hat{f}} \right] 
 \stochleq_{\eta \cdot n }  \rsc_{n,\eta},
\end{equation}
where the pseudo-excess risk $\Expann{\eta}_{Z \sim \Pt}$ is formally defined in \eqref{eq:genren} 
and $\stochleq$ indicates {\em exponential stochastic inequality\/}
(ESI), a useful notational tool which we define. ESI implies both
inequality in expectation and with high probability over the sample
$Z^n$ that determines $\hat{f} \equiv \hat{f}(Z^n)$; the subscript $\eta \cdot n$ is only relevant for the in-probability version (see Proposition~\ref{prop:drop}) and can be ignored for now.  The actual bound
(\ref{eq:mainrascbound-simple}) we provide in
Lemma~\ref{lem:rascbound-simple} generalizes (\ref{eq:protozhang}),
also allowing for estimators that output a distribution such as
generalized Bayesian posteriors as given by (\ref{eq:bayes}).
$\rsc_{n,\eta}$ is a notion of model complexity which, apart from $n$
and $\eta$, also depends (for now suppressed in the notation) on the
data $Z^n$, the choice of estimator $\hat{f}$ or $\dol_n$, and on a distribution $\Prior$ on $\cF$ which we may think
of as ``something like'' a prior: while the bound holds for any fixed $\Prior$, 
the estimator that {\em minimizes\/} $\rsc_{n,\eta}$ for given prior
$\Prior$ and data $Z^n$ is the corresponding $\eta$-generalized
Bayesian posterior $\dol^B_n$ given by (\ref{eq:bayes}).

For this choice of estimator, one can often design priors such that,
with high probability and in expectation, $\rsc_{n,\eta}$ for the
$\eta$-generalized Bayesian estimator can be upper bounded as
\begin{equation}\label{eq:optimalrates}
\rsc_{n,\eta} = \tilde{O} \left(\frac{\textsc{comp}_n}{\eta n}\right),
\end{equation}
for functions $\textsc{comp}_n$ that rely on the model $\cF$'s
complexity as indicated above (the $\tilde{O}$-notation suppresses
logarithmic factors). In Section~\ref{sec:zhang} we show that in the
application to well-specified density estimation, priors can always be
chosen such that the classical posterior contraction rates of GGV are
(essentially) recovered for any fixed $\eta > 0$, in the sense that
(\ref{eq:protozhang}) would imply the same rates if the left-hand side
were replaced by a squared Hellinger distance. For example, for
standard finite and parametric statistical models, we obtain for Bayesian estimators that $\textsc{comp}_n = \tilde{O}(1)$; for
the nonparametric statistical models considered by GGV, we obtain
$\textsc{comp}_n = \tilde{O}(n^{\alpha})$ for an $\alpha$ such that
(\ref{eq:optimalrates}) becomes the minimax optimal rate. Similar
bounds on $\rsc_{n,\eta}$ with general loss functions are given in
Section~\ref{sec:grip}. Henceforth, we use the term \emph{parametric} to refer to $\cF$ for which
generalized Bayes estimators give
$\textsc{comp}_n = O(\log n) = \tilde{O}(1)$.

We would thus get good convergence bounds if the left-hand side of
\eqref{eq:protozhang} were the actual excess risk, but instead it is
an ``annealed'' version thereof, always smaller than the actual excess
risk and sometimes even negative. All of our own results can be viewed
as establishing conditions under which the annealed excess risk can
either be related to the actual excess risk or otherwise to a
(generalized Hellinger) metric measuring ``distance'' between $\fopt$
and $f$ in some manner; this is done by modifying $\eta$. Both the
information complexity and its upper bound \eqref{eq:optimalrates} can
only increase as we decrease $\eta$
(Proposition~\ref{prop:transavia}); yet, for small enough $\eta$,
annealed convergence implies convergence in the {sense in which we are
  interested (either excess risk or generalized Hellinger distance)}
up to some constant factor (Sections~\ref{sec:strongcentral}
and~\ref{sec:witness}) and sometimes with an additional slack term
(Sections~\ref{sec:witness} and~\ref{sec:grip}). Thus, the optimal
$\eta$ is given by a tradeoff between information complexity and these
additional factors and terms.

Sections~\ref{sec:strongcentral}--\ref{sec:grip} each contain (a)
a condition enabling a link between annealed excess risk and the
divergence of interest in that section; (b) a new theoretical concept
underlying the condition, (c) convergence result(s) relating
information complexity to an actual metric or excess risk, and (d) example(s) that illustrate it.

\paragraph{Section~\ref{sec:strongcentral}: The Strong Central Condition and a New Metric; First Convergence Result}
The \emph{strong central condition} \Citep{erven2015fast} expresses that the lower tail of the
excess loss $L_f := \loss_{f} - \loss_{\fopt}$ is exponential,
i.e., $P(\loss_{\fopt} - \loss_f > A)$ is exponentially small in $A$. It
has a parameter $\bar\eta > 0$ that determines the precise bound that
can be obtained. While this may sound like a very strong condition,
due to the nature of the log loss it automatically holds for
density estimation with $\bar\eta =1$ if the model is well-specified
or convex. We show (Theorem~\ref{thm:metric}) that the
$\bar\eta$-strong central condition is sufficient for convergence in a
new ``misspecification'' metric $\dhel_{\bar\eta}$ (Definition~\ref{def:misspec-metric}) that generalizes the Hellinger distance: there exist estimators such
that for every $0 < \eta < \bar\eta$,
\begin{align*}
\dhel^2_{\bar\eta}(\fopt, \hat{f}) \
\stochleq_{\eta \cdot n } C_{\eta} \cdot \rsc_{n,\eta}, 
\end{align*}
where $C_{\eta}$ is a constant that tends to $\infty$ as
$\eta \uparrow \bar\eta$ and is bounded by $1$ if
$\eta \leq \bar\eta/2$.  For misspecified models, $\bar\eta$ can in
principle be either smaller or larger than $1$. This metric is mainly
of interest in the density estimation application of our work, and we
thus compare our results to those of GGV for well-specified density
estimation and illustrate them for the case of misspecified
generalized linear models (GLMs). Plugging in any fixed
$\eta < \bar\eta$ in \eqref{eq:optimalrates} and comparing to
\eqref{eq:convrate}, we see that under the strong central condition,
we can always achieve the fast rate, i.e., \eqref{eq:convrate} with $\gamma=1$.

\paragraph{Section~\ref{sec:witness}: The Witness Condition and a First Excess Risk Convergence Result}
Here we consider when, under the strong central condition, we can get bounds on the actual excess risk (or, in density estimation, on the generalized KL divergence). 
We provide a new concept, the \emph{empirical witness of badness condition}, or \emph{witness condition} for short, which provides control over the upper tail of the excess loss 
$L_f = \loss_{f} - \loss_{\fopt}$ 
(whereas the central condition concerns the lower tail). 
Essentially, the witness condition says that whenever $f \in \cF$ is worse than $\fopt$ in expectation, the probability that we witness this in our training example
should not be negligibly small. We thus rule out the case that $f$ has
extremely large loss with extremely small probability. This condition
turns out to be quite weak --- it can still hold if, for example, the
excess loss $\loss_{f} - \loss_{\fopt}$ is heavy-tailed 
(it suffices for the conditional second moment of the target to be uniformly bounded almost surely; 
see Example~\ref{ex:heavy-tailed-regression}). 
Thus we establish our first excess risk convergence result, Theorem~\ref{thm:firstriskbound}, which, in its simplest form, says that if both the central condition holds with parameter $\bar\eta$ and the witness condition holds, then for all $0 < \eta < \bar\eta$, 
\begin{equation}\label{eq:protofirstrisk}
\xsrisk{\hat{f}} \ \stochleq_{\eta \cdot n / a_\eta} \ a_\eta \cdot
 \rsc_{n,\eta},
\end{equation}
where $a_\eta$ is a constant that again tends to $\infty$ as $\eta \uparrow \bar\eta$. 
Once again, by combining \eqref{eq:protofirstrisk} and \eqref{eq:optimalrates}, we see that under a witness and $\bar\eta$-central condition, we can achieve the fast rate by taking $\gamma = 1$ in \eqref{eq:convrate}.

The witness condition vastly generalizes earlier conditions such as 
boundedness of likelihood ratios in density estimation 
\citep{birge1998minimum,yang1998asymptotic} and the exponential tail
condition of \cite{wong1995probability}. Moreover,
(\ref{eq:protofirstrisk}) (Theorem~\ref{thm:firstriskbound}) is based
on Lemma~\ref{lem:renyi-to-kl-simple}, which  generalizes
earlier results relating KL divergence to Hellinger and R\'enyi-type
divergences such as those of \cite{yang1999information},
\cite{haussler1997mutual}, \cite{birge1998minimum},
\cite{wong1995probability}, and \cite{sason2016f}.
We also discuss the  similarity between the witness condition and the
recently introduced \emph{small-ball assumption} of \cite{mendelson2014learning}.

\paragraph{Section~\ref{sec:grip}: Weaker Fast Rate Conditions; 
the GRIP}
The $\bar\eta$-central condition of Section~\ref{sec:strongcentral}
can be generalized to the $v$-central condition, where
$v: \reals^+ \rightarrow \reals^+$ is a non-decreasing function;
nonconstant $v(x)$ gives weaker conditions that still allow for fast
rates. \Citet{erven2015fast} showed that for the bounded excess loss
case, most existing easiness conditions can be shown to be equivalent to
either a $v$-central condition or to what they call a $v$-{PPC} 
\emph{(pseudo-probability-convexity)} condition. 
In one of their central results, they show these two seemingly different conditions to
be equivalent to one another, and also, if $v$ is of the form
$v(x) \asymp x^{1-\beta}$, (essentially) equivalent to a
$(B,\beta)$-\emph{Bernstein condition} \citep{audibert2004pac,bartlett2006empirical}. 
In this section we show that for unbounded excess losses, the $v$-central and $v$-PPC
conditions become quite different from each other (and also from the
Bernstein condition): the $v$-PPC condition allows for heavy-
(polynomial) tailed loss distributions, whereas the $v$-central
condition does not. 

We first present 
Theorem~\ref{thm:main-bounded-a}, an excess risk bound under the
$v$-central condition that is a relatively straightforward consequence
of Theorem~\ref{thm:firstriskbound}, our risk bound under the
$\bar\eta$-central condition. We then move to Theorem~\ref{thm:main-bounded-b}, a similar excess risk bound under the $v$-PPC condition. This theorem involves 
the \emph{GRIP}, the novel, fundamental concept of this section (Definition~\ref{def:grip}). 
GRIP stands for \emph{generalized reversed information projection} and generalizes the concept of reversed information projection introduced by \cite{li1999estimation}. 
The GRIP $\grip{\eta}$ is an $\eta$-dependent pseudo-predictor 
(it might achieve smaller risk than any $f$ for which $\loss_f$ is
defined). We show that, for each $\eta$, if
$\fopt$ is replaced by the GRIP $\grip{\eta}$, then the convergence
result (\ref{eq:protofirstrisk}) above holds.  We can interpret
the $v$-PPC condition as controlling the excess
risk of $\fopt$ over the GRIP $\grip{\eta}$  as a function of $\eta$: the
smaller $\eta$, the smaller this excess risk. This determines, for
each sample size, an optimal $\eta$ at which the bound \eqref{eq:protofirstrisk} 
and the excess risk of $f^*$ relative to $\grip{\eta}$ balance. Theorem~\ref{thm:main-bounded-a}   
establishes that whenever the
witness condition holds and a $v$-central condition holds, 
we have, for every $\epsilon > 0$, for $\eta < v(\epsilon)$,
\begin{equation}\label{eq:protosecondrisk}
\xsrisk{\hat{f}} \ \stochleq_{\eta \cdot n/a'_{\eta}} \  a'_{\eta}  \cdot
 \rsc_{n,\eta} + \epsilon;
\end{equation}
where again $a'_{\eta}$ is a
constant. Theorem~\ref{thm:main-bounded-b} shows that if a $v$-PPC
condition holds, the same result holds whenever
$\eta < v(\epsilon)/2$, but now only in expectation, for yet another
$a'_{\eta}$.  Thus, the optimal rate now depends on $v$; 
in particular, if $v(\epsilon) \propto \epsilon^{1- \beta}$, then we can optimize over
$\epsilon$ using upper bound (\ref{eq:optimalrates}) and find that, as
long as $\textsc{comp}_n$ is logarithmic in $n$ (as in parametric settings), by
setting $\eta$ at sample size $n$ equal to
$\eta \asymp n^{- (1-\beta)/(2-\beta)}$ we obtain the rate
\begin{equation}\label{eq:tuning}
\xsrisk{\hat{f}} = \tilde{O}\left(n^{-\frac{1}{2-\beta}} \right)
\end{equation}
which interpolates between the fast rate ((\ref{eq:convrate}) with
$\gamma = 1$) and the slow rate ($\gamma = 1/2$), where
$\gamma = 1/(2-\beta)$ depends on $\beta$.  Such calculations are
well-known for the bounded loss case, and our results establish that
the same story continues to hold for the unbounded excess loss case,
as long as a witness condition holds --- even for heavy-tailed losses.
While Theorems~\ref{thm:main-bounded-a} and~\ref{thm:main-bounded-b}
are applicable to the unbounded-loss-yet-bounded risk case (for which
$\sup_{f \in \cF} \E[\loss_{f}] < \infty$),
Theorem~\ref{thm:main-unbounded} extends this result to the unbounded
risk case, requiring a slight generalization of the witness condition.
Examples~\ref{ex:heavy-tailed-regression-II}~and~\ref{ex:smallballagain} illustrate our results by considering
regression with heavy-tailed losses, the latter example further linking
the aforementioned small-ball assumption to our generalized witness condition.
\paragraph{The Picture that Emerges}
Our results point to three separate factors that determine achievable
convergence rates for generalized Bayesian, two-part MDL, and empirical risk minimization (ERM) estimators, 
which often, but not always (see below) coincide with minimax rates:
\begin{enumerate} \item The \emph{information complexity}
  $\rsc_{n,\eta}$, which determines the ``richness'' of the model. It is
  data- and algorithm- dependent, but we can often bound it with high
  probability or even independently of the underlying $P$. 
  In addition, to see what rates can be achieved, we can plug in the ($\eta$-generalized Bayesian) learning algorithms that minimize it. 
\item The \emph{interaction between $P$, $\loss$, and $\cF$} that
  determines, for each $f \in \cF$, the distribution of the \emph{lower tail} of the excess loss $\xsloss{f}$. This interaction is sometimes called
  the \emph{easiness} of the problem \citep{koolen2016combining}; it
  determines the optimal $\eta$ at which a bound on $\eta$-information
  complexity implies a bound on the generalized Hellinger-type metric. 
  This is captured by our $v$-GRIP conditions, which generalize several existing easiness conditions.
\item The \emph{interaction between $P$, $\loss$, and $\cF$} that
  determines the distribution of the \emph{upper tail} of the excess
  loss. This interaction plays \emph{no} role for bounded excess losses and
  \emph{no} role for density estimation if one only cares about
  convergence in the weak misspecification metric. Yet for unbounded
  excess losses with the excess risk target (or density estimation with KL-type target), this interaction becomes crucial
  to take into account and is done so via the witness condition.
\end{enumerate}
In the Discussion (Section~\ref{sec:related-work}), Figure~\ref{fig:overview} summarizes how the various conditions hang together and are in some special cases (e.g.~squared loss) implied by existing, better-known easiness conditions imposed in other works.

\paragraph{What we do \emph{not} cover}
We stress at the outset that we do not cover everything there is to
know about the type of convergence bounds we prove. First of all, our
bounds are most useful for ERM, $\eta$-generalized Bayesian, and MDL
estimators, for a specific $\eta$ that depends on the learning problem
$\smtuple$ and often also on $n$. Thus to apply generalized Bayes/MDL
in practice, $\eta$ needs to be determined in some data-driven way; we
discuss various ways to do this in
Section~\ref{sec:related-work}. Note though that our bounds can be
directly used for ERM, which can be implemented without knowledge of
$\eta$.

We also leave untouched the fact that for parametric
models, Zhang's bounds lead to an unnecessary $\log n$-factor in the
convergence rates. Zhang
(\citeyear{zhang2006information,zhang2006epsilon}), following
\cite{catoni2003pac}, addresses this issue by a relatively
straightforward ``localized'' modification of his bound; since it
distracts from our main points (the witness and GRIP conditions, which lead
to polynomial gains in rate), we will simply ignore all 
logarithmic factors in this paper.

Third, the  new convergence rates for $\eta$-generalized Bayesian,
MDL, and ERM estimators that we establish are in some cases, but not always, minimax optimal. We do explicitly discuss for each example below whether the obtained rates are optimal and discuss exceptions, unknowns, and potential remedies in Section~\ref{sec:related-work}.

Finally, we only discuss \emph{proper} and \emph{randomized proper}
learning algorithms and estimators here. 
This means that our estimators either output an $\hat{f} \in \cF$ or, if they output a
distribution $\dol \mid Z^n$, it is always a distribution on $\cF$, and the
quality of this distribution is evaluated by the expected loss
incurred if one draws an $f$ randomly from $\dol \mid Z^n$. 
The terminology ``proper'' is from learning theory \citep{lee1996efficient}; 
in statistics such estimators are sometimes called ``in-model'' \citep{Grunwald07}. 
In learning theory, one often considers more general ``improper'' set-ups in which one can play an element of (say) $\convhull(\cF)$, the convex hull of $\cF$, which sometimes improves the obtainable rates. 
We briefly return to this issue in Example~\ref{ex:heavy-tailed-regression-II} and Section~\ref{sec:related-work}.

\clearpage
\label{sec:gloss}
\ifdefined\glossaryON
\setlength{\glsdescwidth}{0.67\hsize}
\setlength{\glspagelistwidth}{0.05\hsize}
\renewcommand*{\glossarypreamble}{\vspace{-\baselineskip}}
\renewcommand*{\pagelistname}{Page}
{ \small
\renewcommand*{\arraystretch}{1.2}
\printglossary[style=long3colheaderborder,title={}]
}
\fi

\glsadd{general}
\glsadd{divergences}
\glsadd{pseudo}
\glsaddcond{conditions}

\section{Setting, Technical Preliminaries, Global Assumptions}
\label{sec:extended-intro}

We now formally introduce the problem setting, cover some
preliminaries, and state the assumptions used throughout this work. A
glossary appearing on this page and the last one describes all
frequently used symbols and conditions.

Let \glsadd{loss-of-f} $\loss_f(z) := \loss(f,z) \in \reals \cup \{\infty \}$ denote the loss of action
$f \in \cF$ under outcome $z \in \cZ$.  In the classical
statistical learning problems of classification and regression with
i.i.d.~samples, we have $\cZ = \cX \times \cY$.  Classification (0-1 loss) is recovered by taking $\cY = \{0, 1\}$ and
$\loss_f(x, y) = |y - f(x)|$, and we obtain regression with squared
loss by taking $\cY = \reals$ and $\loss_f(x, y) = (y - f(x))^2$. In either case, the class $\cF$ is
some subset of the set of all functions $f: \cX \rightarrow \cY$, such as the set of decision trees of depth at most 5 for classification. 
Our setting also includes conditional density estimation (see Example~\ref{ex:cde}).
Unless we explicitly state otherwise, whenever we introduce a random
variable we assume it is a function of $Z, Z_1, \ldots, Z_n$ which are
i.i.d.~$\sim \Pt$. If we write $\loss_f$ we mean $\loss_f(Z)$. 

While in frequentist statistics one mostly considers learning algorithms 
(often called ``estimators'') that always output a single $f \in \model$, 
we also will consider algorithms that output {\em distributions\/} on $\model$. 
Such distributions can, but need not, be Bayesian or generalized Bayesian posteriors 
as described below. 
Formally, a learning algorithm based on a set of predictors $\model$ is a function 
\glsadd{dolest}
$\dolest: \bigcup_{n=0}^\infty \cZ^n \rightarrow
\Delta(\model)$, where $\Delta$ is the set of distributions on
$\model$. 
 The output of algorithm $\dolest$ based on sample $Z^n$ is written as
$\dol \mid Z^n$
and abbreviated to \glsadd{dol-n} $\dol_n$. $\dol_n$ is a function of $Z^n$ and
hence a random variable under $\Pt$. For fixed given $z^n$, $\dol \mid
z^n$ is a measure on $\cF$. 
Importantly, our learning algorithms are always defined
such that they can also output a distribution $\Prior$ based on an
empty data sequence; we may think of this as a ``prior'' guess of
$f$. We explain below how to recast standard estimators such as ERM,
for which $\Prior$ is undefined, in this framework. 
Whenever we consider a distribution $\dol$ on $\model$ for a
problem $\smtuple$, we denote its outcome, a random variable, as
$\rv{f}$.
Whenever we compare the performance of a learning algorithm
$\dolest$ to a fixed $\tilde{f} \in \cF$, we call
$\tilde{f}$ a {\em comparator}.
$\tilde{f}$ is called \emph{optimal} or {\em risk-minimizing\/} if $\E
[ \loss_f(Z) - \loss_{\tilde{f}}(Z) ] \geq 0$ for all $f \in
\model$; under the assumptions below, this expectation is always
well-defined. We usually (but not in Section~\ref{sec:grip} and the
proofs) take as our comparator $\tilde{f} = f^*$, where
$f^*$ is a risk minimizer. Whenever this cannot cause confusion, we
write $L_f = \ell_f -
\ell_{f^*}$ for the {\em excess loss} relative to $f^*$.
\paragraph{Assumptions on Learning Algorithms $\dolest$}
Whenever in the sequel we mention a learning algorithm $\dolest$, we make the following (very mild) assumptions: (1) for all $n$, $z^n \in \cZ^n$, $\dol_n$ has a density
$\postdens_n \equiv \postdens \mid z^n$ relative to the prior
distribution $\Prior$; (2) $\Prior$ satisfies the natural requirement that  for all $z \in \cZ$, $\Prior(f \in \cF: \loss_f(z) < \infty) > 0$. 
\paragraph{Assumptions on and Conventions for Learning Problems
  $\smtuple$}
All of our mathematical results concern learning problems $\smtuple$ for which we invariably make the following assumptions:
\begin{enumerate}\item Unless the loss function  $\loss$ is log-loss or conditional log-loss (see the example below), it is  is uniformly bounded from below in the sense that $\inf_{f \in \cF} \inf_{z \in \cZ} \loss_f(Z) > -\infty$.
\item For (conditional) log-loss, we assume for all $f \in
  \cF$ that
  $p_f$ is a probability density relative to some fixed common dominating
  measure \glsadd{mu} $\mu$, so that $P_f$, the distribution with density $p_f$, is absolutely continuous with respect to $\mu$; we also assume that
  $P$ itself is absolutely continuous with respect to
  $\mu$. Moreover, we additionally assume that
\begin{align}
\KL(P \pipes P_{\fopt}) < \infty \label{eqn:f-star-finite}
\end{align}
and, with $H(P)$ the differential entropy of $P$ relative to $\mu$,
\begin{align}
H(P) > -\infty \label{eqn:entropy-bounded-below} .
\end{align}
\item The learning problem is nontrivial in the sense that for some $f \in \cF$, $\Exp_{Z \sim \Pt}[\loss_{f}(Z)] < \infty$ (we require this irrespective of whether $\loss$ is log-loss). 
\item There exists an optimal $f
  \in \cF$. We fix any one among these (our results hold no matter which we take) and denote it by  $f^*$.
\end{enumerate}
Some of our results continue to hold without the final
  assumption; we shall in all cases say so explicitly. Since we invariably want to impose
  these assumptions, from now on learning problems
  $\smtuple$ are {\em defined\/} to be such that they satisfy these
  four assumptions, and we will not explicitly mention them any more.
  The assumptions, and all other issues concerning
  unboundedness and infinities, are discussed in detail in
  Appendix~\ref{app:infinity-new}. The requirement that the loss is
  bounded from below ensures that there are no issues involving
  undefined expectations or problems with interchanging order of
  expectations, as we show in Appendix~\ref{sec:infinities-ulb}.  It
  holds for just about all loss functions encountered in the
  literature, except for log-loss defined on continuous outcome
  spaces, where the log-loss can be unbounded both from above and
  below; in Appendix~\ref{sec:infinities-log-loss} we motivate the
  requirements we impose on log-loss and show that, while very mild,
  they are still sufficient to make all expectations well-defined.

\begin{myexample}{Conditional Density Estimation}
\label{ex:cde} Let $\cZ = \cX \times \cY$ and let $\{ p_f \mid f \in \model \}$ be a statistical model of conditional densities for $Y \mid X$, i.e., for each $x \in \cX$, $p_f(\cdot \mid x)$ is a probability density on $\cY$ relative to a fixed underlying measure $\mu$. 
Take (conditional) \emph{log loss}, defined on outcome $z = (x,y)$ as $\loss_f(x, y) = -\log p_f(y \mid x)$. 
The excess risk, now 
$\xsrisk{f} = \E_{Z \sim \Pt} \Bigl[ \log \frac{p_{\fopt}(Y \midb X)}{p_f(Y|X)} \Bigr]$, 
is formally equivalent to the \emph{generalized KL divergence}, 
as already defined in the original paper by \cite{kullback1951information} that
also introduced what is now the ``standard'' KL divergence. 
Assuming that $\Pt$ has a density $p$ relative to the underlying measure, and 
denoting standard KL divergence by \glsadd{standard-KL} $\kl$, we have 
$\kl(p \pipes p_f) = \E_{Z \sim \Pt} \left[ \log \frac{p(Y \midb X)}{p_f(Y|X)} \right]$, 
so that
$\xsrisk{f} = \kl(p \pipes p_f) - \kl(p \pipes p_{\fopt})$. 
Thus, minimizing the excess risk under log loss is equivalent to learning a distribution minimizing the KL divergence from $\Pt$ over $\{ p_f : f \in \model\}$. 
We have
$\inf_{f \in \model} \kl(p \pipes p_f) = \kl(p \pipes p_{\fopt}) = \epsilon \geq 0$.
If $\epsilon = 0$, we must have $p_{f^*} = p$, so we deal with a standard {\em well-specified\/} density estimation problem, i.e., the model $\{p_f \mid f \in \model\}$ is ``correct'' and $\fopt \in {\cal F}$ represents the true $\Pt$. 
If $\epsilon > 0$, we still have  $\inf_{f \in \model} \xsrisk{f} = 0$ 
and may view our problem as learning an $f$ that is
closest to $\fopt$ in generalized KL divergence.  
\end{myexample}
\paragraph{Generalized (PAC-) Bayesian, Two-Part, and ERM Estimators}
Although our main results hold for general estimators,
Proposition~\ref{prop:transavia} below indicates that they are
especially suited for generalized Bayesian, two-part MDL, or ERM estimators, 
since these minimize the bounds
provided by our theorems under various constraints. To define these
estimators, fix a distribution \glsadd{prior} $\Prior$ on $\model$,
henceforth called {\em prior}, and a {\em learning rate\/}
$\eta > 0$.  The $\eta$-{\em generalized Bayesian posterior\/} based
on prior $\Prior$, $\model$ and sample $z_{1}, \ldots, z_{n}$ is the distribution
$\dol^{B}_n$ on $f \in \model$, defined by (\ref{eq:bayes}). 
By our requirement that for all $z \in \cZ$, $\Prior(f \in \cF: \loss_f(z) < \infty) > 0$,  (\ref{eq:bayes}) is guaranteed to be well-defined.

Now, given a learning problem as defined above, fix a countable subset
$\ddot{\model}$ of $\model$, a distribution $\Prior$ concentrated on $\ddot{\model}$ and define the
$\eta$-\emph{generalized two-part MDL estimator for prior $\Prior$ at
  sample size $n$}
as
\glsadd{gen-two-part}
\begin{align}\label{eq:twopart}
\twopart{f} := \argmin_{f \in \ddot{\model}} \, \sum_{i=1}^{n} \loss_f(Z_i) 
+ \frac{1}{\eta} \cdot \left( - \log \Prior(\{f \}) \right), 
\end{align}
where, if the minimum is achieved by more than one
$f \in \ddot{\model}$, we take the smallest in the countable list, and
if the minimum is not achieved, we take the smallest $f$ in the list
that is within $1/n$ of the minimum.  Note that the $\eta$-two part
estimator is \emph{deterministic}: it concentrates on a single
function. 
ERM is recovered for finite $\cF$ by setting the prior $\Prior$ to be uniform over $\ddot{\model}$. 
We may view the $\eta$-two part estimator as a learning algorithm $\dolest$ in our sense by defining $\dol_0$ to be the prior on $\ddot{\cF}$ as above and, for each $n$, $\dol_n$ as the distribution that puts all of its mass at $\twopart{f}$ at sample size $n$. 
While we could denote this estimator as ${\Pi}_{|\textsc{2-p}}$, it will be convenient to write $(\twopart{f},\Prior)$ so as to also specify the prior. In the same way, general priors $\Prior$
combined with general deterministic estimators $\estim{f}$ defined for
samples of length $\geq 1$ may be viewed as learning algorithms
$\dolest$ which we will denote as \glsadd{det-rand} $(\estim{f},\Prior)$.

Finally, we formally define the ERM estimator as the
$f \in \cF$ that minimizes $\sum_{j=1}^n \loss_f(Z_j)$; whenever we
refer to ERM we will make sure that at least one such $f$ exists; ties
can then be broken in any way desired.  It is important to note that
ERM can be applied without knowledge of $\eta$; however, for general
two-part and Bayesian estimators we need to know $\eta$ --- we return
to this issue in Section~\ref{sec:related-work}.

\section{Annealed Risk, ESI, and Complexity}
\label{sec:zhang}
In this section we present Lemma~\ref{lem:rascbound-simple}, a
PAC-Bayesian style bound that underlies all our results to follow.
Remarkably, it holds without any regularity conditions. However, on
the left hand side it has an ``annealed'' version of the risk rather
than the actual risk. In Sections~\ref{sec:strongcentral},~\ref{sec:witness},~and~\ref{sec:grip}
we give conditions under which the annealed risk can be replaced by
either a Hellinger-type distance or the standard risk, which is what
we are really interested in. Lemma~\ref{lem:rascbound-simple} relates
the annealed risk to an information complexity via {\em exponential
  stochastic inequality\/} (ESI). We now introduce the technical
notions of annealed expectation and ESI. We then present
Lemma~\ref{lem:rascbound-simple} and discuss its right-hand side, the
information complexity.  We do not claim any novelty for the technical
results in this section --- the lemma below can be found in
\citep{zhang2006information,zhang2006epsilon}, for example. Still, we
need to treat these results in some detail to prepare the new results
in subsequent sections.
\subsection{Main Concepts: Annealed and Hellinger Risk, ESI}
For $\eta >0$ and general random variables $U$, we define, respectively, the {\em Hellinger-transformed expectation\/} and the  
\emph{annealed expectation} (terminology from statistical mechanics;
see e.g.~\citep{haussler1996rigorous}), also known as \emph{R\'enyi-transformed
  expectation} (terminology from information theory, see
e.g.~\Citep{erven2014renyi}) as \glsadd{hellinger-E} \glsadd{annealed-E}
\begin{align}\label{eq:genren}
\Exphel{\eta} \left[U \right] 
:= \frac{1}{\eta} \left(1- \E \left[e^{-\eta U} \right]\right) \ \ ; \ \ 
\Expann{\eta} \left[U \right] 
:= -\frac{1}{\eta} \log \E \left[e^{-\eta U} \right],
\end{align}
with $\log$ the natural logarithm. We will frequently use that for
$\eta > 0$,
\begin{align}\label{eq:genrenc}
\Exphel{\eta} \left[U \right]  \leq \Expann{\eta} \left[U \right]
\leq \E[U]
\end{align}
where the first inequality follows from $- \log x \geq 1-x$ and the
second from Jensen. We also note that if, for example, $U$ is bounded,
then the inequalities become equalities in the limit:
\begin{proposition} \label{prop:aff-div-results}
If $\E [ e^{-\eta X} ] < \infty$, we have $ \lim_{\eta
    \downarrow 0} \Exphel{\eta} [ X ] = \E [ X ]$ and we also have
  that $\eta \mapsto \Expann{\eta} [ X ]$ is non-increasing.
\end{proposition}
All our results below may be expressed succinctly via the notion of
\mbox{\emph{exponential stochastic inequality}}.
\begin{definition}[Exponential Stochastic Inequality (ESI)] 
\label{def:esi} Let  $\eta > 0$ and let $U, U'$ be random variables on some probability space with probability measure $P$. We define \glsadd{ESI}
\begin{align}\label{eq:esi}
U \stochleq_{\eta} \ \ U'  \ \ \Leftrightarrow \ \ {\bf E}_{U, U' \sim P} \left[e^{\eta (U- U')} \right] \leq 1 .
\end{align}
\end{definition}
In all our applications of this notation, $P$ is the distribution appearing in a given
learning problem $\smtuple$ that will be clear from the context; hence,
we omit it in the ESI notation.
An ESI simultaneously captures ``with (very) high probability'' and ``in expectation'' results.
\begin{proposition}[ESI Implications] \label{prop:drop} 
For all $\eta > 0$, if $U \stochleq_{\eta} U'$ then, 
(i), $\E [ U ] \leq \E [ U' ]$; 
and, (ii), for all $K >0$, with $P$-probability at least $1- e^{-K}$, 
$U \leq U' + K/\eta$ (or equivalently, for all $\delta \geq 0$, with probability at least $1- \delta$, $U \leq U'+ \eta^{-1} \cdot \log (1/\delta)$).
\end{proposition} 
\begin{Proof} 
  Jensen's inequality yields (i). 
  Apply Markov's inequality to $e^{-\eta (U - U')}$ for (ii).
\end{Proof}
The following proposition will be extremely convenient for our proofs: 
\begin{proposition}[Weak Transitivity] \label{prop:esi-transitive} 
  Let $(U, V)$ be a pair of random variables with joint distribution $P$. 
  For all $\eta > 0$ and $a, b \in \reals$, 
  if $U \stochleq_\eta a$ and $V \stochleq_\eta b$, 
  then $U + V \stochleq_{\eta / 2} a + b$.
\end{proposition}
\begin{Proof}
From Jensen's inequality:
$\E [ e^{\frac{\eta}{2} ((U - a) + (V - b))} ] 
\leq \frac{1}{2} \E [ e^{\eta (U - a)} ] + \frac{1}{2} \E [ e^{\eta (V - b)} ].$
\end{Proof}

\subsection{PAC-Bayesian Style Inequality}
\label{sec:subzhang}
All our results are based on the following lemma due to \cite{zhang2006information}:
\begin{lemma}\label{lem:rascbound-simple} 
  Let $\smtuple$ represent a learning problem
  with $L_f$ the excess loss relative to an optimal $f^*$.
  Let $\dolest$ be a learning algorithm (defining a ``prior'' $\Prior$) for this learning problem that outputs distributions on $\cF$. 
For all $\eta > 0$, $n \in \naturals$, we have:  
\begin{align}\label{eq:mainrascbound-simple}
\E_{\rv{f} \sim \dol_n} \left[ \Expann{\eta}_{Z \sim \Pt} \left[ \xsloss{\rv{f}} \right] \right]
 \stochleq_{\eta \cdot n }  \rsc_{n,\eta} \left( \nocomparator  \dolest \right) .
\end{align}
where $\rsc_{n,\eta}$ is the information complexity, defined as: \glsadd{IC}
\begin{align}\label{eqn:rsc} 
\rsc_{n,\eta}(\nocomparator \dolest) 
& := \Exp_{\rv{f} \sim \dol_n} \left[ \frac{1}{n} \sum_{i=1}^n \xslossat{\rv{f}}{Z_i} \right]
        + \frac{\KL( \dol_n \pipes \Prior)}{\eta \cdot n  } .
\end{align}
\end{lemma}
By the finiteness considerations of Appendix~\ref{app:infinity-new},
$\rsc_{n,\eta}(\nocomparator \dolest)$ is always well-defined but may
in some cases be equal to $- \infty$ or $\infty$.  We prove a
generalized form of this result, which does not require existence of
an optimal $f^*$, in Appendix~\ref{app:rascboundproof} 
  The proof is essentially taken from the proof of Theorem 2.1 of
  \cite{zhang2006information} and is presented only for completeness.

This result is similar to various results that have been called \emph{PAC-Bayesian inequalities}, 
although this name is sometimes reserved for a different type of
inequality involving an empirical (observable) quantity on the right
that does not involve $\fopt$ \citep{McAllester02}. 
Lemma~\ref{lem:rascbound-simple} generalizes earlier in-expectation results by Barron and \cite{li1999estimation} for deterministic estimators rather than
(randomized) learning algorithms; these in-expectation results further
refine in-probability results of \cite{barron1991minimum}, arguably
the starting point of this research.

To explain the potential usefulness of
Lemma~\ref{lem:rascbound-simple}, let us weaken
(\ref{eq:mainrascbound-simple}) to an in-expectation statement via
Proposition~\ref{prop:drop}, so that it reduces to:
\begin{align}\label{eq:simplifieda}
\Exp_{Z^n \sim P} \left[\Exp_{\rv{f} \sim \dol_n} \left[ \Expann{\eta}[\xsloss{{\rv{f}}}] \right] \right]
\leq \E_{Z^n \sim P} \left[ \rsc_{n,\eta} \left( \nocomparator \dolest \right) \right].
\end{align}
If the annealed
expectation were  a standard expectation, the left-hand side would be
an expected excess risk. Then we would have a
great theorem: by \eqref{eq:simplifieda}, the lemma bounds the
expected excess risk of estimator $\dolest$ by a complexity term,
which, as we will see below, generalizes a large number of previous
complexity terms (and allows us to get the same rates), both for
well-specified density estimation and for general loss functions. The
nonstandard inequality $\stochleq$ implies that we get such bounds not
only in expectation but also in probability.  The only problem is that
the left-hand side in Lemma~\ref{lem:rascbound-simple} is not the
standard risk but the annealed risk, which is always smaller and can
even be negative.  It turns out however that --- as already suggested,
but not proved by Proposition~\ref{prop:aff-div-results} --- by making
$\eta$ small enough, the left-hand side can in many cases be related
to the standard excess risk or another divergence-like measure after
all.
The conditions which allow this are the subject of
Sections~\ref{sec:strongcentral}--\ref{sec:grip}; but first, in the
remainder of the this section we study the complexity term in
detail.

\subsection{Information Complexity}
\label{sec:complexity}

The present form of the information complexity is due to
\cite{zhang2006information}, with precursors from
\cite{Rissanen89,barron1991minimum,yamanishi1998decision}.  For
generalized Bayesian, two-part MDL and standard ERM, a first further
bound is given via the following proposition, the first part of which
is also from \cite{zhang2006information}; we note that this result can be extended to the generalized definition of $\rsc_{n,\eta}$ given in Section~\ref{app:rascboundproof}; the extended result does not rely on the existence of $f^*$.
\begin{proposition}\label{prop:transavia}
  Consider a learning problem $\smtuple$ and let
  $Z^n \equiv Z_1, \ldots, Z_n$ be any sample with
  $\sum_{i=1}^n \loss_{\fopt}(Z_i) < \infty$ (this will hold a.s. if
  $Z^n \sim P$). Let $\Prior$ be a distribution on ${\cal F}$. and let
  $\dolest^{\text{B}}$
  be the corresponding $\eta$-generalized Bayesian posterior,
  with, for each
    $n$, $\postdens^B_n$ given by \eqref{eq:bayes}.  We have for all
  $\eta >0$ that $\rsc_{n,\eta}(\nocomparator \dolest^{\text{B}})$ is
  non-increasing in $\eta$, and that
\begin{align}
n\/ \cdot \/ \rsc_{n,\eta}(\nocomparator \dolest^{\text{B}}) 
= \ & n \/ \cdot \/ \inf_{\dolest \in \textsc{RAND}} \rsc_{n,\eta}(\nocomparator \dolest )
\label{eqn:rscmix}
=   - \frac{1}{\eta} \log
 \Exp_{\rv{f} \sim \Prior} 
 \exp \left(- 
 \eta \sum_{i=1}^n \xslossat{\rv{f}}{Z_i}  \right)
 \\
\leq \ & \inf_{A}\ \left\{ 
- \frac{1}{\eta} \log  \Prior(A)
+ n\/ \cdot \/ \rsc_{n,\eta}(\nocomparator \dolest^{\text{B}} \mid f \in A)  \right\}
\label{eqn:rscmixpre}
\\ \leq \ & \inf_{A}\ \left\{ 
- \frac{1}{\eta} \log  \Prior(A)
+  \Exp_{\rv{f} \sim \Prior | A } 
\left[ \sum_{i=1}^n \xslossat{\rv{f}}{Z_i}  \right] \right\},
\label{eqn:rscmixb}
\end{align}
where \textsc{RAND} is the set of \emph{all} learning algorithms
$\dolest'$ that can be defined relative to $\smtuple$ with $\dol'_0 =
\Prior$ and the second infimum is over all measurable subsets $A \subseteq
\cF$. In the special case that 
$\Prior$ has
countable support
$\ddot{\cF}$ so that the $\eta$-two part estimator (\ref{eq:twopart})
is defined, we further have
\begin{align} 
n\/ \cdot \/ \rsc_{n,\eta}(\nocomparator \dolest^{\text{B}}) 
&\leq n\/ \cdot \inf_{\dot{f} \in \textsc{DET}} \rsc_{n,\eta}(\nocomparator (\dot{f},\Prior) ) \label{eqn:twopartcodemix} \\
&= n \/ \cdot \rsc_{n,\eta}(\fopt  \pipes (\twopart{f},\Prior)) \leq 
\inf_{f \in \ddot{\cF}}\ \left\{ 
- \frac{1}{\eta} \log  \Prior(\{f \})
+ 
\sum_{i=1}^n \xslossat{f}{Z_i} \right\} , \nonumber
\end{align}
where \textsc{DET} is the set of \emph{all} deterministic estimators with range $\ddot{\cF}$. 
\end{proposition}
From Lemma~\ref{lem:rascbound-simple} and this result, we see that we
have three equivalent characterizations of information complexity for
$\eta$-generalized Bayesian estimators.  First, there is just the basic definition
\eqref{eqn:rsc} with $\dol_n$ instantiated to the $\eta$-generalized Bayesian
posterior.  Second, there is the characterization as the minimizer of
\eqref{eqn:rsc} for the given data, over all distributions $\dol_n$
on $\cF$. 
And third, there is the characterization in terms of a generalized
Bayesian marginal likelihood: \eqref{eqn:rscmixb} shows that for
$\eta = 1$ and $\loss$ the log loss, the information complexity
$\rsc_{n,\eta}(\nocomparator \dolest^{\text{B}})$ is the log Bayes
marginal likelihood of the data relative to $\fopt$, divided by $n$.
If furthermore  $\cF$ is a sufficiently regular $k$-dimensional parametric
probability model equipped with a prior $\Prior$ with full support on
$\cF$, and the model is correct, i.e., $Z_1, Z_2, \ldots$ are sampled
i.i.d.~from a distribution with density in $\cF$, then, as is
well-known, the information complexity will almost surely coincide, up
to $O(1/n)$, with the BIC penalty:
$n\/ \cdot \/ \rsc_{n,\eta}(\nocomparator \dolest^{\text{B}}) = (k/2)
\log n + O(1)$; see \cite{Grunwald07} for precise results.

\subsubsection{Bounds on Information Complexity for $\eta$-Generalized Bayes}
\cite{GhosalGV00} (GGV from now on) presented
several theorems implying concentration of the (standard) Bayesian
posterior around the true distribution in the well-specified i.i.d.~case;
their results were employed in many subsequent papers such as, for example, \citep{ghosal2007convergence,GhosalLV08,bickel2012semiparametric}. 
We compare our results to theirs in
Example~\ref{ex:ggv} in Section~\ref{sec:strongcentral}. 
One of the conditions they impose is the
existence of a sequence $(\epsilon_n)_{n \geq 1}$ such that $n \epsilon_n^2 \rightarrow \infty$, and, for some constant $C>0$, 
for all $n$, a certain $\epsilon^2_n$-ball around the true distribution has
prior mass at least $\exp(-n C \epsilon_n^2)$. Generalizing from log
loss to arbitrary loss functions, their condition reads
\begin{equation}\label{eq:ggva}
\Prior\left( f: \xsrisk{f} \leq \epsilon_n^2 \ ; \ 
\E\left(\xsloss{f} \right)^2 \leq \epsilon_n^2
 \right) \geq e^{-n C \epsilon_n^2}.
\end{equation}
They then show that, under this and further conditions, the posterior
concentrates with Hellinger rate $\epsilon_n$ 
(see Example~\ref{ex:ggv} of Section~\ref{sec:strongcentral} for the precise meaning). 
Now note that (\ref{eq:ggva}) implies the weaker
\begin{equation}\label{eq:ggvb}
\Prior\left( f: \xsrisk{f} \leq \epsilon_n^2 \right) \geq e^{-n C \epsilon_n^2},
\end{equation}
which in turn implies, via \eqref{eqn:rscmixb},  for any
$0 < \eta \leq 1$, the following bound on $\rsc$ for the  
$\eta$-generalized Bayesian estimator:
\begin{equation}\label{eq:ggvc}
\E_{Z^n \sim P}  \left[ \rsc_{n,\eta} \left( \nocomparator  \dolest \right)\right]
\leq \epsilon_n^2 \cdot (1+ (C/\eta)),
\end{equation}
To see this, note that \eqref{eqn:rscmixb} and \eqref{eq:ggvb} imply  
\begin{align}\label{eqn:rscmixrewrite}
& \rsc_{n,\eta}(\nocomparator \dolest^{\text{B}}) 
\nonumber \\ 
&\leq 
- \frac{1}{n} \sum_{i=1}^n \loss_{\fopt}(Z_i) 
- \frac{1}{n \eta} \log  \Prior\{  f: \xsrisk{f} \leq \epsilon_n^2 \}
+ \frac{1}{n}  \Exp_{\rv{f} \sim \Prior \mid \{  f: \xsrisk{f} \leq \epsilon_n^2 \} } 
\left[ \sum_{i=1}^n \left(\loss_{\rv{f}}(Z_i)  \right) \right] \nonumber \\ 
&\leq 
C \frac{\epsilon_n^2}{\eta} + \frac{1}{n}  \Exp_{\rv{f} \sim \Prior \mid \{  f: \xsrisk{f} \leq \epsilon_n^2 \} } 
\left[ \sum_{i=1}^n \left(\xslossat{\rv{f}}{Z_i} \right) \right] .
\end{align}
This implies (\ref{eq:ggvc}). 

All the examples of nonparametric families provided by GGV (including
priors on sieves, log-spline models and Dirichlet processes) rely on
showing that condition (\ref{eq:ggva}) above holds for specific
priors, and hence in all these cases we get bounds on the
expected-information complexity which, by (\ref{eq:simplifieda})
allows us to establish comparable rates in expectation for the
$\eta$-generalized Bayesian estimator in the well-specified case, for
any $\eta$ such that the left-hand side can be linked to an actual
distance measure --- see Example~\ref{ex:ggv} in Section~\ref{sec:strongcentral}.

We also would like to bound the excess risk in probability in terms of the
expected information complexity. 
For this, we can proceed in either of two ways: 
we either start with an expectation bound such as (\ref{eq:simplifieda})
and then use Markov's inequality (since the excess risk of any
estimator is a.s. nonnegative) to go back from expectation to
in-probability. However, under GGV's condition (\ref{eq:ggva})
(the weaker (\ref{eq:ggvb}) is
not sufficient here), we can also use the in-probability version
of Lemma~\ref{lem:rascbound-simple} directly. In combination with  
Lemma 8.1 of GGV (which straightforwardly
extends to our setting with general loss and $\eta$) this 
implies that for all $\delta > 0$:
\begin{equation}\label{eq:antwerp}
P\left( \rsc_{n,\eta} \left( \nocomparator \dolest \right) \geq \left(1+ \delta^{-1/2}\right) \epsilon_n^2\right) \leq \frac{\delta}{n
  \epsilon_n^2}.
\end{equation}
It follows that under (\ref{eq:ggva}), since $n \epsilon_n^2
\rightarrow \infty$, $\epsilon_n^2$ is, up to constant factors
depending on $\delta$, an upper bound both on $\E \left[\rsc_{n,\eta}
  \left( \nocomparator \dolest \right)\right]$, and, for every
$\delta$, with probability at least $1- \delta$, on $\rsc_{n,\eta}
\left( \nocomparator \dolest \right)$ --- see the discussion below
Theorem~\ref{thm:main-unbounded} in Section~\ref{sec:grip}.

Finally, there often exist nontrivial worst-case (sup norm) or
almost-sure bounds on the information complexity; such bounds ---
mostly developed for parametric models but also, e.g., for Gaussian
processes \citep{SeegerKF08} have historically mostly been established
within the MDL literature; see \citep{Grunwald07} for an extensive
overview. While we will not go into such bounds in detail here, below we provide a very simple such bound for countably infinite classes, which shows the ease by which $\rsc$ allows for model aggregation.

Suppose that we have a countably infinite collection of classes $\cF_1, \cF_2, \ldots$ and a corresponding set of priors $\Prior^{(1)}, \Prior^{(2)}, \ldots$.
Let us select a new prior $q: \mathbb{N} \rightarrow \reals^+$ over the collection $\cF := \bigcup_{j \in \mathbb{N}} \cF_j$. 
Then we may define a new prior $\Prior = \sum_{j \in \mathbb{N}} q(j) \Prior^{(j)}$ over $\cF$. We will assume  that the risk minimizer in the full class, $\fopt$, is equal to $\fopt_{j^*}$ for some $j^* \in \mathbb{N}$. By Proposition~\ref{prop:transavia}, Eq. (\ref{eqn:rscmixpre}), we must now have, for all data $Z_1, \ldots, Z_n$, that 
\begin{align}\label{eq:preaggregation}
 n \cdot  \rsc_{n,\eta} \left( \nocomparator  \dolest \right)
\leq - \frac{1}{\eta} \log q(j^*) + n \cdot \rsc_{n,\eta}( \nocomparator \dol \mid f \in \cF_{j^*} ),
\end{align}
where $\dol \mid f \in \cF_{j^*}$ is the $\eta$-generalized Bayesian
estimator based on the prior $\Prior^{(j^*)}$ within $\cF_{j^*}$.

If we now further assume that, for each $j$, the GGV-type condition \eqref{eq:ggvb} is satisfied (with prior $\Prior^{(j)}$ and with $\fopt$ replaced by $\fopt_j$, the risk minimizer over $\cF_j$), then taking expectations in (\ref{eq:preaggregation}) 
implies that \eqref{eq:ggvb} holds for $\Prior$, with $\fopt = \fopt_{j^*}$,  with the RHS scaled by a factor $q(j^*)$. A simple adaptation of \eqref{eq:ggvc} then gives \begin{align}\label{eq:aggregation}
\E_{Z^n \sim P} \left[ 
  \rsc_{n,\eta} \left( \nocomparator  \dolest \right)
\right] 
\leq \epsilon_n^2 \cdot (1+ (C/\eta)) + \frac{-\log q(j^*)}{n \eta} .
 \end{align}
Thus, the overhead in information complexity for combining the classes is simply $\frac{-\log q(j^*)}{n \eta}$. Moreover, in the case of a finite collection of $M$ classes, we may take $q$ uniform and the overhead becomes $\frac{\log M}{n \eta}$.

\section{The Strong Central Condition}
\label{sec:strongcentral}
As we explained below Lemma~\ref{lem:rascbound-simple}, our strategy
in proving our theorems will be to determine conditions under which
the $\eta$-annealed excess risk is similar enough to either the
standard risk or a meaningful weakening thereof for
Lemma~\ref{lem:rascbound-simple} to be useful. In this section we
present the simplest such condition, which is still quite strong ---
it requires an exponentially small upper tail of the distribution of
$\loss_{\fopt} - \loss_f$. This \emph{strong central} condition has a
parameter $\bar\eta > 0$, and whenever we want to make this explicit
we refer to it as ``the $\bar\eta$-central condition''. \emph{Intuitively}, its usefulness for learning is obvious: it ensures that the probability that a ``bad'' $f$ outperforms $\fopt$ by more than $L$ is exponentially small in $L$. \emph{Technically}, its use is that it ensures that the
annealed risk is positive for all $\eta < \bar\eta$. This allows us to
turn Lemma~\ref{lem:rascbound-simple} into a useful result by
replacing its left-hand side by a metric which (for log loss)
generalizes the squared Hellinger distance.

\subsection{Definitions and Main Results}
We now turn to the strong central
condition, which, along with its weakened versions discussed in
Section~\ref{sec:grip} was introduced by \Citet{erven2015fast}.
\begin{definition}[Central Condition]\label{def:central}
Let $\bar{\eta} > 0$. We say that $\smtuple$ satisfies the 
\emph{strong $\bar\eta$-central condition} if there exists some $\tilde{f} \in \cF$ such that \glsaddcond{strong-central}
\begin{align} \label{eqn:central}
\E\left[e^{-\bar\eta (\loss_{f} - \loss_{\tilde{f}})}  \right] \leq 1, \ \text{i.e.,}\ \loss_{\tilde{f}} - \loss_f \stochleq_{\bar\eta} 0 \qquad \text{for all } f \in \cF.
\end{align}
\end{definition}
Jensen's inequality implies that if a $\tilde{f}$ exists satisfying
(\ref{eqn:central}), it must be optimal; hence we can take
$\tilde{f} = f^*$.  The special case of this condition with
$\bar{\eta} = 1$ under log loss has appeared previously, often
implicitly, in works studying rates of convergence in density
estimation \Citep{barron1991minimum,li1999estimation,zhang2006epsilon,kleijn2006misspecification,grunwald2011safe}. 
For details about the myriad of implications of the central condition
and its equivalences to other conditions we refer to
\Citet{erven2015fast}. Here we merely highlight the most important
facts. First, trivially, the strong central condition automatically
holds for density estimation with log loss in the well-specified
setting since then $p_{\fopt}$ is the density of $P$ (see
Example~\ref{ex:cde}), as we then have
\begin{equation}\label{eq:preciselyone}
\E_{Z \sim P}\left[ e^{- \bar\eta (\loss_{f} - \loss_{\fopt})} \right] = 
\E_{Z \sim P}\left[ \frac{p_{f}(Z)}{p_{\fopt}(Z)} \right] = 1
\end{equation}
Second, less trivially, it also automatically holds under a convex model 
in the misspecified setting (see \citet{li1999estimation} and Example 2.2 of \Citet{erven2015fast}). 
Third, for classification and other bounded 
excess loss cases, it can be related to the \emph{Massart condition}, a
special case of the Bernstein condition
\citep{audibert2004pac,bartlett2006empirical} (as discussed
immediately before Definition~\ref{def:bernstein} in Section~\ref{sec:witness}).

We now introduce a new metric which is derived from the Hellinger metric, introduced below (as is common) in terms of its square.
\begin{definition}[Misspecification Metric] \label{def:misspec-metric}
 For a given  learning problem $\smtuple$, associate each $f \in \cF$ and $\eta > 0$ with a probability density \glsadd{entropified}
\begin{equation}
p_{f,\eta}(z) 
:= p(z) \frac{\exp(- \eta \xsloss{f}(z))}{\E[\exp(-\eta \xsloss{f}(Z))]} ,
\end{equation} 
where $p$ is the density of $P$. Now define $\dhel_{\bar\eta}(f,f')$ as the Hellinger distance between  $p_{f,\bar\eta}$ and $p_{f',\bar\eta}$: \glsadd{mm}
\begin{align}\label{def:mm}
\dhel^2_{\bar\eta}(f,f') & := \frac{2}{\bar{\eta}} \left(1- \int \sqrt{p_{f,\bar\eta}(z) p_{f',\bar\eta}(z)} d \mu(z) \right)
\nonumber \\
& = \Exphel{\bar\eta/2} \left[\xsloss{f}  - \Expann{\bar\eta}\left[ \xsloss{f} \right] 
+ \xsloss{f'} -  \Expann{\bar\eta}\left[ \xsloss{f'} \right] 
\right].
\end{align}
\end{definition} 
The following result is obvious: 
\begin{proposition}\label{prop:hellingerparty}
  If $\loss$ is log loss and ${\cal F}$ is well-specified relative to
  $P$ we can take $\bar\eta=1$ and then for every $f \in \cF$,
  $\dhel^2_{\bar\eta}(\fopt,f)$ coincides with the {standard squared Hellinger
  distance $\Hell_{1/2}(P_{\fopt} \| P_f)$ defined by \glsadd{standard-hellinger}
  $\Hell_{1/2}(P_{f} \|P_{f'}):= 2 \left(1- \int \sqrt{p_{f}(z)
      p_{f'}(z)} d \mu(z) \right).$}
\end{proposition}
Since $\dhel_{\bar\eta}$ is always interpretable as a Hellinger
distance, it is clearly a metric.  This is different from an existing,
more well-known generalization of the Hellinger distance for the
well-specified case \citep{sason2016f}, \glsadd{gen-hellinger} 
$\Hell_\eta(P \pipes Q) :=
\eta^{-1} \left(1- \E_{Z \sim P} \left(q(z)/p(z)
  \right)^{\eta}\right)$ which does not define a metric except for
$\eta = 1/2$ (and then coincides with $\dhel_{1}$). The
$\dhel_{\bar\eta}$ metric is of interest in the misspecified
density estimation setting --- with density estimation, we may not
necessarily be interested in log loss prediction and a metric weaker
than excess risk (i.e.~generalized KL divergence) may be sufficient for our
purposes. With other loss functions, the main interest will usually be
learning an $\hat{f}$ with small prediction error. Then the metric
above, while still well-defined, may not be appropriate, and one is
interested in the excess risk bounds of the next section instead.
\begin{theorem}\label{thm:metric} 
Suppose that the $\bar{\eta}$-strong central condition holds. Then for any $0 < \eta < \bar\eta$, the metric $\dhel_{\bar\eta}$ satisfies
\begin{align*}
\E_{\rv{f} \sim \dol_n} \left[ \dhel^2_{\bar\eta}(\fopt, \rv{f}) \right] 
\stochleq_{\eta \cdot n } C_{\eta} \cdot \rsc_{n,\eta} 
\left( \nocomparator \dolest \right) ,
\end{align*}
with $C_{\eta} = \eta / (\bar\eta- \eta)$. 
In particular, $C_{\eta} < \infty$ for $0 < \eta < \bar\eta$, and $C_{\eta} = 1$ for $\eta = \bar\eta/2$.
\end{theorem}

\begin{myexample}{Comparison to results by GGV} \label{ex:ggv}
  Following \citep{zhang2006epsilon} we illustrate the considerable
  leverage provided in the well-specified density estimation case by
  allowing $\eta$-generalized Bayesian estimators for $\eta < 1$.  GGV
  show that for the standard Bayesian estimator, under condition
  (\ref{eq:ggva}) (which only refers to {\em local\/} properties of
  the prior in neighborhoods of the true density $p_{\fopt}$), in
  combination with a rather stringent {\em global\/} entropy condition,
  the following holds: there exists a constant $C'$ such that
  $\dol_n\left(f \in \cF: \dhel^2_{1}(\fopt,f) > C'
    \epsilon_n^2\right) \rightarrow 0$ in $P$-probability, i.e., for
  every $B> 0$,
\begin{align*}
P\left( \dol_n\left(f \in \cF: \dhel^2_{1}(\fopt,f) > C' \epsilon_n^2\right)
> B  \right) \rightarrow 0 .
\end{align*}
Now, suppose the model is correct so that the $\bar\eta$-central
condition holds for $\bar\eta=1$. Then we get from
Theorem~\ref{thm:metric} that for any $\eta < \bar\eta$, using {\em
  only\/} condition (\ref{eq:ggvb}), the following holds: for any
$\gamma_1, \gamma_2, \ldots$ such that $\gamma_n/\epsilon_n
\rightarrow \infty$, the generalized Bayesian estimator satisfies
$\dol_n\left(f \in \cF: \dhel^2_{1}(\fopt,f) > C' \gamma_n^2\right)
\rightarrow 0$ in $P$-probability, i.e., for every $B> 0$,
\begin{equation}\label{eq:vdv}
P\left( \dol_n\left(f \in \cF: \dhel^2_{1}(\fopt,f) > C' \gamma^2_n \right)
> B \right)  \rightarrow 0,
\end{equation}
as immediately follows from applying Markov's inequality twice as done
below.  Thus, by taking $\eta < 1$ we need neither the stronger
condition \eqref{eq:ggva} nor the much stronger GGV global entropy condition; 
for this we pay only a slight price since our bound is not in terms of $\epsilon_n^2$ 
but is instead in terms of $\gamma_n^2$, which we have to take slightly larger (a factor $\log \log n$ is of course sufficient). 
Under well-specification, we thus obtain the same rates
as GGV for all the statistical models they consider, up to a $\log
\log n$ factor; as GGV show, these rates are usually minimax optimal.
Interestingly, other works on Bayesian and MDL nonparametric
consistency for the well-specified case also consider $\eta
< 1$ \citep{barron1991minimum,zhang2006epsilon,WalkerH02,martin2017empirical} or
invoke an alternative stringent condition to deal with $\eta=1$
(\cite[Section 5.2]{zhang2006epsilon}, \cite{barron1999consistency});
see \cite{zhang2006epsilon} for a very detailed discussion. While it
may be argued that one should be able to deal with standard Bayes
$(\eta = 1)$, in this paper we also aim to deal with misspecification
where we need to take 
$\eta < 1$ (and cannot take it arbitrarily close to $1$) even for simple problems
\citep{GrunwaldO17}, and then there is no special reason to handle
$\eta=1$ via additional conditions.

To show (\ref{eq:vdv}), note that, if the $\bar\eta$-central condition
holds, then for general $A, B > 0$, we have
\begin{align*}
P\left( \dol_n(f \in \cF: \dhel^2_{\bar\eta}(\fopt,f) > A) > B \right) 
&\leq B^{-1} \E_{Z^n} \left[ \dol_n(f \in \cF: \dhel^2_{\bar\eta}(\fopt,f) > A)\right] \\
\leq (AB)^{-1} \E_{Z^n} \E_{\rv{f} \sim \dol_n} \left[\dhel^2_{\bar\eta}(\fopt,\rv{f})\right] 
&\leq (AB)^{-1} \E_{Z^n} \left[ \rsc_{n,\bar\eta/2} \left( \nocomparator \dolest \right)\right],
\end{align*}
where we applied Markov's inequality twice, and 
the final inequality is from Theorem~\ref{thm:metric}. 
Plugging in $A = C' \gamma_n^2$ and $\epsilon_n^2 \geq  \E  \left[ \rsc_{n,\bar\eta/2} \left( \nocomparator  \dolest \right)\right]$ (using \eqref{eq:ggvc}), this can be further bounded as $B^{-1} \epsilon_n^2/\gamma_n^2 \rightarrow 0$. 
\end{myexample}
\subsection{Applying Theorem~\ref{thm:metric} in Misspecified Density
  Estimation}
From the above it is clear that Theorem~\ref{thm:metric} has plenty of
applications whenever the model under consideration is correct. We now
consider applications of Theorem~\ref{thm:metric} to misspecified
models of probability densities $\cF$ with generalized Bayesian
estimators $\dolest^B$. For this we must establish (a) that the central
condition holds for $\cF$, and (b) suitable bounds on the information
complexity relative to $\dolest^B$. As to (a), we know that
the $\bar\eta$-central condition holds for $\bar\eta = 1$ whenever the set of
distributions $\{p_{f} : f \in \cF \}$ is correct or convex;
as shown elsewhere and illustrated in Example~\ref{ex:glm-I} below, 
it also holds for 1-dimensional (nonconvex) exponential
families and high-dimensional generalized linear models (GLMs) under potentially severe misspecification of  the noise, as long as the regression function is well-specified and  $P$ has exponentially small tails. As to (b), we may consider
priors such that in the well-specified case, the GGV condition holds
for some sequence $\epsilon^2_1, \epsilon^2_2, \ldots$ as in Example~\ref{ex:ggv}.
As explained in the example,
the GGV condition then automatically holds for GLMs under misspecification as well, 
so that the same bounds on information complexity can be given as in the well-specified case. 
It appears that this is a special property of GLMs though --- for general $\cF$, we only have
the following proposition which shows that, if the GGV condition holds for some specific prior in the well-specified case with some bounds 
$\epsilon_1, \epsilon_2, \ldots$, 
then, as long as $p_{\fopt}$ dominates $p$, 
it must still hold in the misspecified case for the same prior 
for a strictly larger sequence
$\epsilon'_1, \epsilon'_2, \ldots$,  leading to a potential
deterioration of the bound given by Theorem~\ref{thm:metric}.
\begin{proposition}\label{prop:entrobound}
  Consider a learning problem $\smtuple$ where $\cF$ indexes a set of
  probability distributions 
$\{ P_f : f \in \cF \}$ with densities $p_f$, and suppose that $\sup_{z \in \cZ}
  \frac{dP(z)}{dP_{\fopt}(z)} = C < \infty$. Then for all $f \in \cF$,
\begin{equation}\label{eq:entrobound}
\E_{Z \sim P}[L_{f}] \leq C\cdot \left(
\E_{Z \sim P_{\fopt}}[L_f] + \sqrt{2 \E_{Z \sim P_{\fopt}}[L_f]} \right).  
\end{equation}
\end{proposition}
\begin{proof}
Observe that
\begin{align*}
\E_{Z \sim P}[L_f] \leq \E_{Z \sim P}[0 \opmax L_f] \leq  C \E_{Z \sim P_{\fopt}} [0 \opmax L_f] \leq C \cdot \left( D(\fopt \| f) + \sqrt{2 D(\fopt \| f)} \right),
\end{align*}
where $\E_{Z \sim P_{\fopt}}[L_f] = D(\fopt \| f)$ is the KL divergence
between $\fopt$ and $f$ and the last inequality is from
\cite{yang1998asymptotic} (see the remark under their Lemma
3); for completeness we provide a proof in the appendix.
\end{proof}
As a trivial consequence, whenever the weakened GGV condition
\eqref{eq:ggvb} holds for all $P_f$ with \mbox{$f \in \cF$} for a sequence
$\epsilon_1, \epsilon_2, \ldots$, it will still hold for a sequence
$\epsilon'_1, \epsilon'_2, \ldots$ with
$\epsilon'_j \asymp \sqrt{\epsilon_j}$. It follows from
\eqref{eq:ggvc} that we now automatically have a bound of order
$\epsilon'_n/n$ on the misspecified expected information
complexity. Theorem~\ref{thm:metric} now establishes that whenever the
GGV condition holds in the well-specified case, under the further
(weak) condition that
$\sup_{z \in \cZ} {dP(z)}/{dP_{\fopt}(z)} = C < \infty$, we
automatically get a form of consistency for $\eta$-generalized Bayes,
for $\eta < \bar\eta$. The question whether we get the same rates of
convergence is obfuscated in two ways: first, the misspecification
metric is in general incomparable to the Hellinger metric; second,
even in cases in which the misspecification metric dominates the
standard Hellinger, for nonparametric $\cF$ with
$\E[\rsc_{n,\eta}] \asymp n^{-\gamma}$, the conversion
$\epsilon'_j \asymp \sqrt{\epsilon_j}$ worsens the rates obtained by
Theorem~\ref{thm:metric} to $n^{-\gamma/2}$. To deal with the first
problem, one could establish a condition under which the
misspecification metric dominates standard Hellinger; but this is
tricky and will be left for future work. The second problem is still
of interest in the next section, in which the misspecification metric
is replaced by the excess risk, which has the same meaning
irrespective of whether ${\cal F}$ is well-specified. 
As indicated below, for generalized linear models we can get rid of the square root
in (\ref{eq:entrobound}), but whether this can be done more generally
also remains an important open problem for future work. An
alternative, also to be considered for future work, is to refrain from
using the priors constructed for the well-specified case altogether
and instead directly design priors for the misspecified case, with
hopefully better bounds on information complexity.

\begin{myexample}{Exponential Families and Generalized Linear
    Models} \label{ex:glm-I} Consider a learning problem $\smtuple$ in
  the conditional density estimation setting of Example~\ref{ex:cde},
  so that $\ell$ is the conditional log-loss; $Z = (X,Y)$ with $X$
  taking values in $\cX \subset \reals^k$; and $\{p_f : f\in \cF\}$ 
  for some $\cF \subset \reals^k$ represents a $k$-dimensional generalized
  linear model (GLM), given in its standard parameterization (so that
  $\langle x, f \rangle$ is the linear predictor fed into the link function)
  \citep{McCullaghN89}.  \citet[Theorem 2]{DeHeideKGM19}
  show\footnote{In previous arXiv versions of this paper, we gave
    these results in full detail, adding 7 pages to its length.
    Following referee's comments and consultation with the associate
    editor, we moved them to the paper \citep{DeHeideKGM19}, where
    they are further illustrated by means of actual experiments with
    misspecified GLMs.}  that, under three further conditions on
  $\smtuple$, the central condition holds for some $\bar\eta > 0$,
  even under misspecification. In essence, the conditions require (1)
  that $Y$ has exponential tails, in the sense that
  $\sup_{x \in \cX} \E [\exp(\eta | Y |) \mid X=x] < \infty$ for some
  $\eta >0$ (a requirement that is automatically satisfied for, e.g.,
  logistic regression, for which $\cY$ is finite); (b) that $\cF$ is
  restricted to a compact (though possibly very high dimensional) set,
  and (c), that the misspecification is of a certain type: the noise
  may be misspecified in arbitrary ways, but the GLM should contain
  the distribution with the correct generalized regression
  function. That is, there should be an $f \in \cF$ indexing
  distribution $P_f$ with the correct conditional mean, so that
  $\E_{P_f}[Y \mid X] = \E_P[Y \mid X]$. This $f$ will then in fact be
  equal to the risk-optimal $f^*$. By taking $\cX$ to be a singleton,
  a GLM becomes a 1-dimensional natural exponential family, and the
  result thus also applies to such families. For this simplified case,
  \cite{DeHeideKGM19} show that the smallest $\bar\eta$ for which the
  $\bar\eta$-central condition holds is upper bounded by, and in some
  cases not much smaller than, the ratio of variances
  $\E_{P_{f^*}}[(Y - \E_{P_{f^*}}[Y])^2]/ \E_{P}[(Y - \E_{P}[Y])^2]$.
   
  \citet[Proposition 2]{DeHeideKGM19} shows that, if $\cF$ represents
  a GLM, then under the same three conditions, we have
  $\E_{Z \sim P}[L_{f}] = \E_{Z \sim P_{\fopt}}[L_f]$, so that there
  is no need to resort to Proposition~\ref{prop:entrobound}. This
  implies that for any prior satisfying the GGV condition in the
  well-specified case, the same prior can be used in the misspecified
  case and, using Theorem~\ref{thm:metric}, we can prove the same risk
  bounds, up to a constant factor, as in the well-specified case for
  generalized Bayes with any fixed $\eta < \bar\eta$. In particular,
  $k$-dimensional GLMs being sufficently general parametric models, we
  can use any continuous prior on $\cF$ that is bounded away from $0$
  and obtain that, for any fixed $\eta$,
  \mbox{$n\/ \cdot \/ \rsc_{n,\eta}(\nocomparator \dolest^{\text{B}}) \leq
  (k/2 \eta) \log n + O(1)$}, cf.~the remark after
  Proposition~\ref{prop:transavia}. Theorem~\ref{thm:metric} then
  gives a bound of $\tilde{O}(k/n)$, which is within a log factor of
  the minimax optimal parametric rate $O(k/n)$ for squared Hellinger
  distance in the well-specified case.
\end{myexample}
\begin{myexample}{Comparison to
    \cite{bhattacharya2019bayesian}} \label{ex:bhatta} After
  submitting the present paper, we became aware
  of \citep{bhattacharya2019bayesian}. The analysis and results of
  that paper (first submitted to arXiv in 2016, around the same time
  as the present paper) overlap with our Theorem~\ref{thm:metric}, and
  some of their examples have implications for our work as
  well. \cite{bhattacharya2019bayesian} focus exclusively on
  generalized Bayesian estimators $\Pi^B_|$. Their Theorem 3.6 is a
  variation of Zhang's Lemma~\ref{lem:rascbound-simple}, extended to
  handle non-i.i.d.~$P$. Their $\alpha$-R\'enyi divergence is just our
  $\eta$-annealed excess risk, with $\eta = 1- \alpha$. For $\cF$
  satisfying the $\bar\eta$-central condition, they provide Theorem
  3.1, which has some similarity to Theorem~\ref{thm:metric}: their result extends ours in that it allows non-i.i.d.~$P$; it rephrases ours so that
  the result is directly stated in terms of GGV-style conditions on
  $\Prior$ rather than on bounds on $\rsc_{n,\eta}(\Pi_|^B)$, similar
  to our (\ref{eq:vdv}); and it stays closer to
  Lemma~\ref{lem:rascbound-simple} in that it keeps the annealed
  excess risk on the left (a nonsymmetric divergence) where
  Theorem~\ref{thm:metric} has a (symmetric) metric. In their Lemma
  2.1. they re-prove the result of \cite{li1999estimation} and \Citet{erven2015fast} that $1$-strong
  central holds for convex probability models. Also, they provide
  (Section~5.1) an interesting novel example in which the strong
  $1$-central condition holds: Gaussian regression, with probability
  densities $p_f(y \mid x) \propto \exp(- (y - f(x))^2/2 \sigma^2)$
  with fixed variance $\sigma^2$, where the true noise is Gaussian and
  the set of regression functions $\cF$ is convex (but the
  corresponding density functions $\{p_f : f\in \cF \}$ are not, so
  Li's result does not apply). The model is misspecified in that $\cF$  does not contain the true regression
  function; in contrast, in Example~\ref{ex:glm-I} above we considered
  the reverse case in which the noise is misspecified yet the
  regression function is not. They show that in their setting, bounds on
  the annealed excess risk imply bounds on the $L_2(P)$-parameter
  estimation error that we consider in Example~\ref{ex:smallball}.
  They do not consider the non-annealed excess risk bounds and weaker
  forms of the central condition that we will turn to in the following
  sections.  \end{myexample}
\section{The Witness Condition}
\label{sec:witness}
We have seen via Theorem~\ref{thm:metric} that under the $\bar\eta$-central condition,
Lemma~\ref{lem:rascbound-simple} provides a bound on a weak
Hellinger-type metric. For problems different from density
estimation, i.e., loss functions different from log loss, we often mainly are
 interested in a bound on the excess risk. To get such bounds,
we need a second condition on top of the $\bar\eta$-central condition. To see why, consider again the density estimation example (Example~\ref{ex:cde}). If we assume a correct model, $p = p_{\fopt}$, then from \eqref{eq:preciselyone} the $\bar\eta$-central
condition holds automatically for all $\bar\eta \leq 1$, 
and so Theorem~\ref{thm:metric} gives a bound on the Hellinger distance.
Yet, while the Hellinger distance is bounded, in general we can have
$\kl(p \pipes p_f) = \infty$. If, for example, $\cF$ is the set of densities 
for the Bernoulli model, $P$ is Bernoulli$(1/2)$, and we use ERM for log loss (so that $\estim{f}$ is the maximum likelihood estimator for the Bernoulli model), 
we observe with positive probability only $0$'s. In this case, we will infer
$\estim{f}$ with $p_{\estim{f}}(Y= 0) =1$, and thus with positive
probability the excess risk between $\estim{f}$ and $\fopt$ is $\infty$
even though the expected Hellinger distance is of order $O(1/n)$.  We
thus need an extra condition. 

For log loss, the simplest such condition is that the likelihood ratio of $p_{\fopt}$ to $p_f$ is uniformly bounded
for all $f \in \cF$. For that case, \cite{birge1998minimum} 
proved a tight bound on the ratio between the standard KL divergence and the standard $(\eta =
1/2)$ Hellinger distance. Lemma~\ref{lem:renyi-to-kl-simple} below represents
a generalization of their result to arbitrary $\eta$,
misspecified $\cF$, and general loss functions under the {\em witness
  condition\/} which we introduce below, and which is a significant weakening
of the bounded likelihood ratio condition. It is the cornerstone for
proving our subsequent results: Theorems~\ref{thm:firstriskbound}, \ref{thm:main-bounded-a}, \ref{thm:main-bounded-b}, and \ref{thm:main-unbounded}. 
Whereas the strong central condition imposes exponential decay of the lower tail of the excess loss $\loss_f - \loss_{\fopt}$, the witness condition imposes a
much weaker type of control on the upper tail of $\loss_{f}- \loss_{\fopt}$. 

Below, we show that the witness condition generalizes not only conditions of \cite{birge1998minimum} but also of \cite{sason2016f} and \cite{wong1995probability} (Example~\ref{ex:wong}). We also show that it holds in a variety of settings, e.g., with exponential families with suitably restricted parameter spaces in the well-specified setting and when the log likelihood has exponentially small tails (Example~\ref{ex:glm-II}), but also with bounded regression under heavy-tailed distributions (Example~\ref{ex:heavy-tailed-regression}). 
Moreover, although the conditions are not equivalent, there is an intriguing similarity to the recent \emph{small-ball assumption} of \cite{mendelson2014learning} (Example~\ref{ex:smallball}).

\subsection{Definition and Main Result}
\label{sec:subwitness}
\begin{definition}[Empirical Witness of Badness] \label{def:witness}
We say that $\smtuple$ 
satisfies the \glsaddcond{u-c-witness} $(u,c)$-\emph{empirical witness of badness condition} (or \emph{witness condition}) 
for constants $u > 0$ and $c \in (0,1]$ if for all $f \in \cF$
\begin{align} \label{eqn:emp-witness-badness-simple}
 \E \left[ (\xslosslong{f}) \cdot \ind{\xslosslong{f} \leq u} \right] \geq c \xsrisklong{f} .
\end{align}
More generally, for a function $\tau: \reals^+ \to [1,\infty)$
and constant $c \in (0,1)$ 
we say that $\smtuple$
satisfies the
\glsaddcond{tau-c-witness} \emph{$(\tau,c)$-witness  condition} if 
for all $f \in \cF$, $\xsrisklong{f} < \infty$ and
\begin{align} \label{eqn:emp-witness-badness-general}
\E \left[ (\xslosslong{f}) \cdot \ind{\xslosslong{f} \leq \tau(\xsrisklong{f})} \right] \geq c \xsrisklong{f} .
\end{align}
\end{definition}
The $(u,c)$-witness condition
\eqref{eqn:emp-witness-badness-simple} is just the $(\tau,c)$-witness
condition for the constant function $\tau$ identically equal to $u$.  
In our results we frequently use that, by adding $\E \left[ (\xslosslong{f}) \cdot \ind{\xslosslong{f} > u} \right]$
to both sides of \eqref{eqn:emp-witness-badness-simple} and rearranging, the $(u,c)$-witness condition holds if and only if for $c' = 1-c$ (and hence $c' \in (0,1)$), 
\begin{align} \label{eq:reversewitness}
 \E \left[ (\xslosslong{f}) \cdot \ind{\xslosslong{f} > u} \right] \leq c' 
\xsrisklong{f},
\end{align}
and similarly for the $\tau$-version.  

The intuitive reason for imposing this condition is to rule out
situations in which learnability simply cannot hold.  
{For instance, consider a setting with $\cF = \{f^*, f_1, f_2, \ldots\}$ where $\loss_{\fopt} = 1$ with probability 1 and, for each $j \geq 1$, $\loss_{f_j}$ is equal to 0 with probability $1 - \frac{1}{j}$ and equal to $2j$ with probability $\frac{1}{j}$. Then for all $j$, $\E [ \loss_{f_j} - \loss_{\fopt} ] = 1$, but as $j \rightarrow \infty$, empirically we will never \emph{witness the badness of $f_j$} as it almost surely achieves lower loss than $\fopt$.} 
On the other hand, if
the excess loss is upper bounded by some constant $b$, we may always
take $u = b$ and $c = 1$ so that a witness condition is trivially
satisfied. Below we provide
several nontrivial examples besides bounded excess losses and finite $\cF$ in
which the witness condition holds.

The following result shows how the witness condition, combined with
the strong central condition, leads to fast-rate excess risk bounds:
\begin{lemma} \label{lem:renyi-to-kl-simple}
Let $\bar{\eta} > 0$. Assume that the $\bar{\eta}$-strong central condition (\ref{eqn:central}) holds and let, for arbitrary $0 < \eta < \bar\eta$,  
$c_u := \frac{1}{c} \frac{\eta u + 1}{1 - \frac{\eta}{\bar{\eta}}}$. 
Suppose further that the $(u,c)$-witness condition holds for $u > 0$ and
$c \in (0, 1]$. Then for all $f \in \cF$, all $\eta \in (0, \bar{\eta})$:
\begin{align}\label{eq:rklsimple}
\xsrisk{f} 
\,\,\leq\,\, c_u \cdot \Exphel{\eta} \left[ \xsloss{f} \right]
\,\,\leq\,\, c_u \cdot \Expann{\eta} \left[ \xsloss{f} \right] .
\end{align}
More generally, suppose that the $\bar\eta$-central condition and the $(\tau,c)$-witness condition hold for $c\in (0,1]$ and a {\em non-increasing\/} function $\tau$. Then for {all} $\lambda > 0$, all $f \in \cF$,
\begin{equation}\label{eq:rkltau}
\xsrisk{f} 
\,\,\leq\,\, \lambda \opmax 
\left( c_{\tau(\lambda)}  \cdot \Exphel{\eta} \left[ \xsloss{f} \right] \right) 
\,\,\leq\,\, \lambda \opmax \left( c_{\tau(\lambda)} \cdot \Expann{\eta} \left[ \xsloss{f} \right] \right).
\end{equation}
\end{lemma}
Note that for large $u$, $c_u$ is approximately linear in $u/c$.

The following theorem is now an almost immediate corollary of
Lemma~\ref{lem:rascbound-simple} and
Lemma~\ref{lem:renyi-to-kl-simple}:
\begin{theorem}\label{thm:firstriskbound}
  Consider a learning problem $\smtuple$ and a learning algorithm
  $\dolest$.  Suppose that the $\bar\eta$-strong central condition
  holds.  If the $(u,c)$-witness condition holds, then for any
  $\eta \in (0,\bar\eta)$,
  \begin{align*}
    \E_{\rv{f} \sim \dol_n}\left[\xsrisk{f}\right] 
    \ \stochleq_{\frac{\eta \cdot n}{c_u}} \  
    c_u \cdot \rsc_{n,\eta} \left( \nocomparator  \dolest \right),
  \end{align*}
  with $c_u$ as in Lemma~\ref{lem:renyi-to-kl-simple}. If instead the $(\tau,c)$-witness condition holds for
some non-increasing function $\tau$ as above, then for any $\lambda > 0$
\begin{equation}\label{eq:friday}
\E_{\rv{f} \sim \dol_n} \left[\xsrisk{f} \right] 
\ \stochleq_{\frac{\eta \cdot n}{c_{\tau(\lambda)}}} \ 
\lambda + c_{\tau(\lambda)} \cdot  \rsc_{n,\eta} \left( \nocomparator  \dolest \right).
\end{equation}
\end{theorem}
\begin{proof} 
The first and second inequalities are from chaining Lemma~\ref{lem:rascbound-simple} with Lemma~\ref{lem:renyi-to-kl-simple} (\eqref{eq:rklsimple} and \eqref{eq:rkltau} respectively). The first inequality is immediate using that for general random variables $U,V$, we have $U \stochleq_{a} V \Leftrightarrow  c U \stochleq_{a/c} c V$. 
For the second inequality, we first upper bound the max on the RHS of \eqref{eq:rkltau} by the sum of the terms.
\end{proof}
This theorem is applicable if the $(\tau,c)$-witness condition holds for a
non-increasing $\tau$. If the risk $\sup_{f \in \cF} \xsrisk{f}$ is
unbounded, we can only expect the witness condition to hold for $\tau$
such that for large $x$, $\tau(x)$ is increasing; such $\tau$ are considered
in Section~\ref{sec:main-unbounded}.  Non-increasing $\tau$ are often appropriate
for scenarios with bounded risk (even though the loss may be unbounded and
even heavy-tailed); we encounter one instance thereof in the exponential family example
below. There, $\lim_{x \downarrow 0} \tau(x) = \infty$, but the
increase as $x \downarrow 0$ is so slow that the optimal $\lambda$ at sample size $n$ is
of order $O(1/n)$ and $c_{\tau(\delta)} = O(\log n)$, leading only to an additional log factor in the bound compared to the case where the $(u,c)$-witness condition holds for constant $u$.

\paragraph{Some Existing Bounds Generalized by Lemma~\ref{lem:renyi-to-kl-simple}}
Lemma~\ref{lem:renyi-to-kl-simple} generalizes
a result of \citet[Lemma 5]{birge1998minimum} (also stated and proved
in \citet[Lemma 4]{yang1998asymptotic}) that bounds the ratio between
the standard KL divergence $\KL(P \pipes Q)$ and the (standard)
$1/2$-squared Hellinger distance $\Hell_{1/2}(P\| Q)$  for distributions $P$
and $Q$. To see this, take density estimation under log loss in the
well-specified setting with $\eta < \bar{\eta} = 1$, so that $\fopt = p$
and $f = q$; then the left-hand side becomes $\KL(P \pipes Q)$ and
the right-hand side 
$\frac{1}{\eta} \E [ 1 - e^{-\eta \xsloss{f}} ] = \frac{1}{\eta}
\left( 1 -  \E [ (q / p)^\eta ] \right) = \Hell_\eta(P \pipes Q)$ (this notation was introduced below Proposition~\ref{prop:hellingerparty}). Under a bounded density
ratio $p / q \leq V$, we can take $u = \log V$ and $c = 1$ (the
$(u,c)$-witness condition is then trivially satisfied), so that
$c_u = \frac{\eta \log V + 1}{1- \eta}$, which for
$\eta = \nicefrac{1}{2}$ coincides with the Birg\'e-Massart bound. 
The case of general $\eta \in (0,1)$ first was handled by \cite{haussler1997mutual} (see Lemma 4 therein), but their bound stops short of providing an explicit upper bound for the ratio.

\cite{sason2016f} independently obtained
an upper bound (see Theorem 9 therein) on the ratio of the standard
KL divergence $\KL(P \pipes Q)$ to the $\eta$-generalized Hellinger
divergence in the case of bounded density ratio $\operatorname{ess}
\sup \frac{dP}{dQ}$, for general $\eta$. Theorem
\ref{lem:renyi-to-kl-simple} generalizes Theorem 9 of
\cite{sason2016f} by allowing for misspecification in the case of
density estimation with log loss, allowing for general losses, and,
critically for our applications, allowing for unbounded density ratios
under a witness condition. We note that in the case of bounded density
ratio $\frac{dP}{dQ}$ and the regime $\eta \in (0,1)$ (corresponding
to $\alpha = 1 - \eta \in (0,1)$ in Theorem 9 of \cite{sason2016f}),
their bound and the unsimplified form of our bound (see $C_{0
  \leftarrow \eta}(V)$ in Lemma~\ref{lemma:ratio} in
Appendix~\ref{app:proofs-witness}) are identical, as they should be
since both bounds are tight. The additional, slightly looser
simplified bound that we provide greatly helps to simplify the treatment
for unbounded excess losses under the witness condition. We stress
though that \cite{sason2016f} treat general $F$-divergences under
well-specification, including a wide array of divergences beyond
$\eta$-generalized Hellinger for $\eta \in (0,1)$, so in that respect,
their bounds are far more general.
In the next section we establish that
Lemma~\ref{lem:renyi-to-kl-simple} also generalizes a bound by
\cite{wong1995probability}.

\subsection{Example Situations in which the Witness Condition Holds}
We now present some examples of common learning problems in which the
$(\tau,c)$-witness condition holds for a suitable $\tau$.  We first
consider a case where the distribution of the excess loss has
exponentially decaying tails in both directions.  The $(u,c)$-witness
condition \eqref{eqn:emp-witness-badness-simple} does not always hold
for such excess losses, but we now show that the $\tau$-witness
condition is {\em always\/} guaranteed to hold in such cases for a
non-increasing function $\tau$, which leads to a bound on excess risk
that is only a log factor worse than the direct bound on the annealed
risk of Lemma~\ref{lem:rascbound-simple}. \glsaddcond{unif-exp-tail}
\begin{definition}\label{def:uniformlyexponential}
Suppose that for given  $\smtuple$ and a collection of random variables $\{ U_f : f \in \cF \}$, there is a $0 < \kappa < \infty$ such that $\sup_{f \in \cF} \E \left[ e^{\kappa U_f} \right] < \infty$. Then we say that {\em $U_f$  has a uniformly exponential upper tail}. 
\end{definition}
The name reflects that $U_f$ has uniformly exponential upper tails if
and only if there are constants $c_1, c_2 > 0$ such that for all
$u > 0$, $f \in \cF$, $P(U_f \geq u) \leq c_1 e^{- c_2 u}$, as is
easily shown (we omit the details).
\begin{lemma} \label{lem:kl-hell-exp-tails} Define
  $M_{\kappa} := \sup_{f \in \cF} \E \left[ e^{\kappa \xsloss{f}} \right]$ and assume
  that $L_f$ has a uniformly exponential upper tail, so that $M_{\kappa} < \infty$.  Then, for the map
  $\tau: x \mapsto 1 \opmax \kappa^{-1} {\log \frac{2 M_\kappa }{\kappa x}} = O(1 \opmax \log (1/x))$, the $(\tau,c)$-witness condition holds with $c = \nicefrac{1}{2}$. 
\end{lemma}
Now let $\bar{\eta} > 0$. Assume both the $\bar\eta$-strong central condition, i.e.,
$\E \left[ e^{-\bar{\eta} \xsloss{f}} \right] \leq 1$, and that $L_f$ has a uniformly exponential upper
tail.  As an immediate consequence of the lemma above,
Theorem~\ref{thm:firstriskbound} now gives that for any learning
algorithm $\dolest$ for any $\eta \in (0,\bar\eta)$, (using $\lambda=
1/n$), there is $C_{\eta} < \infty$ such that
\begin{align} \label{eqn:kl-hell-exp-tails-2} 
\E_{\rv{f} \sim \dol_n} \left[ \xsrisk{\rv{f}} \right] \stochleq_{\frac{\eta \cdot n}{C_\eta \log n}} 
    \frac{1}{n}
+ C_{\eta} \cdot (\log n) \rsc_{n,\eta} \left( \nocomparator  \dolest \right) ,
\end{align}
so we obtain an excess risk bound that is only a log factor worse than the bound that can be obtained for the generalized Hellinger metric in Theorem~\ref{thm:firstriskbound}. 

\begin{myexample}{Generalized Linear Models and
    Witness} \label{ex:glm-II} Consider again Example~\ref{ex:glm-I},
  about GLMs. \citet[Appendix B]{DeHeideKGM19} show that, under the
  three assumptions that we informally listed in
  Example~\ref{ex:glm-I}, the conditions of
  Lemma~\ref{lem:kl-hell-exp-tails} are satisfied. We can thus use
  (\ref{eqn:kl-hell-exp-tails-2}) to give us that, up to log-factors,
  for misspecified GLMs satisfying the three conditions mentioned in
  Example~\ref{ex:glm-I} and generalized Bayesian estimators based on
  priors that are continuous and bounded away from $0$ on $\cF$, we
  can prove a rate of order $\tilde{O}(d/n)$, which, up to log
  factors, is equal to the minimax parametric rate.
  \end{myexample}

As a second consequence of Lemma~\ref{lem:kl-hell-exp-tails}, this time combined with \eqref{eq:rkltau} from Lemma~\ref{lem:renyi-to-kl-simple} with 
$\lambda= \Exphel{\eta}\left[\xsloss{f}\right]$, 
we find that under the conditions of Lemma~\ref{lem:kl-hell-exp-tails}, 
there is $C_{\eta} < \infty$ such that
\begin{align} \label{eqn:kl-hell-exp-tails-4} 
\xsrisk{\rv{f}} 
\leq 
    \max \left\{ 
        \Exphel{\eta}\left[\xsloss{f}\right],
         C_{\eta} \cdot \Exphel{\eta}\left[\xsloss{f}\right] 
         \cdot \log \frac{1}{\Exphel{\eta}\left[\xsloss{f}\right]} 
    \right\} .
\end{align}
The above result generalizes a bound due to \cite{wong1995probability}, as we now show.

\begin{myexampleplain} \label{ex:wong}
  The bound \eqref{eqn:kl-hell-exp-tails-4} generalizes a bound of \cite{wong1995probability}. 
  Their result, the first part of their Theorem 5, allows one to bound KL divergence in terms of Hellinger distance, 
  i.e., it holds in the special case of well-specified density estimation under log loss with the choice $\bar\eta = 1$, $\eta = 1/2$. 
Formally, consider probability model $\{P_f \mid f \in \cF \}$ where each $P_f$ has density $p_f$, and assume the model is well-specified in that $Z \sim P = P_{f^*}$ with $f^* \in \cF$. 
\cite{wong1995probability} consider the condition that for some $0 < \kappa < 1$,  it holds that  $M'_\kappa   := \sup_{f \in \cF} \int_{(p_f/p_{f^*}) \geq e^{1/\kappa}} p_{f^*}
  (p_{f^*}/p_f)^\kappa < \infty$. 
They show that, under this condition, the following holds for all $f \in \cF$ in the regime  $\Hell_{1/2}(P_{\fopt}
  \pipes P_{f}) = \Exphel{\eta}\left[\xsloss{f}\right] \leq \frac{1}{2}
  \left( 1 - e^{-1} \right)^2$:
\begin{align}
 \xsrisk{f} 
\leq \left( 6                 + \frac{2 \log 2}{(1 - e^{-1})^2} 
                + \frac{4}{\kappa} 
                    \max \left\{ 2, \log \frac{M'_\kappa}{
\Exphel{\eta}\left[\xsloss{f}\right]}
        \right\} \right) \Exphel{\eta}\left[\xsloss{f}\right],  
\label{eqn:theirs-small-hell}
\end{align}
where $\loss_f = - \log p_f$ is log loss.
Now, note that for this loss function and in the case $\bar{\eta}=1$ (where their result applies too),  $M_{\kappa}$ in Lemma~\ref{lem:kl-hell-exp-tails} and $M'_{\kappa}$ in \eqref{eqn:theirs-small-hell} satisfy 
$M'_{\kappa} \leq M_{\kappa} \leq M'_{\kappa} + e$. 
Comparing \eqref{eqn:theirs-small-hell} to \eqref{eqn:kl-hell-exp-tails-4}, we see that up to  values of the constants, our result generalizes Wong and Shen's.
\end{myexampleplain}

We just showed that a $\tau$-witness condition always holds under
exponential tails of the loss. The following example shows that even
if the loss random variables $\loss_f$ have fat (polynomial) tails, the witness condition
often holds, even for constant $\tau$. Before providing the example,
we first recall the Bernstein condition
\citep{audibert2004pac,bartlett2006empirical} and a useful proposition
that will be leveraged in the example.

\glsaddcond{bernstein}
\begin{definition}[Bernstein Condition] \label{def:bernstein}
For some $B > 0$ and $\beta \in (0, 1]$, we say $\smtuple$ satisfies the $(\beta,B)$-Bernstein condition if, for all $f \in \cF$, 
$\E [ \xsloss{f}^2 ] \leq B \left( \xsrisk{f} \right)^\beta$.
\end{definition}
The best case of the Bernstein condition is when the exponent $\beta$
is equal to 1. In past works, the Bernstein condition has mostly been
used to characterize fast rates in the bounded excess loss regime, where
the $(u,c)$-witness condition holds automatically. In that regime, the Bernstein
condition for $\beta = 1$ and the central condition become equivalent
(i.e.~for each $(\beta,C)$ pair there is some $\bar{\eta}$ and vice
versa, where the relationship depends only on the upper bound on the
loss; see Theorem 5.4 of \Citet{erven2015fast}). The following
proposition shows that with unbounded excess losses, the Bernstein condition
can also be related to the witness condition:
\begin{proposition}[Bernstein implies
  Witness] \label{prop:bernstein-witness} If $\smtuple$ satisfies the
  $(\beta,B)$-Bernstein condition, then, for any $u > B$, $\smtuple$
  satisfies the $(\tau,c)$-witness condition with $\tau(x) =
  u \cdot (1/x)^{1- \beta}$ and $c = 1 - \frac{B}{u}$. In particular, if
  $\beta = 1$ then $\smtuple$ satisfies the $(u,c)$-witness condition
  with \mbox{constant $u$}.
\end{proposition}
The special case of this result for $\beta =1$ will be put to use in Example~\ref{ex:heavy-tailed-regression-II} in Section~\ref{sec:grip}.
\begin{myexample}{Heavy-tailed regression with convex luckiness and bounded predictions} \label{ex:heavy-tailed-regression}
Consider a
  regression problem with squared loss, so that $\cZ = \cX \times
  \cY$. Further assume that the risk minimizer $\fopt$ over $\cF$
  continues to be a minimizer when taking the minimum risk over the convex hull of $\cF$. 
We call this assumption \glsaddcond{conv-lucky} \emph{convex luckiness} for squared loss. 
It is implied, for example, when $\cF$ is convex or when the model is well-specified in
  the sense that $Y = \fopt(X) + \xi$ for $\xi$ a zero-mean random
  variable that is independent of $X$. Thus, when $\cF$ is convex, we
  can enforce it; if we are not willing to work with a convex $\cF$
  (for example, because this would blow up the $\textsc{comp}_n$ in
  (\ref{eq:optimalrates})), then we are ``lucky'' if it holds --- since it
  allows, in general, for better rates (see
  Section~\ref{sec:related-work} for additional discussion).

Now assume further that $\E [ Y^2 \mid X ] \leq C$ a.s.~and the function class
$\cF$ consists of functions $f$ for which the predictions $f(X)$ are
bounded as $|f(X)| \leq r$ almost surely. Proposition~\ref{prop:awkwardwithout} shows that in this setup,  the Bernstein
condition holds with exponent 1 and multiplicative constant
$8 (\sqrt{C} + r)^2$.  Proposition \ref{prop:bernstein-witness} then
implies that the $(u,c)$-witness condition holds with
$u = 16 (\sqrt{C} + r)^2$ and $c = \frac{1}{2}$.
\end{myexample}

\begin{proposition}\label{prop:awkwardwithout}
  Under the assumptions of the example above, the
  $(1, 8 (\sqrt{C} + r)^2)$-Bernstein condition holds.
\end{proposition}
We note that Theorem~\ref{thm:firstriskbound} cannot be used with
squared loss when $Y$ is heavy-tailed as then the strong central
condition cannot hold. Thus, while Example~\ref{ex:heavy-tailed-regression} might imply in this case that a $(u,c)$-witness condition holds, we do not yet have the
machinery to put this fact to use. 
However, in Example~\ref{ex:heavy-tailed-regression-II}, we show that weaker
easiness conditions can still hold and fast rates can still be obtained.

\begin{myexample}{Example~\ref{ex:heavy-tailed-regression} and Lemma~\ref{lem:renyi-to-kl-simple} in light of \cite{Birge04}}
  Proposition 1 of \cite{Birge04} shows that, in the case of
  well-specified bounded regression with Gaussian noise $\xi$, the
  excess risk is bounded by the $1/2$-annealed excess risk times a
  constant proportional to $r^2$, where $r$ is the bound on $|f(X)|$
  as in Example~\ref{ex:heavy-tailed-regression}. This result thus
  gives an analogue of Lemma~\ref{lem:renyi-to-kl-simple} for
  bounded regression with Gaussian noise and also allows us to apply
  one of our main results, Theorem~\ref{thm:main-bounded-b} below
  (excess risk bounds with heavy-tailed losses), for this model. Our
  earlier Example~\ref{ex:heavy-tailed-regression} extends Birg\'e's
  result, since it shows that the excess risk can be bounded by a
  constant times the annealed excess risk if the target
  $Y$ has an almost surely uniformly bounded conditional second moment, 
which, in the well-specified setting in particular, specializes to $\xi \mid X$ almost surely having (uniformly) bounded second moment
  (and thus potentially having quite heavy tails) rather than Gaussian tails. On
  the other hand, \cite[Section 2.2]{Birge04} also gives a negative
  result for sets $\cF$ that are not bounded
  (i.e.~$\sup_{x \in \cX, f \in \cF} |f(x)| = \infty$): even in the
  ``nice'' case of Gaussian regression, there exist such sets for which
  the ratio between excess risk and annealed excess risk can be
  arbitrarily large, i.e., there exists no finite constant $c_u$ for
  which (\ref{eq:rklsimple}) holds for all $f \in \cF$. From this we
  infer, by using Lemma~\ref{lem:renyi-to-kl-simple} in the
  contrapositive direction, that for such $\cF$ the witness condition
  also does not hold.
\end{myexample}
 
\begin{myexample}{witness vs. the small-ball assumption}\label{ex:smallball}
Intriguingly, on an intuitive level the witness condition bears some
similarity to the small-ball assumption of \cite{mendelson2014learning}. 
\glsaddcond{small-ball} This assumption states that there
exist constants $\kappa > 0$ and $\epsilon \in (0, 1)$ such that, for
all $f, h \in \cF$, we have 
\begin{align} \label{eqn:small-ball}
\Pr \left( |f - h| \geq \kappa \|f - h\|_{L_2(P)} \right) \geq \varepsilon. 
\end{align}
Under this assumption, \cite{mendelson2014learning} established
bounds on the $L_2(P)$-parameter estimation error
$\|\estim{f} - \fopt\|_{L_2(P)}$ in function learning.  For the
special case that $h = \fopt$, one can read the small-ball assumption
as saying that ``no $f$ behaving very similarly to $\fopt$ with high
probability is very different from $\fopt$ only with very small
probability so that it is still quite different on average.''  The
witness condition reads as ``there should be no $f$ that is no worse
than $\fopt$ with high probability and yet with very small probability
is much worse than $\fopt$, so that on average it is still
substantially worse''.  Despite this similarity, the details are quite
different.  In order to compare the approaches, we may consider
regression with squared loss in the well-specified setting as in the
example above. Then the $L_2(P)$-estimation error becomes equivalent
to the excess risk, so both Mendelson's and our results below bound
the same quantity. But in that setting one can
  easily construct an example where the witness and strong central
  conditions hold (so Theorem \ref{thm:firstriskbound} applies) yet
  the small-ball assumption does not (Example \ref{ex:no-small-ball}
  in Appendix~\ref{app:examples}); but it is also straightforward to
  construct examples of the opposite by noting that small-ball
  assumption does not refer to $Y$ whereas the witness condition
  does. In Section~\ref{sec:main-unbounded} we will see that,
nevertheless, the small-ball assumption can be related to the
$\tau$-witness condition for a particular $\tau$ that is needed in the
unbounded risk scenario (Theorem~\ref{thm:main-unbounded}).
\end{myexample}

\section{Bounds under Weaker Easiness Conditions}
\label{sec:grip}
In many learning problems, there is no $\eta > 0$ such that the strong
$\eta$-central condition is satisfied. 
Yet, it turns out that in many cases of interest 
there still exist weaker conditions under which fast convergence rates are possible. 
We consider two types of conditions. Both are best understood by 
generalizing the notion of excess risk: whereas hitherto, this was
invariably defined as the risk (expected loss of some learner
$\dolest$) relative to the comparator $f^*$ that was optimal within
$\cF$, we will now also allow more general comparators that lie
outside $\cF$.  In particular we will consider as comparator 
a {\em pseudo-predictor\/} $g$ with risk
$\E[\ell_g] = \E[\ell_{f^*}] - \epsilon$ for some small $\epsilon > 0$. 
Being better than $f^*$, $g$ does not correspond to an action that can be actually played, 
but one can often find a $g$ such that, with $f^*$ replaced by $g$, 
the $\eta$-central condition does hold for some $\eta > 0$ while, simultaneously, 
$\epsilon$ is so small that an excess risk bound relative to $g$ implies also a good excess risk bound relative to the original comparator $f^*$. 
We will soon introduce a function $v$ that modulates how large one can take $\eta$ for a desired $\epsilon$ (the larger $\eta$, the better the bounds that ensue).

In order to work with comparators that are pseudo-predictors, we now introduce \glsadd{allactionset} $\allactionset$, an enlarged action space that is a superset of $\cF$ and that also contains the pseudo-predictors we use in the remainder of this work. These pseudo-predictors always will be deterministic and typically will be constant-shifted versions of $\loss_f$ (for some $f \in \cF$) or versions of a GRIP (introduced in Definition~\ref{def:grip}). Although a given pseudo-predictor $f \in \allactionset$ can fail to be well-defined as a playable action, the loss $\loss_f$ of any pseudo-action we employ will always be well-defined. We thus extend our loss notation $\loss_f(z)$ to all $f \in \allactionset$.

We first consider the \emph{$v$-central condition}, 
a strict weakening of the strong central condition which applies if
the excess loss is bounded or has exponential tails; here the
comparator can be taken to be a trivial modification of $f^*$.  We
next consider the \emph{$v$-PPC condition}, a strict weakening of
the $v$-central condition, which applies if the losses have polynomial tails.  It is
based on using a new type of comparator, the \emph{generalized
  reversed information projection} (\emph{GRIP}), which generalizes a
concept from Barron and \cite{li1999estimation}. 
In Section~\ref{sec:vcentral} we present the $v$-central condition and a corresponding excess risk bound for bounded excess risks. 
Section~\ref{sec:vppc} presents the $v$-PPC condition, the GRIP, and the corresponding excess risk bound for bounded excess risks. 
Finally, Section~\ref{sec:main-unbounded} shows risk bounds under the $v$-PPC and $v$-central conditions for unbounded excess risks.
\subsection{The $v$-Central Condition}
\label{sec:vcentral}
\begin{definition}[$\boldsymbol{v}$-Central Condition \Citep{erven2015fast}] \label{def:v-central} 
Let $\eta > 0$ and $\epsilon \geq 0$. We say that $\smtuple$ satisfies the
\glsaddcond{central-up-to-eps} $\eta$-central condition up to $\epsilon$ if there exists some $\tilde{f} \in \cF$ such that
\begin{align} \label{eqn:centraleps}
\loss_{\tilde{f}} - \loss_f \stochleq_{\eta} \epsilon \qquad \text{for all } f \in \cF .
\end{align}
Let $v: [0, \infty) \rightarrow [0, \infty)$ be a bounded,
non-decreasing function satisfying $v(\epsilon) > 0$ for all
$\epsilon > 0$. We say that $\smtuple$ satisfies the \glsaddcond{v-central} 
$v$-central condition if, for all $\epsilon \geq 0$, there exists a function
$\tilde{f} \in \cF$ such that \eqref{eqn:centraleps} is satisfied with
$\eta = v(\epsilon)$.
\end{definition}
The special case with constant $v(\epsilon) \equiv\bar\eta$ reduces to
the earlier strong $\bar\eta$-central condition (and then $\tilde{f}$
must be optimal so we can take $\tilde{f} = f^*$); for nonconstant $v$, the condition is weaker in that
it allows a little slack $\epsilon$, and to make $\epsilon$ small, we
need to take $\eta$ small.  For each $\epsilon \geq 0$, we now
define \glsadd{f-star-epsilon} $f^*_{\epsilon}$ in terms of its loss
by $\forall z\in \cZ:\ell_{f^*_{\epsilon}}(z) := \ell_{f^*}(z) -
\epsilon$.
This $f^*_{\epsilon}$ plays the role of alternative comparator referred to above.
We can now apply Lemma~\ref{lem:rascbound-simple} with $f^*_{\epsilon}$
instead of $f^*$  to get a bound on the annealed excess risk: 
\begin{equation}\label{eq:simplifyproof}
\E_{\rv{f} \sim \dol_n} \left[ 
\Expann{\eta} \left[ \loss_f - \loss_{f^*_{\epsilon}} \right] \right]
\ \ \stochleq_{\eta \cdot n}
\ \  
\rsc_{n,\eta}(\dolest) + \epsilon.
\end{equation}
Analogous to the story in  Section~\ref{sec:subwitness}, we want to turn this bound into an
actual excess risk bound. This is done by the following lemma, 
which is a straightforward consequence from the first part of 
Lemma~\ref{lem:renyi-to-kl-simple} and only differs from it in that
it has $\loss_{\fopt}$ on the right-hand side replaced by
$\loss_{f^*_{\epsilon}}$ and a slightly larger constant factor.
\begin{lemma} \label{lem:renyi-to-kl-grip-a} Let $\smtuple$ be a
  learning problem that satisfies the $v$-central condition for some $v$. Let $f \in \cF$. 
Suppose that (\ref{eqn:emp-witness-badness-simple}) holds for some $u> 0$ and $c \in (0,1]$, i.e., $(P,\ell,\{f, f^*\})$ satisfies the $(u,c)$-witness condition. Fix $\epsilon \geq 0$ and let $\bar{\eta} = v(\epsilon)$. 
As in Lemma~\ref{lem:renyi-to-kl-simple}, 
let $c_u = \frac{1}{c} \frac{\eta u + 1}{1 - \frac{\eta}{\bar{\eta}}}$. 
Then for all $\eta \in (0, \bar{\eta})$,
\begin{align}\label{eq:rklsimplegrip-a}
\xsrisk{f} 
\,\,\leq\,\, c_{u+ \epsilon} \cdot \Expann{\eta} \left[ \loss_f - \loss_{f^*_{\epsilon}} \right].
\end{align}
In particular, if $(P,\ell, \cF)$ satisfies the $(u,c)$-witness condition then (\ref{eq:rklsimplegrip-a}) holds for all $f \in \cF$. 
\end{lemma}
The key to the proof is that, if  $\smtuple$ satisfies the $v$-central condition, then we have that 
\begin{align} \label{eq:unusual}
(P,\ell,\cF \cup \{f^*_{\epsilon} \} )  \text{\ satisfies\ the $\eta$-central condition with $\eta = v(\epsilon)$}.
\end{align} 
We now show how Lemma~\ref{lem:renyi-to-kl-grip-a}  straightforwardly implies a strict strengthening of Theorem~\ref{thm:firstriskbound}, one which holds under the $v$-central condition rather than just the $\bar\eta$-central condition: 
since \eqref{eq:rklsimplegrip-a} holds for all $f \in \cF$, it also holds in expectation over $f$, 
under any arbitrary distribution $\dol$ over $f$. 
We can thus take expectations over $\dol_n$ on both sides of  \eqref{eq:rklsimplegrip-a} and chain the resulting inequality with ESI  \eqref{eq:simplifyproof}. 
Using  that for general random variables $U,V$ and $ c > 0$, 
$U \stochleq_a V \Leftrightarrow c U \stochleq_{u/c} c V$, this gives: 
\begin{theorem}[$\boldsymbol{v}$-Central Excess Risk Bound - Bounded Excess Risk Case] \label{thm:main-bounded-a}
Let $\dolest$ be an arbitrary learning
algorithm based on $\cF$. Assume that $\smtuple$ satisfies the $(u,c)$-witness condition (\ref{eqn:emp-witness-badness-simple}) and let $c_u$ be defined as in Lemma~\ref{lem:renyi-to-kl-grip-a}. Then 
under the $v$-central condition, for any $\epsilon \geq 0$, any
$0 < \eta < v(\epsilon)$:
\begin{align} \label{eqn:v-central-bound}
\E_{\rv{f} \sim \dol_n} \left[ \xsrisk{\rv{f}} \right]
\ \ \stochleq_{\frac{\eta \cdot n}{c_{u+\epsilon}}}
\ \ c_{u+\epsilon} \cdot \left( \rsc_{n,\eta}(\nocomparator \dolest) + \epsilon \right) .
\end{align}
\end{theorem}
Analogously to the second part of Lemma~\ref{lem:renyi-to-kl-simple} and Theorem~\ref{thm:firstriskbound},
one can  give versions of this result for the $\tau$-witness condition as well, but for simplicity we will not do so.
This theorem allows unbounded losses but is only useful when the excess risk is bounded, 
i.e., $\sup_{f \in \cF} \ \E[\xsloss{f}] < \infty$, because for unbounded risk, the required
$(u,c)$-witness condition is excessively strong; see
Section~\ref{sec:main-unbounded}.

The factor $c_{u+\epsilon}$ explodes if $\eta \uparrow
v(\epsilon)$. If the $v$-central condition holds for some $v$, it
clearly also holds for any smaller $v$, in particular for
{$\underline{v}(\epsilon) := v(\epsilon) \opmin 1$}.  Applying the
theorem with $\underline{v}$ (which
will not affect the rates obtained), we may thus take
$\eta = \underline{v}(\epsilon) / 2$, so that $c_{u+{\epsilon}}$ is
bounded by $\frac{1}{c} (u + \epsilon+ 2)$.  The ESI in
(\ref{eqn:v-central-bound}) then implies that with probability at
least $1 - e^{-K}$ the left-hand side exceeds the right-hand side by
at most $\frac{ (u + \epsilon +2) K }{c \eta n}$.  For the case of
bounded excess loss, we can further take $u$ to be
$\sup_{f \in \cF} \| \xsloss{f} \|_\infty$ and $c = 1$.  Finally, in
the special case when the strong $\bar{\eta}$-central condition holds,
we can take $\epsilon = 0$ and $v(0) = \bar{\eta}$ and
Theorem~\ref{thm:main-bounded-a} specializes to
Theorem~\ref{thm:firstriskbound}.

In Section~\ref{sec:vppc} below we introduce the $v$-PPC
condition. One of the main results of \Citet{erven2015fast} (in their
Section 5) is that, for bounded excess losses, the $v$-central
condition holds for some $v$ with
$v(\epsilon) \asymp \epsilon^{1- \beta}$ if and only if the $v$-PPC
condition hold for some $v$ with
$v(\epsilon) \asymp \epsilon^{1- \beta}$ if and only if the Bernstein
condition holds for exponent $\beta$ and some $B > 0$; the three conditions
are thus equivalent up to constant factors in the bounded excess loss
case.  The best case of the Bernstein condition of $\beta = 1$
corresponds to a $v$ with $v(0) > 0$, i.e., to the strong central
condition.  The Bernstein condition is known to characterize the rates
that can be obtained in bounded excess loss problems for proper
learners, and the same thus holds for the $v$-central and $v$-PPC
conditions. It is also implied by the well-known \emph{Tsybakov margin
  condition} as long as $\cF$ contains the Bayes optimal classifier
(see \citep{lecue2011interplay} and \Citep{erven2015fast} for
discussion).

We now illustrate Theorem~\ref{thm:main-bounded-a}
for the case of ERM over certain parametric classes when the
$v$-central condition holds for $v$ of the form $v(\epsilon) \asymp \epsilon^{1-\beta}$, so that a Bernstein condition holds with exponent $\beta$. 
We will see that for bounded losses our result recovers, 
up to log factors, rates that are known to be minimax optimal. 
We first need some notation. 
\glsadd{covering-number} For a pseudo-metric space $(\mathcal{A}, \|\cdot\|)$ and any $\epsilon > 0$, 
let $\mathcal{N}(\mathcal{A}, \|\cdot\|, \epsilon)$ be the $\epsilon$-covering number of $(\mathcal{A}, \epsilon)$, 
defined as the minimum number of radius-$\epsilon$ balls whose union contains $\mathcal{A}$.

\begin{myexample}{Lipschitz (and Bounded) Loss}\label{ex:bounded-bernstein} 
Suppose that \emph{(i)} for each $z \in \cZ$, the loss $\loss$ is
  $G$-Lipschitz as a function of $f \in \cF$; \emph{(ii)} $\cF$ has
  bounded metric entropy in some pseudometric $\|\cdot\|$; and \emph{(iii)} the loss is uniformly bounded over $\cF$ (so that a witness condition holds). Let
  $\cF_\epsilon$ be an optimal $\epsilon$-net with respect to
  $\|\cdot\|$. Take a uniform prior over $\cF$, and (purely for the analysis) consider the
  randomized predictor $\dolest$ that predicts by drawing an $f$
  uniformly from a radius-$\epsilon$ ball around $\hat{f}$, the ERM
  predictor. If the $v$-central condition holds, it follows that the
  information complexity of $\dolest$ is bounded as $G \epsilon +
  \frac{\log \mathcal{N}(\cF, \|\cdot\|, \epsilon)}{v(\varepsilon) n}$. To see
  this, for any $A \subset \cF$ let $A^\epsilon$ be the
  $\epsilon$-extension of $A$, defined as $\left\{ f \in \cF \colon
    \inf_{f' \in A} \|f - f'\| \leq \epsilon \right\}$. Then observe
  that
\begin{align*}
e^{\KL( \dol_n \pipes \Prior)} 
= \frac{\volume(\cF)}{\volume(\{\hat{f}\}^\epsilon)} 
\leq \frac{\volume(\bigcup_{f \in \cF_\epsilon} \{f\}^\epsilon)}
                  {\volume(\{\hat{f}\}^\epsilon)} 
\leq \frac{\sum_{f \in \cF_\epsilon} \volume(\{f\}^\epsilon))}
                 {\volume(\{\hat{f}\}^\epsilon)} 
= \mathcal{N}(\cF, \|\cdot\|, \epsilon) .
\end{align*}
Moreover, it is easy to see that the risk of standard ERM (rather than
its randomized version) over the entire class $\cF$ is at most the risk of
$\dol_n$ plus an additional $G \epsilon$. Hence, if $v$ satisfies
$v(\epsilon) = C \epsilon^{1- \beta}$ for some $\beta \in [0, 1]$ and
if the metric entropy is logarithmic in $\epsilon$, then by tuning
$\epsilon$ and $\eta$ as in (\ref{eq:tuning}) we see from
(\ref{eqn:v-central-bound}) that ERM obtains a rate of
$\tilde{O}(n^{-1/(2 - \beta)})$ (suppressing $\log$-factors) with high
probability --- which is the minimax optimal rate in this setting \Citep{erven2015fast}. Note that the Bernstein condition is automatically
satisfied for $\beta = 0$, yielding the slow rate of
$\tilde{O}(1/\sqrt{n})$, and the other extreme of $\beta = 1$ yields
the fast rate of $\tilde{O}(1/n)$. 
\end{myexample}

\subsection{The $v$-PPC Condition and the GRIP}
\label{sec:vppc}
Trivially, if the $v$-central condition holds for some function $v$, then there
exists $\epsilon > 0$ such that, with
$c = e^{\epsilon v(\epsilon)}$, for all $f \in \cF$,
$\E[ e^{- v(\epsilon) L_f} ] \leq c$, so that $- L_f$ must have 
a uniformly exponential upper tail as in
Definition~\ref{def:uniformlyexponential}. Thus, if $- L_f$ has a
polynomial upper tail, the $v$-central condition cannot hold. The $v$-PPC condition is 
a further weakening of the $v$-central condition which can still hold in the latter case. 
We achieve this by replacing the comparator $f^*_{\epsilon}$ by
a more sophisticated pseudo-predictor $\grip{\eta}$, the \emph{generalized
  reversed information projection} (GRIP). The original projection
\citep{li1999estimation} was used in the context of density estimation
under $\log$ loss. We now extend it to general learning problems:
\begin{definition}[GRIP] \label{def:grip}
Let $\smtuple$ be a learning problem. 
Define\footnote{This transformation is known as \emph{entropification} \citep{grunwald1999viewing}. For $\eta=1$ and log-loss, pseudo-probability densities are just  standard probability densities, while for general $\eta$ and $\ell$, the analogy to probability densites is still useful, hence the name; in particular, $\xi_Q$ shares some properties of mixture  distributions \Citep{erven2015fast}.} 
the set of pseudoprobability densities
\glsadd{pseudoprobs} $\mathcal{E}_{\cF,\eta} := \left\{ e^{-\eta \loss_f} : f \in \cF \right\}$. 
For $Q \in \Delta(\cF)$, define \glsadd{mix-pseudoprob} $\xi_Q := \E_{\rv{f} \sim Q} [ e^{-\eta \loss_{\rv{f}}} ]$. 
The generalized reversed information projection of $\Pt$ onto $\convhull(\mathcal{E})$ 
is defined as the pseudo-loss $\oldgriplosseta$ satisfying \glsadd{grip}
\begin{align*}
\E [ \oldgriplosseta ] 
= \inf_{Q \in \Delta(\cF)} \E \left[ -\frac{1}{\eta} \log \E_{\rv{f} \sim Q} \bigl[ e^{-\eta \loss_{\rv{f}}} \bigr] \right] 
= \inf_{\xi_Q \in \convhull(\mathcal{E})} \E \left[ -\frac{1}{\eta} \log \xi_Q \right] .
\end{align*}
Following terminology from the individual-sequence prediction literature, we
call the quantity appearing in the center expectation above a ``mix loss'' \Citep{rooij2014follow} defined for a distribution $Q
\in \Delta(\cF)$ as \glsadd{mixloss-Q} $\mixloss{\eta}{Q} := -\frac{1}{\eta} \log \E_{\rv{f} \sim
  Q} \bigl[ e^{- \eta \loss_{\rv{f}}} \bigr]$. The notion of mix loss can be extended from distributions to sets by defining, for any $A \subseteq \allactionset$, the object \glsadd{gripgen} $\gripgen{\eta}{A}$ as the pseudo-loss satisfying $\E \bigl[ \gripgen{\eta}{A} \bigr] = \inf_{Q \in \Delta(A \cup \{\fopt\})} \E \bigl[ \mixloss{\eta}{Q} \bigr]$.\footnote{The reason for automatically taking the union of $A$ with $\fopt$ is to lessen the notation for the mini-grip, introduced in Appendix~\ref{app:mini-grip}.} We thus have that $\oldgriplosseta = \grip{\eta}$, and we use the latter notation from here on out.
\end{definition}
Even though the GRIP is only a
{pseudo-predictor}, meaning that it may fail to correspond to any
actual prediction function, the corresponding loss for a GRIP
\emph{is} well-defined, as shown in Appendix~\ref{app:grip}. 
The main use of the GRIP lies in the fact that the
probability that its loss exceeds that of any $f \in \cF$ is
exponentially small:
\begin{proposition} \label{prop:grip-central} For all $f \in \cF$, for
  every $\eta > 0$, we have $\grip{\eta} - \loss_f \stochleq_\eta 0$.
\end{proposition}
The proposition implies that $\grip{\eta} \stochleq_{\eta}
\loss_{\fopt}$ and hence $\E[\grip{\eta}] \leq \E[\loss_{\fopt}]$ and,
for any $\eta > 0$, $\cF \cup \{ \grip{\eta} \}$ satisfies the
$\eta$-central condition, with $\grip{\eta}$ in the role of $f^*$.
We can now define the $v$-PPC condition: 
\begin{definition}[Pseudoprobability convexity (PPC) condition] \label{def:ppc} 
Let $\eta > 0$ and $\varepsilon \geq 0$. We say that $(P, \loss, \cF)$ satisfies the \glsaddcond{PPC-up-to-eps} $\eta$-PPC condition up to $\varepsilon$ if there exists some $\tilde{f} \in \cF$ such that
\begin{align} \label{eqn:ppceps} 
\E_{Z \sim \Pt} \left[ \loss_{\tilde{f}} \right] 
- \inf_{Q \in \Delta(\cF)} \E \left[ -\frac{1}{\eta} \log \E_{\rv{f} \sim Q} \bigl[ e^{-\eta \loss_{\rv{f}}} \bigr] \right]
\leq \epsilon, \text{\rm \ \  i.e.,\ \ } \E_{Z \sim \Pt} \left[ \loss_{\tilde{f}}
    - \grip{\eta}\right] \leq \epsilon.
\end{align}
Let $v: [0, \infty) \rightarrow [0, \infty)$ be a bounded,
non-decreasing function satisfying $v(\epsilon) > 0$ for all
$\epsilon > 0$. We say that $\smtuple$ satisfies the \glsaddcond{v-PPC} $v$-PPC
condition if, for all $\epsilon \geq 0$, there exists a function
$\tilde{f} \in \cF$ such that \eqref{eqn:ppceps} is satisfied with
$\eta = v(\epsilon)$.
\end{definition}
In both the $v$-central and $v$-PPC conditions, we look at pairs $(\eta,\epsilon)$
such that there exists a comparator $g$ which has risk no better than
$\E[\ell_{f^*}] - \epsilon$, and for which $(P,\ell, \cF \cup \{g \})$
satisfies the $\eta$-central condition. We achieve this for any
$(\eta,\epsilon)$ with $0 < \eta \leq v(\epsilon)$, where for
the $v$-central condition, the comparator was $g = f^*_{\epsilon}$ (see
\eqref{eq:unusual}) and for the $v$-PPC condition, it is $g = m_{\cF}^{\eta}$.

The name ``PPC'' stems from the fact that the condition expresses a
pseudo-convexity property of the set of pseudoprobability densities
mentioned in Definition~\ref{def:grip}; see \Citet{erven2015fast} for
a graphical illustration and for the proof that the $v$-central condition implies
the $v$-PPC condition for the same $v$.  We already mentioned that
\Citet{erven2015fast} (in their Section 5) proved the reverse implication, hence equivalence of the $v$-central and $v$-PPC conditions, up to constant factors, for bounded excess losses. 
To give some initial intuition for the unbounded case, we
note that the $v$-PPC condition is satisfied for
$v(\epsilon) = C \cdot \epsilon$ for a suitable constant $C$ whenever
the witness condition holds. While this was known for bounded excess
losses (where linear $v$ corresponds to the weakest Bernstein
condition, which automatically holds), by
Proposition~\ref{prop:slowrate} below it turns out to hold even if the
excess losses are heavy-tailed (so the $v$-central condition can never hold) and the
risk can be unbounded, as long as the second moment of the risk of
$\fopt$ is finite. This will imply, for example,
(Theorem~\ref{thm:main-unbounded} below and discussion) that the
``slow'' $\tilde{O}\left(1/\sqrt{n}\right)$ excess risk rate for
parametric models can be obtained in-probability by
$\eta_n$-generalized Bayes (with the optimal $\eta_n$ depending on the
sample size as $\eta_n \asymp 1/\sqrt{n}$) under hardly any
conditions.
\begin{proposition}\label{prop:slowrate}
  Let $\smtuple$ be such that for all $f \in \cF$, all $z \in \cZ$, $\loss_f(z) \geq 0$ and such that for some fixed  $u> 0$, for all $f \in \cF$ with $\xsrisk{f} > 0$, 
\begin{equation}\label{eq:positivethinking}
 \E \left[ (\xslosslong{f}) \cdot \ind{\xslosslong{f} \leq u} \right] \geq 0.
\end{equation}
(in particular this is implied by the $(u,c)$-witness condition
(\ref{eqn:emp-witness-badness-simple})).  Then for all $\eta
\leq 1/\E[\loss_{\fopt}]$,
$$
\E_{Z \sim P} \left[ \loss_{\fopt} - \grip{\eta} \right]  \leq  \eta \cdot e \cdot 
\left(u^2 + \frac{3}{2} \E[\loss_{\fopt}^2] \right). 
$$
As a consequence, if
$\E_{Z \sim P}\left[ \loss_{\fopt}^2 \right] < \infty$ then the $v$-PPC
condition holds with $v(\epsilon) = (C \epsilon) \opmin (1/\E[\loss_{\fopt}])$ with
$C = e^{-1} \cdot (u^2+ \frac{3}{2} \E[\loss_{\fopt}^2])^{-1}$.
\end{proposition}
The proof of this proposition is based on the following fact, interesting in its own right and also used in the proof of later results:
\begin{proposition}\label{prop:ppcprop}
  For given learning problem $\smtuple$, let $\loss'$ be such that (a)
  for all $f \in \cF$, all $z \in \cZ$,
  $\loss'_f(z) \leq \loss_f(z)$, and (b),
  $\loss'_{\fopt}(z) = \loss_{\fopt}(z)$. If the ``smaller-loss'' learning problem
  $(\Pt,\loss',\cF)$ satisfies the $v$-PPC condition for some function $v$, then so
  does $\smtuple$.
\end{proposition}

We now work towards a first risk bound under the $v$-PPC condition,
using the GRIP. The development is entirely analogous to that leading
up to Theorem~\ref{thm:main-bounded-a}, our risk bound under the
$v$-central condition. We start with the following result, which
essentially only differs from Lemma~\ref{lem:renyi-to-kl-simple} and
the corresponding lemma for the $v$-central condition and
$f^*_{\epsilon}$-comparator, Lemma~\ref{lem:renyi-to-kl-grip-a}, in
that it has $\loss_{\fopt}$ (as in Lemma~\ref{lem:renyi-to-kl-simple})
and $\ell_{f^*_{\epsilon}}$ (as in Lemma~\ref{lem:renyi-to-kl-grip-a})
on the right-hand side replaced by the GRIP loss $\grip{\bar{\eta}}$
and requires $\eta < \bar\eta/2$. The proof is much more involved
though since the comparators on the left and the right are not
connected in a straightforward manner.
\begin{lemma} \label{lem:renyi-to-kl-grip-b} Let $\smtuple$ be a
  learning problem and let $f \in \cF$. Let $\bar{\eta} > 0$.  
Suppose that (\ref{eqn:emp-witness-badness-simple}) holds for some $u> 0$ and $c \in (0,1]$, i.e., $(P,\ell,\{f, f^*\})$ satisfies the $(u,c)$-witness condition. Let $c'_u := \frac{1}{c} \frac{\eta \cdot u + 1}{1 - \frac{2
      \eta}{\bar{\eta}}}$.  Then for all $\eta \in (0, \bar{\eta}/2)$,
\begin{align}\label{eq:rklsimplegrip-b}
\xsrisk{f} 
\,\,\leq\,\, c'_{2 u} \cdot \Expann{\eta} \left[ \loss_f - \grip{\bar{\eta}} \right] .
\end{align}
In particular, if $(P,\ell, \cF)$ satisfies the $(u,c)$-witness condition then (\ref{eq:rklsimplegrip-b}) holds for all $f \in \cF$. 
\end{lemma}
Based on this lemma  it is now easy to prove analogues of
Theorem~\ref{thm:firstriskbound}. Below we first present our second
main result, an excess risk bound that holds under the basic witness
condition. The result allows unbounded and heavy-tailed losses but is
only useful when the excess risk is bounded; see
Section~\ref{sec:main-unbounded}.
\begin{theorem}[Excess Risk Bound - Bounded Excess Risk Case] \label{thm:main-bounded-b}
Let $\dolest$ be an arbitrary  learning
algorithm based on $\cF$. Assume that $\smtuple$ satisfies the $(u,c)$-witness condition (\ref{eqn:emp-witness-badness-simple}).
Let $c'_{u}$ be as in Lemma~\ref{lem:renyi-to-kl-grip-b}. Then 
under the $v$-PPC condition, for any $\eta < \frac{v(\epsilon)}{2}$,
\begin{align} \label{eqn:v-ppc-bound}
\E_{Z_1^n} \left[ \E_{\rv{f} \sim \dol_n} \left[ \xsrisk{\rv{f}} \right] \right]
\leq
c'_{2u} \left( \E_{Z_1^n} \left[ \rsc_{n,\eta}(\nocomparator \dolest) \right] + \epsilon \right) .
\end{align}
\end{theorem}
The result is entirely analogous to Theorem~\ref{thm:main-bounded-a}
(and the remarks made there apply here as well), with two differences:
first, $v$ is replaced by $v/2$, which will worsen the obtainable
bounds by a factor of 2 and hence will not affect the rates. Second,
the ESI in (\ref{eqn:v-central-bound}) is replaced by an
expectation. Thus, we have an exponential in-probability bound
(holding with probability $1- \delta$ up to an
$O(\log (1/\delta))$-term) under the $v$-central condition but not
under the $v$-PPC condition.  That such a deviation bound does not
hold under the $v$-PPC condition is inevitable since all of our bounds
are valid for ERM estimators, which, under heavy-tailed loss
distributions, are known to behave poorly in probability
\citep[Proposition 6.2]{catoni2012challenging}. There exist
specialized $M$-estimators for mean estimation problems
\citep{catoni2012challenging} or more generally (for regression
problems) that achieve better high-probability bounds by employing a
variation of the median-of-means idea
\citep{nemirovskii1983problem,hsu2016loss,lugosi2019regularization}.

To illustrate Theorem~\ref{thm:main-bounded-b}, we now provide an example where the $v$-central condition cannot hold because the excess risk has polynomially decaying tails; yet, the $v$-PPC condition may still hold for $v$ that allow for faster rates
than the ``slow'' $\tilde{O}(1/\sqrt{n})$.

\begin{myexample}{Heavy-tailed regression with bounded predictions} 
\label{ex:heavy-tailed-regression-II}
{We continue with the setting of Example~\ref{ex:heavy-tailed-regression}.} 
In addition to assuming that $\E [ Y^2 \mid X ] \leq C$ a.s.~for a constant $C$, we also assume that $\E [ |Y|^s ] < \infty$ for some $s \geq 2$; note that the first assumption already implies the second for $s = 2$. We further assume that $\cF$ has bounded metric entropy in $\sup$-norm, with covering numbers \mbox{$\mathcal{N}(\cF, \|\cdot\|_\infty, \epsilon)$} growing polynomially in $\epsilon$. Without subexponential tail decay, the $v$-central condition fails to hold for any non-trivial $v$; however, as shown by \Citet[Example 5.10]{erven2015fast} (based on a result of \cite{juditsky2008learning}), 
if $\E [ |Y|^s ] < \infty$ for some $s \geq 2$, then the $v$-PPC condition holds for $v(\epsilon) = O(\epsilon^{2/s})$.\footnote{What is actually shown there is that a property called $v$-stochastic exp-concavity holds, but, the results of that paper imply then that $v$-stochastic mixability holds which in turn implies that the $v$-PPC condition holds.} 
Moreover, as we showed in Example~\ref{ex:heavy-tailed-regression}, the witness condition holds if {$\E [ Y^2 \mid X ] < \infty$ a.s.}; there, we also established that the Bernstein condition holds with $\beta = 1$.

Now, take a uniform prior over $\cF$, and take the randomized predictor $\dolest$ as in Example~\ref{ex:bounded-bernstein} which randomizes over an $\epsilon$-ball around the ERM predictor $\hat{f}$. Then, for $s \geq 2$, Theorem \ref{thm:main-bounded-b} implies that the expected excess risk of $\dol_n$ is at most 
\begin{align*}
&\E_{Z_1^n} \left[ 
              \E_{\rv{f} \sim \dol_n} \left[ 
                  \frac{1}{n} \sum_{j=1}^n \xslossat{\rv{f}}{Z_j} 
              \right] 
          \right]
          + \frac{\log \mathcal{N}(\cF, \|\cdot\|, \epsilon)}{v(\epsilon) n} + \epsilon .
\end{align*}
The first term can be bounded as
\begin{align*}
&\E_{Z_1^n} \left[ 
              \E_{\rv{f} \sim \dol_n} \left[ 
                  \frac{1}{n} \sum_{j=1}^n \left( \xslossat{\hat{f}}{Z_j} + \loss_{\rv{f}}(Z_j) - \loss_{\hat{f}}(Z_j) \right)
              \right]
       \right] 
\\
&\leq 
\E_{Z_1^n} \left[ 
              \E_{\rv{f} \sim \dol_n} \left[ 
                  \frac{1}{n} \sum_{j=1}^n \left( \loss_{\rv{f}}(Z_j) - \loss_{\hat{f}}(Z_j) \right)
              \right]
       \right] 
\\
&= \E_{Z_1^n} \left[ 
              \E_{\rv{f} \sim \dol_n} \left[ 
                  \frac{1}{n} \sum_{j=1}^n \left( \rv{f}^2(X_j) - \hat{f}^2(X_j) + 2 Y_j (\hat{f}(X_j) - \rv{f}(X_j)) \right)
              \right]
           \right] \\
&\leq \E_{Z_1^n} \left[ 
               2 \epsilon \left( \|\cF\|_\infty + \left( \frac{1}{n} \sum_{j=1}^n Y_j^2 \right)^{1/2} \right)
           \right] ,
\end{align*}
which is at most $2 \epsilon \left( \|\cF\|_\infty + \|Y\|_{L_2(P)} \right) = O(\epsilon)$, and it is simple to verify that the ERM predictor $\hat{f}$ satisfies the same bound.  
Tuning $\epsilon$ in $O \left( \epsilon + \frac{\log \mathcal{N}(\cF, \| \cdot\|, \epsilon)}{\epsilon^{2/s} n} \right)$ yields a rate of $\tilde{O}(n^{-s/(s+2)})$ in expectation, where the notation hides log factors.
\end{myexample}
Two remarks are in order about the rate obtained in the above example. 

First, \cite{juditsky2008learning} previously obtained this rate for finite classes $\cF$ without the assumption that $\E [ Y^2 \mid X ]$ is almost surely uniformly bounded; their result is achieved by an online-to-batch conversion of a sequential algorithm which, after the conversion, plays actions in the convex hull of $\cF$. It is unclear if we truly need the assumption on the conditional second moment of $Y$ or if the need for this assumption is just an artifact of our analysis. In the regime where our stronger assumption holds, in the case of convex luckiness (see Example~\ref{ex:heavy-tailed-regression}) the rates obtained in the present paper match those of \cite{juditsky2008learning}. 
However, if convex luckiness does not hold, then the results of \cite{juditsky2008learning} still enjoy the rate of $\tilde{O}(n^{-s/(s+2)})$ whereas we cannot guarantee this rate. 
This is not surprising: without convex luckiness, ``improper learners'' that play in the convex hull of $\cF$ are inherently more powerful than (randomized) proper learners.

Second, even when convex luckiness does hold, the rate obtained in Example~\ref{ex:heavy-tailed-regression-II} above is not optimal.
The reason is that in the setting of this example, a Bernstein condition with $\beta = 1$ does hold, as was established earlier in Example~\ref{ex:heavy-tailed-regression}. 
Thus, via Corollary 6.2 of \cite{audibert2009fast} it is possible to obtain the better rate of
$\tilde{O}(1/n)$ in expectation using Audibert's SeqRand algorithm. 
Notably, the SeqRand algorithm for statistical learning involves using
a sequential learning algorithm which incorporates a second-order
loss-difference term. For new predictions, SeqRand employs an
online-to-batch conversion based on drawing functions uniformly at
random from the set of previously played functions. It is thus a
randomized proper learning algorithm. There are now two possibilities. The
first is that there exist $\cF$ satisfying the condition of
Example~\ref{ex:heavy-tailed-regression} for which ERM and
$\eta$-generalized Bayes simply do not achieve the rate of
$\tilde{O}(1/n)$; in that case either SeqRand's second-order nature or
its online-to-batch step may be needed to get the fast rate. The other
possibility is that ERM and $\eta$-generalized Bayes do generally attain the
fast rate under the Bernstein condition and a.s.~bounded
$\E[Y^2 \mid X]$-condition, in which case
Theorem~\ref{thm:main-bounded-b} is suboptimal for this situation --- we
return to this issue in the Discussion
(Section~\ref{sec:related-work}). In any case, SeqRand is
computationally intractable for most infinite classes, and we are not
aware of any polynomial-time learning algorithms that match the rate of SeqRand.

\subsection{Bounds for Unbounded Excess Risk}
\label{sec:main-unbounded}
We now present a result for a learning problem $\smtuple$ with
unbounded excess risk. 
Once again, the result follows (now with some work) from
Lemma~\ref{lem:renyi-to-kl-grip-b}, but now we need to be careful
because the $(u,c)$-witness condition with fixed $u$ and $c$ cannot be
expected to hold: it would become an exceedingly strong condition
for $\xsrisk{f} \rightarrow \infty$. We will thus require the
$\tau$-witness condition for a particular, easier $\tau$, namely
$\tau(x) = u (1 \opmax x)$ for some $u \geq 1$, so that for large $x$,
$\tau(x) \asymp x$. We first show, in
Proposition~\ref{prop:smallballagain} below, that at least for the
squared loss this condition can be expected to hold in a variety of
situations. The price to pay for using this $\tau$ is that we only get
in-probability results --- we show those in
Theorem~\ref{thm:main-unbounded} (we do not know whether
in-expectation results hold as well). Note that one could obtain better
constants in that theorem if one employed $\tau(x) = a \opmax (b x)$ for
the best possible $a$ and $b$, but for simplicity we did not do this.
\begin{proposition}[Bernstein plus small-ball implies unbounded witness]
\label{prop:smallballagain}
Consider the setting of Example~\ref{ex:heavy-tailed-regression}, 
i.e., regression with $\loss$ the squared loss and convex luckiness. 
We still assume convex luckiness and make the weaker assumption $\E[Y^2] < \infty$, 
but now we do not \emph{not} assume that the risk is bounded; 
i.e., we can have $\sup_{f \in \cF} \E[ \loss_f ] = \infty$. 
Fix some $b >0$ and suppose that there exists constants $\kappa > 0, \epsilon \in (0,1)$ such that
\begin{enumerate} 
\item for all $f \in \cF$ with $\xsrisk{f} > b$, Mendelson's (\citeyear{mendelson2014learning}) small-ball assumption (\ref{eqn:small-ball}) holds with constants $\epsilon,\kappa$ for $f, \fopt$ (i.e.~with $\fopt$ in the role of $h$),
\item For all $c_0 > b$, all  $f \in \cF$ with $\xsrisk{f} \leq c_0$, there is a $B$ such that the $(1, B)$-Bernstein condition holds, i.e., $\E [ \xsloss{f}^2 ] \leq B \xsrisk{f}$.
\end{enumerate}
Then $\smtuple$ satisfies the $(\tau,c)$-witness condition, with $\tau(x) = u (1 \opmax x)$ for some $u \geq 1$ and with $c \in (0, 1]$ which depends only on $\kappa$, $\epsilon$, $b$, and $\E [ \loss_{\fopt} ]$.
\end{proposition}
\begin{myexample}{Heavy-Tailed Regression,
    Continued} \label{ex:smallballagain} Mendelson provides several
  examples of convex $\cF$ for which the small-ball assumption holds; {the proposition above} shows that for all these examples, the
  $\tau$-witness condition holds as well as soon as, for $f$ with
  small excess risk, the Bernstein condition holds. For example,
  under the following ``meta''-condition the small-ball assumption
  holds (see \cite[Lemma 4.1]{mendelson2014learning}) and, as we show
  in Appendix~\ref{app:heavy-tailed-regression-example}, the Bernstein
  condition holds as well for $\cF_{c_0} := \{f \in \cF: \xsrisk{f} <
  c_0\}$, for all $c_0 \geq b$, as long as we assume convex luckiness (see Example~\ref{ex:heavy-tailed-regression}).
  \begin{align*}
\E[\loss_{\fopt}^2] < \infty \text{\quad and \quad} &\text{for some $A > 0$, for all $f \in \cF_{c_0}$}, \\
&\E[(f(X)-\fopt(X))^4]^{1/2} \leq A \cdot  \E[(f(X)- \fopt(X))^2] . 
  \end{align*}
We stress however that our theorem below does not recover Mendelson's
rates for $L_2(P)$-estimation error (Section~\ref{sec:related-work}),
which rely on further highly sophisticated analysis of the squared loss situation; 
our goal here is merely to show that our $\tau$-witness condition for the unbounded risk case is not a very strong one.
\end{myexample} 
\begin{theorem}[Excess Risk Bound - Unbounded Excess Risk
  Case] \label{thm:main-unbounded} Assume that $\smtuple$ satisfies
  the $(\tau,c)$-witness condition
  \eqref{eqn:emp-witness-badness-general} with
  $\tau : x \mapsto u(1 \opmax {x})$ for some $u \geq 1$ and constant
  $c$. Let $\epsilon_1, \epsilon_2, \ldots$ and $\eta_1, \eta_2, \ldots$ be
  sequences such that
  $$\epsilon_n \rightarrow 0, \ \ \ \ n \eta_n \rightarrow \infty.$$
  Let
  $c_u := \frac{u}{c} \frac{\eta_n + 1}{1 -
    \frac{\eta_n}{v(\epsilon_n)}}$ and
  $c'_{u} := \frac{u}{c} \frac{\eta_n + 1}{1 - \frac{2
      \eta_n}{v(\epsilon_n)}}$.
  Suppose that $\rsc_{n,\eta} := \rsc_{n,\eta}(\nocomparator \dolest)$ is nontrivial
  in the sense that $\E[{\rsc}_{n,\eta_n}] \rightarrow 0$.
  \begin{enumerate} \item Let {$\dolest \equiv (\estim{f},\Prior)$}
    represent a deterministic estimator. Suppose that, for given
    function $v$, the $v$-PPC condition holds and that for all $n$,
    $0 < \eta_n < v(\epsilon_n)/2$. Then for all $n$ larger than some
    $n_0$, the right-hand side of the following equation is bounded by
    $1$, and for all such $n$, for all $\delta > 0$, with
    probability at least $1-\delta$,
\begin{align}\label{eq:deterministicboundppc}
\xsrisk{\estim{f}}
\leq  \left( c'_{2u} \cdot  \frac{1}{\delta} \right) \cdot \bound, \ \ \text{with}\ \ \bound = 
 \left( \E\left[{\rsc}_{n,\eta_n}\right]  + \epsilon_n  \right).
\end{align} 
Now suppose that, more strongly, the $v$-central condition holds as well. 
Let $\overline{\rsc}_{n,\eta}$ be any upper bound on
  $\rsc_{n,\eta}(\fopt \| \dolest)$ that is nontrivial in that 
$\E[\overline{\rsc}_{n,\eta_n}] \rightarrow 0$. Let
$C_{n,\delta}$ be a function of
$\delta\in (0,1)$ such that for all $\delta \in (0,1)$,
$C_{n,\delta} > 2 \log(2/\delta)$ and
\begin{equation}\label{eq:rscupperbound}
P\left(\overline{\rsc}_{n,\eta_n} \geq C_{n,\delta} \cdot \E\left[\overline{\rsc}_{n,\eta_n}\right] \right) \leq \delta.
\end{equation} 
Then for all $n$ larger than some
    $n_0$, the right-hand side of the following equation is bounded by
    $1$, and for all such $n$, for all $0< \delta < 1$, with
    probability at least $1-\delta$, 
\begin{align}\label{eq:deterministicboundcentral}
\xsrisk{\estim{f}}
\leq  \left( c'_{u+\epsilon_n} \cdot C_{n,\delta} \right) \cdot \bound, \ \ \text{with}\ \ \bound = 
 \left( \E\left[\overline{\rsc}_{n,\eta_n}\right]  + \epsilon_n + \frac{2}{n\eta_n}  \right).
\end{align} 
\item Now let $\dolest$ be a general,
  potentially nondeterministic estimator, suppose that the $v$-PPC condition holds and let
  $\overline{\rsc}_{n,\eta_n}$ be any bound on $\rsc(\nocomparator \dolest)$
  that is slightly larger than ${\rsc}_{n,\eta_n}$, i.e., there
  exist a sequence $a_1, a_2, \ldots \rightarrow \infty$ such that,
  for all $n$, all $z^n$,
  $\overline{\rsc}_{n,\eta_n}\geq  a_n {\rsc}_{n,\eta_n}$. Then
\begin{equation} \label{eq:ggvtypebound}
\dol_n\left(\left\{ f \in \cF : \xsrisk{f} > c'_{2u} \cdot \bound
\right\} \right) \rightarrow 0 \text{\ in $P$-probability},
\end{equation}
with $\bound = 
\E\left[\overline{\rsc}_{n,\eta_n}\right]  + \epsilon_n $.
\end{enumerate}\end{theorem} 
When $\dolest$ represents a deterministic
  estimator $\estim{f}$ such as an $\eta$-two part MDL estimator, the result 
   is just a standard convergence-in-probability result. For learning algorithms that output a distribution such as generalized Bayes, the result seems fairly weak as nothing is said about the rate at which the deviation probability goes to $0$. Note, however, that the same holds for most standard results about posterior convergence in Bayesian statistics; for example, the results of GGV (see Example~\ref{ex:ggv}) are stated in exactly the same manner. 

   Note that the factor for the PPC-results increases quickly with
   $\delta$; depending on how strong a bound (\ref{eq:rscupperbound})
   can be given, the $v$-central results can thus become substantially
   stronger asymptotically. This is the case even though their bound
   has an additional $1/ (n \eta_n)$ term. Indeed, this extra term is
   of the right order, comparable to the upper bound on
   $\rsc_{n,\eta_n}$ given by (\ref{eq:optimalrates}). Therefore, for
   $v(x) \asymp x^{1-\beta}$, optimization of $\epsilon_n$ and $\eta_n$
   can be done in the same way as for the bounded risk case, leading to
   a rate of $\tilde{O}(n^{-1/(2- \beta)})$ as in
   (\ref{eq:tuning}). To give an example in which the bound for
   the $v$-central condition gets a better dependence on $\delta$ than $v$-PPC
   consider generalized Bayesian posteriors under the GGV condition
   (\ref{eq:ggva}) discussed in Section~\ref{sec:complexity};
   in that case, we get the bound (\ref{eq:antwerp}) which implies
   (\ref{eq:rscupperbound}) for a $C_{n,\delta} = o(\delta^{-1/2})$
   (rather than the $O(\delta^{-1})$ in the PPC-result) and with
   $\epsilon_n$, as defined there used as an upper bound on
   $\rsc_{n,\eta}$.  Still, in this example $C_{n,\delta}$ is
   polynomial in $\delta$ whereas Theorem~\ref{thm:main-bounded-a} had
   only a logarithmic dependence on $\delta$. As mentioned earlier,
   this stronger dependence on $\delta$ is unavoidable as the results
   under the $v$-PPC condition apply to methods like ERM, which have
   poor deviation properties.

   To derive further corollaries from this theorem, we mention the
   following extension of Proposition~\ref{prop:slowrate}:
 \begin{proposition}[when $(\tau,c)$-witness implies $v$-PPC] \label{prop:slowerrate}
   Suppose that the $(\tau,c)$-witness condition holds for given
   learning problem $\smtuple$ with $\tau : x \mapsto u(1 \opmax {x})$
   for some $u \geq 1$ and constant $c \in (0, 1]$ as in
   Theorem~\ref{thm:main-unbounded}. Further suppose that that
   $\loss_f(z) \geq 0$ for all $f \in \cF$ and all $z \in \cZ$.  Then
   the $v$-PPC condition holds with
   $v(\epsilon) = (C \epsilon) \opmin (1/\E[\loss_{\fopt}])$, where
   {$C = e^{-1} \cdot ( u^2\left(1 \opmax \left( \E[\ell_{f^*}]/c\right)^2  \right)  + \frac{3}{2}
   \E[\loss_{\fopt}^2])^{-1}$}.
 \end{proposition} 
 The above proposition implies that if the $\tau$-witness condition holds with
 $\tau$ as in Theorem~\ref{thm:main-unbounded} above, then the results
 (\ref{eq:deterministicboundppc}) and (\ref{eq:ggvtypebound})
 automatically hold with choice
 $2 \eta_n < (C \epsilon_n) \opmin (1/\E[\loss_{\fopt}])$, which for large
 $n$ is equivalent to $\eta_n < C \epsilon_n / 2$. For parametric
 $\cF$ we can take $\epsilon_n \asymp 1/\sqrt{n}$, so that the $v$-PPC condition is
 satisfied with $\eta_n \asymp 1/\sqrt{n}$. 
Thus, under quite weak conditions 
(for all $f,z$, $\loss_f(z) \geq 0$, $\E[\loss_{\fopt}^2] < \infty$, and the $\tau$-witness condition holds as above), 
but with unbounded, heavy tailed losses and without
 explicitly imposing any GRIP conditions, 
we get in all three cases of Theorem~\ref{thm:main-unbounded}, by choosing
 $\eta_n \asymp 1/\sqrt{n}$, 
that $\bound = \tilde{O}\left( 1/ \sqrt{{n}} \right)$. 
Consequently, even under very weak assumptions, 
we still get convergence for generalized $\eta_n$-Bayesian estimators, albeit at the ``slow'' rate.

\section{Discussion \& Open Questions}
\label{sec:related-work}
In this paper we presented several theorems that gave convergence
rates for general estimators, including pseudo-Bayesian and ERM
estimators, under general ``easiness conditions''. We end by
putting these conditions in context and discussing some of the
limitations of our approach, thereby pointing to avenues for future
work.
\paragraph{Easiness Conditions} We proved our 
convergence rates under the \emph{GRIP} conditions (the $v$-central and
$v$-PPC conditions) and the $\tau$-\emph{witness} condition, and we
provided some relations to other conditions such as convex luckiness
for squared loss (defined in Example~\ref{ex:heavy-tailed-regression}),
Bernstein conditions (Definition~\ref{def:bernstein}), and uniformly exponential tails (Definition~\ref{def:uniformlyexponential}). As promised in
the beginning of this paper, our conditions and results complement
those of \Citet{erven2015fast} which are mostly 
for the bounded case. The most important conditions of that paper that
did not show up here are \mbox{(a) the} extension of \emph{convex luckiness}
beyond the squared loss (it is formally defined for general losses by
\Citet{erven2015fast} under the name ``Assumption B'') and (b) the \emph{$v$-stochastic mixability} condition (see Definition 5.9 of \Citet{erven2015fast}). 
We will restrict discussion of the $v$-stochastic mixability condition to the case where the decision set $\cF_d$ from \Citet{erven2015fast} is equal to $\convhull(\cF)$. In the present paper, where the set $\mathcal{P}$ from \Citet{erven2015fast} is always equal to the singleton $\{P\}$, it is easy to see that $v$-stochastic mixability is equivalent to the $v$-PPC condition but with the minimizer $\fopt$ over $\cF$ replaced by the minimizer $\fopt_{\convhull}$ over $\convhull(\cF)$. 
\Citet{erven2015fast} show that for bounded excess losses, $v$-stochastic mixability characterizes obtainable rates for \emph{improper} learners that are allowed to play in the convex hull of $\cF$. 
$v$-stochastic mixability is in turn implied by the easiness conditions of \cite{juditsky2008learning}, (for constant $v$) by conditions on the loss function such as mixability and exp-concavity \citep{CesaBianchiL06}, and by strong convexity. 
For clarity we give an overview of the relevant implications between our conditions and those of \Citet{erven2015fast} in Figure~\ref{fig:overview}.
\begin{figure}
\begin{center}
  \begin{tabular}{|p{1.7cm}|p{1.5cm}|p{2cm}|p{8cm}|} \hline
    excess loss is... & condition type & loss function & result \\
    \hline \hline bounded & GRIP & general & $v$-PPC
    $\Leftrightarrow$ $v$-central  (vE) \\
    & & & $x^{1-\beta}$-PPC $\Leftrightarrow$ $x^{\beta}$-Bernstein (vE) \\
    \hline
    & witness & general & $(u,c)$-witness always holds (trivial) \\
    \hline \hline unbounded & GRIP & general & convex luckiness $+$
    $v$-stochastic mixability $\Rightarrow$ $v$-PPC (vE)
    \\
    & & general & $v$-central $\Rightarrow$ $v$-PPC (vE) \\
    & & general & $v$-central $\Rightarrow$ $L_f$ has uniformly
    exponential lower tail (vE)
    \\
    & & log loss & convex luckiness $\Rightarrow$ $1$-central (vE)
    \\
    & & squared loss & convex luckiness $+$ bounded predictions $+$ $Y \mid X$ has a.s.~uniformly bounded 2${^\mathrm{nd}}$ moment
    $\Rightarrow$ $(1,B)$-Bernstein (GM,
    Example~\ref{ex:heavy-tailed-regression}) \\ \hline
    unbounded  & witness & general  & $(\beta,B)$-Bernstein $\Rightarrow$ $(\tau,c)$-witness, $\tau(x) \asymp x^{\beta-1}$ (GM, Proposition~\ref{prop:bernstein-witness}) \\
    &  &  general  & $L_f$ has uniformly exponential upper tail  $\Rightarrow$ $(\tau,c)$-witness, $\tau(x) \asymp 1 \opmax \log(1/x)$  (GM,  Lemma~\ref{lem:kl-hell-exp-tails}) \\
    & & log loss, correct model & Wong-Shen $\Leftrightarrow$ $L_f$
    has uniformly
    exponential tails (GM, Example~\ref{ex:wong}) \\
    \hline
\end{tabular}\caption{\label{fig:overview}GM stands for ``established in the present paper'', vE refers to \Citet{erven2015fast}. All implications hold up to constant factors. Note that boundedness always refers to {\em excess loss}. For example, for Lipschitz losses on a bounded domain, the losses themselves may have heavy tails but the excess loss will be bounded. }
\end{center}\end{figure}
\paragraph{Misspecification}
We showed that our methods are particularly well-suited for proving a form of consistency 
for (generalized Bayesian) density estimation under
misspecification; under only the $\bar\eta$-central condition, a weak condition on the support of $p_{f^*}$, and using a prior such that the weakened GGV condition (\ref{eq:ggvb}) holds, 
we can show that for any $\eta < \bar\eta$, the $\eta$-generalized Bayesian posterior is consistent in the sense of our misspecification metric (see Proposition~\ref{prop:entrobound} and discussion below it). 
As stated there, an interesting open question is under which conditions the metric entropy for the misspecified case is of the same order as the metric entropy for the well-specified case, as then the misspecification metric dominates the standard Hellinger metric.

\paragraph{Proper vs.~Improper}  
There exist learning problems $\smtuple$ on which no proper learner --- one which always predicts inside $\cF$ --- can achieve a rate as good as that of an improper learner, which can select $\hat{f}_n \not \in \cF$
\Citep{audibert2007progressive,erven2015fast}.  In this paper we considered
\emph{randomized} proper estimators, to which the same lower bounds
apply;
hence, they cannot in general compete with improper methods 
such as exponentially weighted average forecasters and other aggregation methods. 
Such methods achieve fast rates under conditions such as stochastic exp-concavity \citep{juditsky2008learning}, 
which imply the ``stochastic mixability'' condition that, 
as explained by \Citet{erven2015fast}, is sufficient for fast rates for aggregation methods. 
To get rates comparable to those of improper learners, we invariably need to
make a ``convex luckiness'' assumption under which, as again shown by
\Citet{erven2015fast}, $v$-stochastic mixability implies the $v$-PPC condition 
(see also Figure~\ref{fig:overview}); the latter allows for fast rates for
randomized proper learners. An interesting question for future work is
whether our proof techniques can be extended to incorporate, and get the
right rates for, improper methods such as the empirical star estimator 
\citep{audibert2007progressive} and Q-aggregation
\citep{lecue2014optimal}. Since the original analysis of these methods
bears some similarity to our techniques, this might very well be
possible.

While superior rates for improper learners are inevitable, it is more worrying that the rate we showed for ERM in heavy-tailed bounded regression is worse than the rate for the SeqRand algorithm, which is also randomized proper (see Example~\ref{ex:heavy-tailed-regression-II} and text below it). 
We do not know whether the rate we obtain is the actual worst-case rate that ERM achieves under our conditions, or whether ERM achieves the same rate as SeqRand, or something in between. In the latter two cases, it would mean that our bounds are suboptimal. Sorting this out is a major goal for future work.

\paragraph{Empirical process vs Information-theoretic}
Broadly speaking, one can distinguish approaches to proving excess
risk bounds into two main groups: on the one hand are approaches based
on empirical process theory (EPT) such as
\citep{bartlett2005local,bartlett2006empirical,koltchinskii2006local,mendelson2014learning,liang2015learning,dinh2016fast}
and most work involving VC dimension in classification.  On the other
hand are information-theoretic approaches based on prior measures,
change-of-measure arguments, and KL penalties such as PAC-Bayesian and
MDL approaches
\citep{barron1991minimum,li1999estimation,catoni2003pac,audibert2004pac,Grunwald07,audibert2009fast}.
A significant advantage of EPT approaches is that they often can achieve optimal rates of convergence for ``large'' models $\cF$ with metric entropy $\log \mathcal{N}(\cF,\| \cdot \|, \epsilon)$ that increases polynomially in $1/\epsilon$, where $\| \cdot \|$ is the $L_1(P)$ or $L_2(P)$-metric. 
Prior-based approaches (including the one in this paper) may yield suboptimal rates in such cases (see \cite{audibert2009fast} for discussion). A closely related advantage of EPT approaches is that they can handle empirical covers of ${\cal F}$, thus allowing one to prove bounds for VC classes, among others.

An advantage of prior-based approaches is that they inherently
penalize, so that whenever one has a countably infinite union of
classes $\cF = \bigcup_{j \in {\mathbb N}} \cF_j$, the approaches
automatically adapt to the rate that can be obtained as if the best
$\cF_j$ containing $\fopt$ were known in advance; this adaptation was
illustrated at various places in this paper (see final display in
Proposition~\ref{prop:transavia}, equation \eqref{eq:aggregation}). This happens even if
for every $n$, there is a $j$ and $f \in \cF_j$ with empirical error
$0$; in such a case unpenalized methods as often used in EPT methods
would overfit. In the paper \citep{grunwald2017a}, a companion paper
to the present one, we show for bounded excess losses that the two
approaches may be combined.  In fact one can provide a single excess
risk bound in which the information complexity is replaced by a
strictly smaller quantity and instead of a prior one uses a more
general ``luckiness function'' \citep{Grunwald07} that is better suited
for dealing with penalized estimators.  For some choices of luckiness
function, one gets a slight strengthening of the excess risk bounds
given in this paper; for other choices, one gets bounds in terms of
Rademacher complexity, $L_2(P)$ and empirical $L_2(P_n)$ covering
numbers.  Thus, the best of both worlds is achievable, but for the
time being only for bounded excess losses.

Another major goal for future work is thus to provide such a combined EPT-information theoretic bound for unbounded excess losses that allows for heavy-tailed excess loss. 
Within the EPT literature, some work has been done: \cite{mendelson2014learning,mendelson2017learning} provides bounds on 
the $L_2(P)$-estimation error $\| \hat{f} - \fopt\|_{L_2(P)}^2$ 
and \cite{liang2015learning} on the related squared loss risk. 
For other loss functions not much seems to be known: 
\cite{mendelson2017learning} shows that improved $L_2(P)$-estimation
error rates may be obtained by using other, proxy loss functions
during training; 
however, the target remains $L_2(P)$-estimation. 
In contrast, our approach allows for general loss functions $\loss$
including density estimation, but we do not specially study proxy training losses. 

These last three EPT-based works can deal with $\smtuple$ with
unbounded excess (squared loss) risk. This is in contrast to earlier
papers in the information-theoretic/PAC-Bayes tradition; as far as we
know, our work is the first one that allows one to prove excess risk
convergence rates in the unbounded risk case
(Theorem~\ref{thm:main-unbounded}) for general models including
countable infinite unions of models as in
Proposition~\ref{prop:transavia}. Previous works dealing with {unbounded excess loss} all rely on a Bernstein condition --- we are aware of
\citep{zhang2006epsilon}, requiring $\beta = 1$; 
\citep{audibert2004pac}, for the transductive setting rather than our
inductive setting; and, the most general, \citep{audibert2009fast}.
However, for convex or linear losses, a Bernstein condition can
\emph{never} hold if $\sup_{f \in \cF} \xsrisk{f}$ is unbounded, as follows
trivially from inspecting Definition~\ref{def:bernstein}, whereas the
$v$-central and PPC-conditions \emph{can} hold. 
See for instance Example \ref{ex:no-bernstein-unbounded} in Appendix~\ref{app:examples}, 
where $\cF$ is just the densities of the normal location family without any bounds on the mean: 
here the Bernstein condition must fail, yet the strong central condition and the witness condition both hold and thus Theorem~\ref{thm:main-unbounded} applies (for some moderate $M$).

In the unbounded-excess-loss-yet-bounded-risk case, the difference between
these works and ours opaques:
there may well be cases (though we have not produced one) where the Bernstein condition holds for some $\beta$ but the $v$-PPC condition does not hold for $v(\epsilon) \asymp \epsilon^{1-\beta}$; the opposite certainly can happen (note however that in the bounded excess loss case these two conditions are equivalent; see Figure~\ref{fig:overview}). 
Indeed, Example~\ref{ex:no-bernstein-bounded} in Appendix~\ref{app:examples} exhibits an $\cF$ for which the excess risk is bounded but its second moment is not, whence the Bernstein condition fails to hold for \emph{any} positive exponent, while both the strong central condition and the witness condition hold. Theorem \ref{thm:main-bounded-b} therefore applies whereas the results of \cite{audibert2009fast} and \cite{zhang2006information} do not. 
Finally we note that \cite{audibert2009fast} proves his bounds for his ingenious SeqRand learning algorithm, whereas Zhang's and our bounds hold for general estimators. 

Yet another major goal for current work is thus to disentangle the
role of the PPC condition and the Bernstein condition for unbounded excess losses; 
ideally we would extend our bounds to cover faster rates under a weaker
condition implied by either of the Bernstein or PPC conditions.

\paragraph{Additional future work: learning $\eta$}
A general issue with generalized Bayesian and MDL methods, but one that is avoided by ERM, 
is the fact that they depend on the learning rate parameter $\eta$. 
While this is often pragmatically resolved by cross-validation 
(see e.g.~\cite{audibert2009fast} and many others), 
\cite{grunwald2011safe,Grunwald12} give a method for learning $\eta$
that provably finds the ``right'' $\eta$ (i.e.~optimal for the best
Bernstein condition that holds for the given learning problem) for
{bounded excess loss functions} and likelihood ratios; experiments
\Citep{GrunwaldO17} indicate that this ``safe Bayesian'' method works
excellently in the unbounded case as well.  While it seems that the
proof technique to handle learning $\eta$ carries over to the present
unbounded setting, actually proving that the SafeBayes method still
works remains a task for future work.

\section{Acknowledgments}
We would like to thank Bob Williamson (who saw the use of bringing
structure in the various existing 'easiness conditions'), Tim van
Erven (who brought our attention to the fact that GGV's entropy
condition is not needed whenever $\eta < 1$, as earlier noted by Tong
Zhang) and Bas Kleijn (who brought \cite{wong1995probability} to our
attention), Andrew Barron (for various discussions), and Alice Kirichenko (for pointing out numerous small mistakes). An anonymous referee made some highly useful suggestions.  This
research was supported by the Netherlands Organization for Scientific
Research (NWO) VICI Project Nr.  639.073.04.
\bibliography{fast_and_unbounded}

\begin{thebibliography}{70}
\providecommand{\natexlab}[1]{#1}
\providecommand{\url}[1]{\texttt{#1}}
\expandafter\ifx\csname urlstyle\endcsname\relax
  \providecommand{\doi}[1]{doi: #1}\else
  \providecommand{\doi}{doi: \begingroup \urlstyle{rm}\Url}\fi

\bibitem[Audibert(2004)]{audibert2004pac}
Jean-Yves Audibert.
\newblock {PAC-B}ayesian statistical learning theory.
\newblock \emph{Th\`ese de doctorat de l'Universit{\'e} Paris}, 6:\penalty0 29,
  2004.

\bibitem[Audibert(2007)]{audibert2007progressive}
Jean-Yves Audibert.
\newblock Progressive mixture rules are deviation suboptimal.
\newblock In \emph{NIPS}, 2007.

\bibitem[Audibert(2009)]{audibert2009fast}
Jean-Yves Audibert.
\newblock Fast learning rates in statistical inference through aggregation.
\newblock \emph{The Annals of Statistics}, 37\penalty0 (4):\penalty0
  1591--1646, 2009.

\bibitem[Audibert and Tsybakov(2007)]{audibert2007fast}
Jean-Yves Audibert and Alexandre~B. Tsybakov.
\newblock Fast learning rates for plug-in classifiers.
\newblock \emph{The Annals of statistics}, 35\penalty0 (2):\penalty0 608--633,
  2007.

\bibitem[Barron et~al.(1999)Barron, Schervish, and
  Wasserman]{barron1999consistency}
Andrew Barron, Mark~J. Schervish, and Larry Wasserman.
\newblock The consistency of posterior distributions in nonparametric problems.
\newblock \emph{The Annals of Statistics}, 27\penalty0 (2):\penalty0 536--561,
  1999.

\bibitem[Barron and Cover(1991)]{barron1991minimum}
Andrew~R. Barron and Thomas~M. Cover.
\newblock Minimum complexity density estimation.
\newblock \emph{IEEE Transactions on Information Theory}, 37\penalty0
  (4):\penalty0 1034--1054, 1991.

\bibitem[Bartlett and Mendelson(2006)]{bartlett2006empirical}
Peter~L. Bartlett and Shahar Mendelson.
\newblock Empirical minimization.
\newblock \emph{Probability Theory and Related Fields}, 135\penalty0
  (3):\penalty0 311--334, 2006.

\bibitem[Bartlett et~al.(2005)Bartlett, Bousquet, and
  Mendelson]{bartlett2005local}
Peter~L. Bartlett, Olivier Bousquet, and Shahar Mendelson.
\newblock Local {R}ademacher complexities.
\newblock \emph{The Annals of Statistics}, 33\penalty0 (4):\penalty0
  1497--1537, 2005.

\bibitem[Bhattacharya et~al.(2019)Bhattacharya, Pati, Yang,
  et~al.]{bhattacharya2019bayesian}
Anirban Bhattacharya, Debdeep Pati, Yun Yang, et~al.
\newblock Bayesian fractional posteriors.
\newblock \emph{The Annals of Statistics}, 47\penalty0 (1):\penalty0 39--66,
  2019.

\bibitem[Bickel and Kleijn(2012)]{bickel2012semiparametric}
Peter~J. Bickel and Bas~J.K. Kleijn.
\newblock The semiparametric {B}ernstein--von {M}ises theorem.
\newblock \emph{The Annals of Statistics}, 40\penalty0 (1):\penalty0 206--237,
  2012.

\bibitem[Birg\'e(2004)]{Birge04}
Lucien Birg\'e.
\newblock Model selection for {G}aussian regression with random design.
\newblock \emph{Bernoulli}, 10\penalty0 (6):\penalty0 1039--1051, 2004.

\bibitem[Birg{\'e} and Massart(1998)]{birge1998minimum}
Lucien Birg{\'e} and Pascal Massart.
\newblock Minimum contrast estimators on sieves: exponential bounds and rates
  of convergence.
\newblock \emph{Bernoulli}, 4\penalty0 (3):\penalty0 329--375, 1998.

\bibitem[Bissiri et~al.(2016)Bissiri, Holmes, and Walker]{bissiri2016general}
Pier~Giovanni Bissiri, Chris~C. Holmes, and Stephen~G. Walker.
\newblock A general framework for updating belief distributions.
\newblock \emph{Journal of the Royal Statistical Society: Series B (Statistical
  Methodology)}, 2016.

\bibitem[Catoni(2003)]{catoni2003pac}
Olivier Catoni.
\newblock A {PAC-B}ayesian approach to adaptive classification.
\newblock Technical report, Laboratoire de Probabilit\'es et Mod\`eles
  Al\'eatoires, Universit\'es Paris 6 and Paris 7, 2003.

\bibitem[Catoni(2012)]{catoni2012challenging}
Olivier Catoni.
\newblock Challenging the empirical mean and empirical variance: a deviation
  study.
\newblock \emph{Annales de l'Institut Henri Poincar{\'e}, Probabilit{\'e}s et
  Statistiques}, 48\penalty0 (4):\penalty0 1148--1185, 2012.

\bibitem[Cesa-Bianchi and Lugosi(2006)]{CesaBianchiL06}
Nic\`olo Cesa-Bianchi and G\'abor Lugosi.
\newblock \emph{Prediction, Learning and Games}.
\newblock Cambridge University Press, Cambridge, UK, 2006.

\bibitem[Cortes et~al.(2019)Cortes, Greenberg, and Mohri]{CortesGM19}
Corinna Cortes, Spencer Greenberg, and Mehryar Mohri.
\newblock Relative deviation learning bounds and generalization with unbounded
  loss functions.
\newblock \emph{Ann. Math. Artif. Intell.}, 85\penalty0 (1):\penalty0 45--70,
  2019.

\bibitem[de~Rooij et~al.(2014)de~Rooij, van Erven, Gr{\"u}nwald, and
  Koolen]{rooij2014follow}
Steven de~Rooij, Tim van Erven, Peter~D. Gr{\"u}nwald, and Wouter~M. Koolen.
\newblock Follow the leader if you can, hedge if you must.
\newblock \emph{Journal of Machine Learning Research}, 15:\penalty0 1281--1316,
  2014.

\bibitem[Dinh et~al.(2016)Dinh, Ho, Nguyen, and Nguyen]{dinh2016fast}
Vu~C. Dinh, Lam~S. Ho, Binh Nguyen, and Duy Nguyen.
\newblock Fast learning rates with heavy-tailed losses.
\newblock In \emph{Advances in Neural Information Processing Systems 29}, pages
  505--513. Curran Associates, Inc., 2016.

\bibitem[Dudley(2002)]{dudley2002real}
Richard~M. Dudley.
\newblock \emph{Real analysis and probability}, volume~74.
\newblock Cambridge University Press, 2002.

\bibitem[Ghosal and Van Der~Vaart(2007)]{ghosal2007convergence}
Subhashis Ghosal and Aad~W. Van Der~Vaart.
\newblock Convergence rates of posterior distributions for noniid observations.
\newblock \emph{The Annals of Statistics}, 35\penalty0 (1):\penalty0 192--223,
  2007.

\bibitem[Ghosal et~al.(2000)Ghosal, Ghosh, and van~der Vaart]{GhosalGV00}
Subhashis Ghosal, Jayanta~K. Ghosh, and Aad~W. van~der Vaart.
\newblock Convergence rates of posterior distributions.
\newblock \emph{The Annals of Statistics}, 28\penalty0 (2):\penalty0 500--531,
  2000.

\bibitem[Ghosal et~al.(2008)Ghosal, Lember, and van~der Vaart]{GhosalLV08}
Subhashis Ghosal, J\"uri Lember, and Aad~W. van~der Vaart.
\newblock Nonparametric {B}ayesian model selection and averaging.
\newblock \emph{Electronic Journal of Statistics}, 2:\penalty0 63--89, 2008.

\bibitem[Gr{\"u}nwald(1999)]{grunwald1999viewing}
Peter~D. Gr{\"u}nwald.
\newblock Viewing all models as “probabilistic”.
\newblock In \emph{Proceedings of the twelfth annual conference on
  Computational learning theory}, pages 171--182. ACM, 1999.

\bibitem[Gr\"unwald(2007)]{Grunwald07}
Peter~D. Gr\"unwald.
\newblock \emph{The Minimum Description Length Principle}.
\newblock MIT Press, Cambridge, MA, 2007.

\bibitem[Gr{\"u}nwald(2011)]{grunwald2011safe}
Peter~D. Gr{\"u}nwald.
\newblock Safe learning: bridging the gap between {B}ayes, {MDL} and
  statistical learning theory via empirical convexity.
\newblock In \emph{COLT}, pages 397--420, 2011.

\bibitem[Gr{\"u}nwald(2012)]{Grunwald12}
Peter~D. Gr{\"u}nwald.
\newblock The safe {B}ayesian: learning the learning rate via the mixability
  gap.
\newblock In \emph{Proceedings 23rd International Conference on Algorithmic
  Learning Theory (ALT '12)}. Springer, 2012.

\bibitem[Gr{\"u}nwald and Dawid(2004)]{grunwald2004game}
Peter~D. Gr{\"u}nwald and A.~Philip Dawid.
\newblock Game theory, maximum entropy, minimum discrepancy and robust
  {B}ayesian decision theory.
\newblock \emph{The Annals of Statistics}, 32\penalty0 (4):\penalty0
  1367--1433, 2004.

\bibitem[Gr\"unwald and Mehta(2019)]{grunwald2017a}
Peter~D. Gr\"unwald and Nishant~A. Mehta.
\newblock A tight excess risk bound via a unified
  {PAC-Bayesian-Rademacher-Shtarkov-MDL} complexity.
\newblock In \emph{Proceedings 30th Conference on Algorithmic Learning Theory
  (ALT '19)}, 2019.

\bibitem[Gr\"unwald and {V}an Ommen(2017)]{GrunwaldO17}
Peter~D. Gr\"unwald and Thijs {V}an Ommen.
\newblock Inconsistency of {B}ayesian inference for misspecified linear models,
  and a proposal for repairing it.
\newblock \emph{Bayesian Analysis}, 2017.

\bibitem[Haussler and Opper(1997)]{haussler1997mutual}
David Haussler and Manfred Opper.
\newblock Mutual information, metric entropy and cumulative relative entropy
  risk.
\newblock \emph{The Annals of Statistics}, 25\penalty0 (6):\penalty0
  2451--2492, 1997.

\bibitem[Haussler et~al.(1996)Haussler, Kearns, Seung, and
  Tishby]{haussler1996rigorous}
David Haussler, Michael Kearns, H.~Sebastian Seung, and Naftali Tishby.
\newblock Rigorous learning curve bounds from statistical mechanics.
\newblock \emph{Machine Learning}, 25\penalty0 (2-3):\penalty0 195--236, 1996.

\bibitem[Heide et~al.(2019)Heide, Kirichenko, Gr{\"u}nwald, and
  Mehta]{DeHeideKGM19}
R.~De Heide, A.~Kirichenko, P.~Gr{\"u}nwald, and N.~Mehta.
\newblock Safe-{B}ayesian generalized linear regression.
\newblock \emph{arXiv preprint arXiv:1910.....}, 2019.

\bibitem[Hsu and Sabato(2016)]{hsu2016loss}
Daniel~J. Hsu and Sivan Sabato.
\newblock Loss minimization and parameter estimation with heavy tails.
\newblock \emph{The Journal of Machine Learning Research}, 17\penalty0
  (1):\penalty0 543--582, 2016.

\bibitem[Juditsky et~al.(2008)Juditsky, Rigollet, and
  Tsybakov]{juditsky2008learning}
Anatoli Juditsky, Philippe Rigollet, and Alexandre~B. Tsybakov.
\newblock Learning by mirror averaging.
\newblock \emph{The Annals of Statistics}, 36\penalty0 (5):\penalty0
  2183--2206, 2008.

\bibitem[Kleijn and van~der Vaart(2006)]{kleijn2006misspecification}
Bas~J.K. Kleijn and Aad~W. van~der Vaart.
\newblock Misspecification in infinite-dimensional {B}ayesian statistics.
\newblock \emph{The Annals of Statistics}, 34\penalty0 (2):\penalty0 837--877,
  2006.

\bibitem[Koltchinskii(2006)]{koltchinskii2006local}
Vladimir Koltchinskii.
\newblock Local {R}ademacher complexities and oracle inequalities in risk
  minimization.
\newblock \emph{The Annals of Statistics}, 34\penalty0 (6):\penalty0
  2593--2656, 2006.

\bibitem[Koolen et~al.(2016)Koolen, Gr{\"u}nwald, and van
  Erven]{koolen2016combining}
Wouter~M. Koolen, Peter Gr{\"u}nwald, and Tim van Erven.
\newblock Combining adversarial guarantees and stochastic fast rates in online
  learning.
\newblock In \emph{Advances in Neural Information Processing Systems}, pages
  4457--4465, 2016.

\bibitem[Kullback and Leibler(1951)]{kullback1951information}
Solomon Kullback and Richard~A. Leibler.
\newblock On information and sufficiency.
\newblock \emph{The Annals of Mathematical Statistics}, 22\penalty0
  (1):\penalty0 79--86, 1951.

\bibitem[Lecu{\'e}(2011)]{lecue2011interplay}
Guillaume Lecu{\'e}.
\newblock \emph{Interplay between concentration, complexity and geometry in
  learning theory with applications to high dimensional data analysis}.
\newblock Habilitation \`a diriger des recherches, Universit{\'e} Paris-Est,
  2011.

\bibitem[Lecu{\'e} and Rigollet(2014)]{lecue2014optimal}
Guillaume Lecu{\'e} and Philippe Rigollet.
\newblock Optimal learning with {$Q$}-aggregation.
\newblock \emph{The Annals of Statistics}, 42\penalty0 (1):\penalty0 211--224,
  2014.

\bibitem[Lee et~al.(1996)Lee, Bartlett, and Williamson]{lee1996efficient}
Wee~Sun Lee, Peter~L. Bartlett, and Robert~C. Williamson.
\newblock Efficient agnostic learning of neural networks with bounded fan-in.
\newblock \emph{IEEE Transactions on Information Theory}, 42\penalty0
  (6):\penalty0 2118--2132, 1996.

\bibitem[Li(1999)]{li1999estimation}
Qiang~(Jonathan) Li.
\newblock \emph{Estimation of mixture models}.
\newblock PhD thesis, Yale University, 1999.

\bibitem[Liang et~al.(2015)Liang, Rakhlin, and Sridharan]{liang2015learning}
Tengyuan Liang, Alexander Rakhlin, and Karthik Sridharan.
\newblock Learning with square loss: localization through offset {R}ademacher
  complexity.
\newblock In \emph{Proceedings of The 27th Conference on Learning Theory (COLT
  2015)}, pages 1260--1285, 2015.

\bibitem[Lugosi and Mendelson(2019)]{lugosi2019regularization}
G{\'a}bor Lugosi and Shahar Mendelson.
\newblock Regularization, sparse recovery, and median-of-means tournaments.
\newblock \emph{Bernoulli}, 25\penalty0 (3):\penalty0 2075--2106, 2019.

\bibitem[Martin et~al.(2017)Martin, Mess, and Walker]{martin2017empirical}
Ryan Martin, Raymond Mess, and Stephen~G. Walker.
\newblock Empirical {B}ayes posterior concentration in sparse high-dimensional
  linear models.
\newblock \emph{Bernoulli}, 23\penalty0 (3):\penalty0 1822--1847, 2017.

\bibitem[McAllester(2003)]{McAllester02}
David McAllester.
\newblock {PAC}-{B}ayesian stochastic model selection.
\newblock \emph{Machine Learning}, 51\penalty0 (1):\penalty0 5--21, 2003.

\bibitem[McCullagh and Nelder(1989)]{McCullaghN89}
Peter McCullagh and John Nelder.
\newblock \emph{Generalized Linear Models}.
\newblock Chapman and Hall/CRC, Boca Raton, second edition, 1989.

\bibitem[Meir and Zhang(2003)]{MeirZ03}
R.~Meir and T.~Zhang.
\newblock Generalization error bounds for {B}ayesian mixture algorithms.
\newblock \emph{Journal of Machine Learning Research}, 4:\penalty0 839--860,
  2003.

\bibitem[Mendelson(2014)]{mendelson2014learning}
Shahar Mendelson.
\newblock Learning without concentration.
\newblock In \emph{Proceedings of The 27th Conference on Learning Theory},
  pages 25--39, 2014.

\bibitem[Mendelson(2017{\natexlab{a}})]{mendelson2017aggregation}
Shahar Mendelson.
\newblock On aggregation for heavy-tailed classes.
\newblock \emph{Probability Theory and Related Fields}, 168\penalty0
  (3-4):\penalty0 641--674, 2017{\natexlab{a}}.

\bibitem[Mendelson(2017{\natexlab{b}})]{mendelson2017learning}
Shahar Mendelson.
\newblock Learning without concentration for general loss functions.
\newblock \emph{Probability Theory and Related Fields}, Jun 2017{\natexlab{b}}.

\bibitem[Miller and Dunson(2018)]{miller2018robust}
Jeffrey~W Miller and David~B Dunson.
\newblock Robust {B}ayesian inference via coarsening.
\newblock \emph{Journal of the American Statistical Association}, pages 1--13,
  2018.

\bibitem[Nemirovskii and Yudin(1983)]{nemirovskii1983problem}
Arkadii Nemirovskii and David~Borisovich Yudin.
\newblock \emph{Problem complexity and method efficiency in optimization}.
\newblock Wiley-Interscience, 1983.

\bibitem[Rissanen(1989)]{Rissanen89}
Jorma Rissanen.
\newblock \emph{Stochastic Complexity in Statistical Inquiry}.
\newblock World Scientific, Hackensack, NJ, 1989.

\bibitem[Rockafellar(1970)]{rockafellar1970convex}
R.~Tyrrell Rockafellar.
\newblock \emph{Convex Analysis}.
\newblock Princeton University Press, Princeton, NJ, 1970.

\bibitem[Sason and Verd{\'u}(2016)]{sason2016f}
Igal Sason and Sergio Verd{\'u}.
\newblock $f$-divergence inequalities.
\newblock \emph{IEEE Transactions on Information Theory}, 62\penalty0
  (11):\penalty0 5973--6006, 2016.

\bibitem[Seeger et~al.(2008)Seeger, Kakade, and Foster]{SeegerKF08}
Matthias~W. Seeger, Sham~M. Kakade, and Dean~P. Foster.
\newblock Information consistency of nonparametric {G}aussian process methods.
\newblock \emph{IEEE Transactions on Information Theory}, 54\penalty0
  (5):\penalty0 2376--2382, 2008.

\bibitem[Tsybakov(2004)]{tsybakov2004optimal}
Alexander~B. Tsybakov.
\newblock Optimal aggregation of classifiers in statistical learning.
\newblock \emph{The Annals of Statistics}, 32\penalty0 (1):\penalty0 135--166,
  2004.

\bibitem[van Erven and Harremo{\"e}s(2014)]{erven2014renyi}
Tim van Erven and Peter Harremo{\"e}s.
\newblock R{\'e}nyi divergence and {K}ullback-{L}eibler divergence.
\newblock \emph{IEEE Transactions on Information Theory}, 60\penalty0
  (7):\penalty0 3797--3820, 2014.

\bibitem[van Erven et~al.(2015)van Erven, Gr{{\"u}}nwald, Mehta, Reid, and
  Williamson]{erven2015fast}
Tim van Erven, Peter~D. Gr{{\"u}}nwald, Nishant~A. Mehta, Mark~D. Reid, and
  Robert~C. Williamson.
\newblock Fast rates in statistical and online learning.
\newblock \emph{Journal of Machine Learning Research}, 16:\penalty0 1793--1861,
  2015.

\bibitem[Vapnik(1995)]{vapnik1995nature}
Vladimir~N. Vapnik.
\newblock \emph{The nature of statistical learning theory}.
\newblock Springer-Verlag New York, Inc., 1995.

\bibitem[Vovk(1990)]{vovk1990aggregating}
Vladimir Vovk.
\newblock Aggregating strategies.
\newblock In \emph{Proceedings of the third annual workshop on Computational
  learning theory}, pages 371--383. Morgan Kaufmann Publishers Inc., 1990.

\bibitem[Walker and Hjort(2002)]{WalkerH02}
Stephen Walker and Nils~Lid Hjort.
\newblock On {B}ayesian consistency.
\newblock \emph{Journal of the Royal Statistical Society: Series B (Statistical
  Methodology)}, 63\penalty0 (4):\penalty0 811--821, 2002.

\bibitem[Wong and Shen(1995)]{wong1995probability}
Wing~Hung Wong and Xiaotong Shen.
\newblock Probability inequalities for likelihood ratios and convergence rates
  of sieve {MLE}s.
\newblock \emph{The Annals of Statistics}, 23\penalty0 (2):\penalty0 339--362,
  1995.

\bibitem[Yamanishi(1998)]{yamanishi1998decision}
Kenji Yamanishi.
\newblock A decision-theoretic extension of stochastic complexity and its
  applications to learning.
\newblock \emph{IEEE Transactions on Information Theory}, 44\penalty0
  (4):\penalty0 1424--1439, 1998.

\bibitem[Yang and Barron(1999)]{yang1999information}
Yuhong Yang and Andrew Barron.
\newblock Information-theoretic determination of minimax rates of convergence.
\newblock \emph{The Annals of Statistics}, 27\penalty0 (5):\penalty0
  1564--1599, 1999.

\bibitem[Yang and Barron(1998)]{yang1998asymptotic}
Yuhong Yang and Andrew~R Barron.
\newblock An asymptotic property of model selection criteria.
\newblock \emph{IEEE Transactions on Information Theory}, 44\penalty0
  (1):\penalty0 95--116, 1998.

\bibitem[Zhang(2006{\natexlab{a}})]{zhang2006epsilon}
Tong Zhang.
\newblock From $\varepsilon$-entropy to {KL}-entropy: Analysis of minimum
  information complexity density estimation.
\newblock \emph{The Annals of Statistics}, 34\penalty0 (5):\penalty0
  2180--2210, 2006{\natexlab{a}}.

\bibitem[Zhang(2006{\natexlab{b}})]{zhang2006information}
Tong Zhang.
\newblock Information-theoretic upper and lower bounds for statistical
  estimation.
\newblock \emph{IEEE Transactions on Information Theory}, 52\penalty0
  (4):\penalty0 1307--1321, 2006{\natexlab{b}}.

\end{thebibliography}

\appendix

\section{Proofs for Section~\ref{sec:zhang}}

\begin{proof}{\bf (of Proposition~\ref{prop:aff-div-results})}
First, we prove (a), i.e., $\lim_{\eta \downarrow 0} -\frac{1}{\eta} \log \E [ e^{-\eta X} ] 
= \lim_{\eta \downarrow 0} \frac{1}{\eta} \left( 1 - \E [ e^{-\eta X} ] \right)
= \E [ X ]$.

Define $y_\eta := \E [ e^{-\eta X} ]$; we will use the fact that $\lim_{\eta \downarrow 0} \E [ e^{-\eta X} ] = 1$ (from Fatou's Lemma, using the nonnegativity of $e^{-\eta x}$).

Now, from Lemma 2 of \Citet{erven2014renyi}, for $y \geq \frac{1}{2}$ we have
$(y - 1) \left( 1 + \frac{1 - y}{2} \right) \leq \log y \leq  y - 1$. 
Hence, 
\begin{align*}
\lim_{\eta \downarrow 0} -\frac{1}{\eta} \log \E [ e^{\eta X} ] 
= \lim_{\eta \downarrow 0} -\frac{1}{\eta} \log y_\eta 
= \lim_{\eta \downarrow 0} -\frac{1}{\eta} (y_\eta - 1) 
= \lim_{\eta \downarrow 0} \frac{1}{\eta} \E [ 1 - e^{-\eta X} ] ,
\end{align*}
which completes the proof of the first equality.

Now, for all $x$ the function $\eta \rightarrow \frac{1}{\eta} (1- e^{-\eta x})$ is non-increasing, as may be verified since 
$\sign(x e^{-\eta x} - \frac{1 - e^{-\eta x}}{\eta}) = -\sign(e^{\eta x} - (\eta x + 1)) \leq 0$. 

Next, we rewrite the following Hellinger-divergence-like quantity:
\begin{align*}
\E \left[ \frac{1}{\alpha \bar{\eta}} \left( 1 - e^{-\alpha \bar{\eta} X} \right) \right] 
= \E \left[ \frac{1}{\alpha \bar{\eta}} \left( 1 - e^{-\alpha \bar{\eta} X} \right)
                 - \frac{1}{\bar{\eta}} (1 - e^{-\bar{\eta} X}) \right] 
    + \frac{1}{\bar{\eta}} \E \left[ 1 - e^{-\bar{\eta} X} \right] .
\end{align*}
Now take any decreasing sequence $\alpha = \alpha_j \in (\alpha_i)_{i
  \geq 1}$ going to zero with $\alpha_1 < 1$. We have for all $j$ that
$x \mapsto \frac{1}{\alpha_j \bar{\eta}} \left( 1 - e^{-\alpha_j
    \bar{\eta} x} \right) - \frac{1}{\bar{\eta}} (1 - e^{-\bar{\eta}
  x})$ is a positive function, and the corresponding sequence with
respect to $j$ is non-decreasing. Hence, the monotone convergence
theorem applies and we may interchange the limit and expectation,
yielding
\begin{align*}
&\lim_{\alpha \downarrow 0} 
    \E \left[ \frac{1}{\alpha \bar{\eta}} \left( 1 - e^{-\alpha \bar{\eta} X} \right)
                  - \frac{1}{\bar{\eta}} (1 - e^{-\bar{\eta} X}) \right] 
    + \frac{1}{\bar{\eta}} \E \left[ 1 - e^{-\bar{\eta} X} \right] \\
&= \E \left[ \lim_{\alpha \downarrow 0} \frac{1}{\alpha \bar{\eta}} \left( 1 - e^{-\alpha \bar{\eta} X} \right)
                   - \frac{1}{\bar{\eta}} (1 - e^{-\bar{\eta} X}) \right] 
      + \frac{1}{\bar{\eta}} \E \left[ 1 - e^{-\bar{\eta} X} \right] \\ &= \E \left[ \lim_{\eta \downarrow 0} \frac{1 - e^{-\eta X}}{\eta} \right] = \E \left[ \frac{\lim_{\eta \downarrow 0} X e^{-\eta X}}{1} \right] = \E [ X ] ,
\end{align*}
where the penultimate equality follows from L'H\^opital's rule. This concludes the proof of the second part of (a).
Next, we show (b). Observe that for any $\eta' \leq \eta$, the concavity of $x \mapsto x^{\eta' / \eta}$ together with Jensen's inequality implies that
\begin{align*}
-\frac{1}{\eta'} \log \E \left[ e^{-\eta' X} \right] 
= -\frac{1}{\eta'} \log \E \left[ \left( e^{-\eta X} \right)^{\eta' / \eta} \right] 
\geq -\frac{1}{\eta'} \log \left( \E \left[ e^{-\eta X} \right] \right)^{\eta' / \eta} 
= -\frac{1}{\eta} \log \E \left[ e^{-\eta X} \right] .
\end{align*}
\end{proof}

\subsection{Proof of Lemma~\ref{lem:rascbound-complex}, extending
  Lemma~\ref{lem:rascbound-simple}}
\label{app:rascboundproof}
We begin with an extension of Lemma~\ref{lem:rascbound-simple}. This
more general result will be used in the proof of Theorem
\ref{thm:main-bounded-b}.  It generalizes
Lemma~\ref{lem:rascbound-simple} in that it allows general comparators
$\phi(f)$, which depend on the $f$ being compared, instead of just the
risk-minimizing $\fopt$
(and it continues to hold even if $\cF$ does
not contain an optimal $\fopt$). 
Formally, let $(\Pt,\loss,{\cF})$ be
a learning problem.  For $f \in \cF$, we work with the excess loss
$\loss_f - \loss_{\phi(f)}$, where
$\phi: \cF \rightarrow \allactionset$ is a \emph{comparator
  map}\footnote{The set $\allactionset$ is defined at the beginning of
  Section~\ref{sec:grip}.} which, in the special case of
Lemma~\ref{lem:rascbound-simple}, is simply the trivial function
mapping each $f \in \cF$ to $f^*$.
\begin{lemma}\label{lem:rascbound-complex} 
Let $\smtuple$ represent a learning problem. 
Let $\dolest$ be a learning algorithm for this learning problem that outputs distributions on $\cF$. 
Let $\phi: \cF \rightarrow \allactionset$ be any deterministic function mapping the predictor $\rv{f} \sim \dol_n$ to a set of nontrivial comparators. 
Then for all $\eta > 0$, we have:  
\begin{align}\label{eq:mainrascbound}
\E_{\rv{f} \sim \dol_n} \left[ \Expann{\eta}_{Z \sim \Pt} \left[ \loss_{\rv{f}} - \loss_{\phi(\rv{f})} \right] \right]
 \stochleq_{\eta \cdot n }  \rsc_{n,\eta} \left( \phi(\rv{f}) \pipes  \dolest \right) .
\end{align}
where $\rsc_{\eta}$ is the (generalized) information complexity, defined as
\begin{align}\label{eqn:rsc-general} 
\rsc_{n,\eta} \left( \phi(\rv{f}) \pipes \dolest \right)
& := 
\Exp_{\rv{f} \sim \dol_n} \left[ \frac{1}{n} \sum_{i=1}^{n} \left( 
      \loss_{\rv{f}}(Z_i) -
      \loss_{\phi(\rv{f})}(Z_i) \right) \right]
  + \frac{\KL( \dol_n \pipes \Prior)}{\eta \cdot n  } .
\end{align}
\end{lemma}
By the finiteness considerations of Appendix~\ref{app:infinity-new}, $\rsc_{n,\eta}(\phi(\rv{f}) \pipes \dolest)$ is always well-defined but may in some cases be equal to $- \infty$ or $\infty$. The explicit use above of a comparator function $\phi$ differs from Zhang's statement, in which the ability to use such a mapping was left quite implicit; however, inspection of the proof of Theorem 2.1 of \cite{zhang2006information} reveals that our version above with comparator functions is also true. Comparator functions will be critical to our application of Lemma \ref{lem:rascbound-complex}. For completeness, we provide a proof of this generalized result.

\begin{Proof}[Proof (of Lemma~\ref{lem:rascbound-complex})] 
For  any measurable function $\psi : \cF \times \cZ^n \rightarrow \reals$ it holds that
\begin{align} \label{eqn:duality}
\E_{f \sim \dol_n} [ \psi(f, Z^n) ] - \KL(\dol_n \pipes \Prior) 
\leq \log \E_{f \sim \Prior} \left[ e^{\psi(f, Z^n)} \right] .
\end{align}
This result, a variation of the ``Donsker-Varadhan variational bound'' follows from convex duality; see \cite{zhang2006information} for an explicit proof.

Define the function $R_n \colon \cF \times \cZ^n \rightarrow \reals$ 
as $R_n(f, z^n) = \sum_{j=1}^n \left( \loss_f(z_j) - \loss_{\phi(f)}(z_j) \right)$. 
Then \eqref{eqn:duality} with the choice $\psi(f, Z^n) = -\eta R_n(f, Z^n) - \log \E_{\bar{Z}^n \sim P^n} \left[ e^{-\eta R_n(f, \bar{Z}^n)} \right]$ yields
\begin{align*}
\E_{f \sim \dol_n} \left[ -\eta R_n(f, Z^n) - \log \E_{\bar{Z}^n} \left[ e^{-\eta R_n(f, \bar{Z}^n)} \right] \right] - \KL(\dol_n \pipes \Prior) 
\leq \log \E_{f \sim \Prior} \left[ 
               \frac
                 {e^{-\eta R_n(f, Z^n)}}
                 {\E_{\bar{Z}^n} \left[ e^{-\eta R_n(f, \bar{Z}^n)} \right]}
               \right] ,
\end{align*}
which, after exponentiating and taking the expectation with respect to $Z^n \sim P^n$, gives
\begin{align*}
&\E_{Z_n} \left[ \exp \left( \E_{f \sim \dol_n} \left[ -\eta R_n(f, Z^n) - \log \E_{\bar{Z}^n} \left[ e^{-\eta R_n(f, \bar{Z}^n)} \right] \right] - \KL(\dol_n \pipes \Prior) \right) \right] \\
&\leq \E_{Z^n} \left[ \E_{f \sim \Prior} \left[ 
               \frac
                 {e^{-\eta R_n(f, Z^n)}}
                 {\E_{\bar{Z}^n} \left[ e^{-\eta R_n(f, \bar{Z}^n)} \right]}
               \right] \right] .
\end{align*}
From the Tonelli-Fubini theorem (see e.g.~\cite[p.~137]{dudley2002real}), we can exchange the two outermost expectations on the RHS, and so the RHS is at most 1. Using ESI notation, we then have
\begin{align*}
\E_{f \sim \dol_n} \left[ -\log \E_{\bar{Z}^n} \left[ e^{-\eta R_n(f, \bar{Z}^n)} \right] \right] 
\stochleq_1 
    \E_{f \sim \dol_n} \left[ \eta R_n(f, Z^n) \right] 
    + \KL(\dol_n \pipes \Prior)  .
\end{align*}
Using that the $\bar{Z}_1, \ldots, \bar{Z}_n$ are drawn i.i.d.~from $P$ and dividing by $\eta \cdot n$ then yields
\begin{align*}
\E_{f \sim \dol_n} \left[ 
    -\frac{1}{\eta} \log \E_Z \left[ e^{-\eta (\loss_f(Z) - \loss_{\phi(f)}(Z))} \right] 
\right] 
\stochleq_{\eta \cdot n} 
    \E_{f \sim \dol_n} \left[ 
        \frac{1}{n} \sum_{j=1}^n \left( \loss_f(Z_j) - \loss_{\phi(f)}(Z_j) \right) 
    \right] 
    + \frac{1}{\eta}\KL(\dol_n \pipes \Prior)  .
\end{align*}
\end{Proof}

\begin{Proof}[Proof (of Proposition~\ref{prop:transavia})]
\cite{zhang2006epsilon} showed the first inequality in
(\ref{eqn:rscmix}) and (\ref{eqn:twopartcodemix}). The equality of the
first and third terms and the inequality in (\ref{eqn:rscmix}) are
``folklore'' in the individual sequence-prediction and MDL communities. For completness we provide a proof.

The two equalities in \eqref{eqn:rscmix} are easy to see after rewriting the center term as
\begin{align*}
n \cdot \inf_{\dolest \in \textsc{RAND}} \rsc_{n,\eta}(\nocomparator \dolest )  
= -\frac{1}{\eta} \sup_{\dol \in \Delta(\cF)} \left\{ 
        -\sum_{j=1}^n \xslossat{f}{Z_j} - \KL(\dol \pipes \Prior) 
    \right\} .
\end{align*}
Now, from Legendre duality, we have for some map $\varphi: \cX \rightarrow \reals$ that
\begin{align*}
\sup_{\nu \in \Delta(\cX)} \bigl\{ \E_{X \sim \nu} [ \varphi(X) ] - \KL(\nu \pipes \mu) \bigr\} 
= \log \E_{X \sim \mu} \left[ e^{\varphi(X)} \right] ,
\end{align*}
and the supremum is achieved by taking $\nu(dx) = \frac{e^\varphi(dx)}{\E_{X \sim \mu} \left[ e^{\varphi(X)} \right]}$. This proves the equalities in \eqref{eqn:rscmix}. 

To see \eqref{eqn:rscmixpre} and \eqref{eqn:rscmixb}, observe that for any $A \subset \cF$, we have
\begin{align*}
-\frac{1}{\eta} \log \E_{\rv{f} \sim \Prior} \left[ e^{-\sum_{j=1}^n \xslossat{\rv{f}}{Z_j}} \right] 
&= -\frac{1}{\eta} \log \E_{\rv{f} \sim \Prior} \left[ 
          \left( \ind{\rv{f} \in A} + \ind{\rv{f} \notin A} \right)
          e^{-\sum_{j=1}^n \xslossat{\rv{f}}{Z_j}} 
      \right] \\ \leq
 -\frac{1}{\eta} \log \E_{\rv{f} \sim \Prior} \left[ 
              \ind{\rv{f} \in A} 
              \cdot e^{-\sum_{j=1}^n \xslossat{\rv{f}}{Z_j}} 
          \right] 
&= -\frac{1}{\eta} \log \Prior(A) 
      - \frac{1}{\eta} \log \E_{\rv{f} \sim \Prior \mid A} \left[ 
             e^{-\sum_{j=1}^n \xslossat{\rv{f}}{Z_j}} 
          \right] \\
&\leq -\frac{1}{\eta} \log \Prior(A) 
          + \E_{\rv{f} \sim \Prior \mid A} \left[ \sum_{j=1}^n \xslossat{\rv{f}}{Z_j} \right] ,
\end{align*}
where the last line follows from Jensen's inequality. Together with the second equality in the already-established (\ref{eqn:rscmix}), the third line implies (\ref{eqn:rscmixpre}); the last line implies (\ref{eqn:rscmixb}).

For \eqref{eqn:twopartcodemix}, the first inequality is obvious since the infimum over \textsc{DET} is at least the infimum over \textsc{RAND}. The equality is immediate from the definition of the two-part MDL estimator. The second inequality follows as a special case of the inequality in \eqref{eqn:rscmix}. 
\end{Proof}

\section{Proofs for Section~\ref{sec:strongcentral}}
\label{sec:strongcentralproofs}
\begin{Proof}[Proof (of Theorem~\ref{thm:metric})] 
\glsadd{renyi} The R\'enyi divergence \Citep{erven2014renyi} of order $\alpha$ is defined as
$D_{\alpha}(p \| q) = \frac{1}{\alpha-1} \log  \int p^{\alpha} q^{1- \alpha} d \mu$,
so that, for $0 < \alpha < 1$,  with $\eta = (1- \alpha) \bar\eta$,
\begin{align*}
D_{\alpha}(p_{\fopt,\bar\eta} \| p_{f,\bar\eta}) 
&= \frac{1}{\alpha-1} \log  
\int  p(z) \frac{e^{- {\alpha \bar\eta}
  \xsloss{\fopt}} \cdot e^{
- {(1-\alpha) \bar\eta}
  \xsloss{f}}}{
(\E[e^{-\bar\eta  \xsloss{\fopt}(Z)}])^{\alpha}  
(\E[e^{-\bar\eta \xsloss{f}(Z)}])^{1-\alpha}} d\mu \\
&= \frac{1}{\alpha-1} \log  
\int  p(z) \frac{ e^{
- {(1-\alpha) \bar\eta}
  \xsloss{f}}}{
(\E[e^{-\bar\eta \xsloss{f}(Z)}])^{1-\alpha}} d\mu 
= - \frac{\bar\eta}{\eta} \left(\log \E[ e^{- \eta \xsloss{f}} ] - \frac{\eta}{\bar{\eta}} 
\log  \E[e^{-\bar\eta \xsloss{f}(Z)}]\right) \\ &= \bar\eta \Expann{\eta}[\xsloss{f}] + \log  \E[e^{-\bar\eta \xsloss{f}(Z)}] 
\leq \bar\eta  \Expann{\eta}[\xsloss{f}] ,
\end{align*}
where we used the $\bar\eta$-central condition. 
\Citet{erven2014renyi} show that the squared Hellinger distance between two densities $p$ and $q$ is always bounded by their R\'enyi divergence of order $1/2$ and also that the latter is bounded by the R\'enyi divergence of order $0 < \alpha< 1/2$ via 
$D_{1/2}(p \|q) \leq \frac{1-\alpha}{\alpha} D_{\alpha}(p \|q)$,
so that we get
\begin{align*}
\dhel^2_{\bar\eta}(f,f')
\leq \frac{1}{\bar{\eta}} \cdot \frac{1-\alpha}{\alpha} \cdot \bar\eta  \Expann{\eta}[\xsloss{f}] 
= \frac{\eta}{\bar{\eta} - \eta}  \Expann{\eta}[\xsloss{f}] .
\end{align*}
The result is now immediate from Lemma~\ref{lem:rascbound-simple}.  
\end{Proof}

\begin{proof}{\bf (of Proposition~\ref{prop:entrobound}, cont.)}
  We use the familiar rewrite of the KL divergence
  $\E_{Z \sim P_{f^*}} [L_f] = D(f^* \| f)$ as
  $\E_{Z \sim P_{f^*}} [L_f] = \E[L_f + S]$, with
  $S= (p_f(Z)/p_{f^*}(Z))-1$, where as is well-known, 
  $L_f +S$ is nonnegative on $\cZ$.  Using this in the second
  inequality below gives:
\begin{align*}
\E_{Z \sim P_{f^*}} [L_f \opmax 0] & =
\E_{Z \sim P_{f^*}} [\ind{L_{f} \geq 0} (L_f   + S)] - \E_{Z \sim P_{f^*}}[\ind{L_f \geq  0} S]
\leq
\E_{Z \sim P_{f^*}} [L_f] +  \E_{Z \sim P_{f^*}}[ |S|] \\
& = \E_{Z \sim P_{f^*}} [L_f] + \int p_{f^*} \left|\frac{p_{f} - p_{f^*}}{p_{f^*}} \right|  d \mu(z) 
\leq D(f^* \| f) + \int \left| p_{f} - p_{f^*} \right| d \mu,
\end{align*}
and the result follows by Pinsker's inequality. 
\end{proof}

\section{Proofs for Section~\ref{sec:witness} and Example~\ref{ex:smallballagain}}
\label{app:proofs-witness}

\subsection{Proof of Lemma~\ref{lem:renyi-to-kl}, extending Lemma~\ref{lem:renyi-to-kl-simple}}
Below we state and prove Lemma \ref{lem:renyi-to-kl} which generalizes
Lemma~\ref{lem:renyi-to-kl-simple} in the main text in that it allows general comparators $\phi(f)$, as introduced above Lemma~\ref{lem:rascbound-complex}. This extension is pivotal for our results in
Section~\ref{sec:grip} involving the GRIP.
\begin{lemma} \label{lem:renyi-to-kl} Let $\bar{\eta} > 0$. Let $\phi$ be 
  any comparator map $\phi$ such that 
for any given $f$, $\phi(f)$ satisfies $\E [ \loss_{\phi(f)} ] \leq \E [ \loss_f ]$. Assume that the strong
  $\bar{\eta}$-central condition is satisfied with respect to 
  comparator $\phi$ for some fixed $f \in \cF$ , i.e.,
\begin{align}\label{eq:baby}
\loss_f - \loss_{\phi(f)} \stochleq_{\bar{\eta}} 0.
\end{align}
Furthermore assume that the $(u,c)$-witness condition holds for this $f$,
relative to $\phi(f)$, for some constants $u > 0$ and $c \in (0, 1]$,
i.e.,
\begin{align}\label{eq:witnessagain}
c \xsrisk{f} \leq \E [ (\loss_f - \loss_{\phi(f)}) \cdot \ind{\loss_f - \loss_{\phi(f)} \leq u} ]. 
\end{align}
Then for all $\eta \in (0, \bar{\eta})$
\begin{align}\label{eq:outwitnessed}
\xsrisk{f} 
\,\,\leq\,\, c_u \cdot \Exphel{\eta} \left[ \loss_f - \loss_{\phi(f)} \right]
\,\,\leq\,\, c_u \cdot \Expann{\eta} \left[ \loss_f - \loss_{\phi(f)} \right] ,
\end{align}
with $c_u := \frac{1}{c} \frac{\eta u + 1}{1 - \frac{\eta}{\bar{\eta}}}$. 
Moreover, suppose that the $(\tau,c)$-witness condition holds for a
non-increasing $\tau$ and $c$ as in Definition~\ref{def:witness}, for
all $f \in \cF$, relative to comparator $\phi(\cdot)$, i.e., $\E
\bigl[ (\loss_f - \loss_{\phi(f)}) \cdot \ind{\loss_f -
  \loss_{\phi(f)} \leq \tau(\E [ \loss_f - \loss_{\phi(f)}])} \bigr]
\geq c \xsrisk{f}$.  For all $f \in \cF$, all $\eta
\in (0, \bar{\eta})$, all $\epsilon > 0$, we have:
\begin{equation}\label{eq:rkltaub}
\xsrisk{f} 
\,\,\leq\,\, \epsilon \opmax c_{\tau(\epsilon)}  \cdot \Exphel{\eta} \left[ \loss_f - \loss_{\phi(f)} \right]
\,\,\leq\,\, \epsilon \opmax c_{\tau(\epsilon)} \cdot \Expann{\eta} \left[ \loss_f - \loss_{\phi(f)} \right].
\end{equation}
\end{lemma}

\begin{Proof}\ \\
  \emph{Proof of \eqref{eq:outwitnessed}}. Define
  $L'_{f} := \loss_f - \loss_{\phi(f)}$. For any $\eta \in [0,
  \bar{\eta}]$, define:
\begin{align*}
h_{f,\eta} := \frac{1}{\eta} \left( 1 - e^{-\eta L'_f} \right) 
\qquad 
S_{f,\eta} := h_{f,\eta} - h_{f,\bar{\eta}} 
\qquad 
\hellp{\eta} := \Exphel{\eta}[L'_f]  = \E [ h_{f,\eta} ] .
\end{align*}
It is easy to verify that the map $\eta \mapsto h_{f,\eta}$ is non-increasing, and hence $S_{f,\eta}$ is a positive random variable for any $\eta \in [0, \bar{\eta}]$. It also is easy to verify that $\lim_{\eta \downarrow 0} h_{f,\eta} = L'_f$. We thus can define 
$h_{f,0} = L'_f$ and 
$S_{f,0} = L'_f - h_{f,\bar\eta}$ 
and hence can rewrite the excess risk of $f$ (with respect to $\phi(f)$) as
\begin{align*}
\E [ L'_f ] 
= \E [ h_{f,0} - h_{f,\bar{\eta}} + h_{f,\bar{\eta}} ] = \E [ S_{f,0} ] + \hellp{\bar{\eta}} .
\end{align*}
Splitting up the expectation into two components, we have
\begin{align*}
\E [ S_{f,0} \cdot \ind{L'_f \leq u} ] 
+ \E [ S_{f,0} \cdot \ind{L'_f > u} ] 
+ \hellp{\bar{\eta}} .
\end{align*}
Now, from Lemma \ref{lemma:bounded-part} (stated and proved immediately after this proof), the positivity of $S_{f,\eta}$, and using $\bar{C} := C_{\bar{\eta},\eta,u}$ to avoid cluttering notation, we have
\begin{align*}
\E [ L'_f ] 
&\leq \bar{C} \E [ S_{f,\eta} \cdot \ind{L'_f \leq u} ] 
          + \E [ S_{f,0} \cdot \ind{L'_f > u} ] 
          + \hellp{\bar{\eta}} \leq \bar{C} \E [ S_{f,\eta} ] 
          + \E [ S_{f,0} \cdot \ind{L'_f > u} ] 
          + \hellp{\bar{\eta}} \\
&= \bar{C} \left( \hellp{\eta}  - \hellp{\bar{\eta}} \right)
      + \E [ S_{f,0} \cdot \ind{L'_f > u} ] 
      + \hellp{\bar{\eta}}
= \bar{C} \hellp{\eta}  - (\bar{C} - 1) \hellp{\bar{\eta}} 
      + \E [ S_{f,0} \cdot \ind{L'_f  > u} ] .
\end{align*}
We observe that $\hellp{\bar{\eta}} \geq 0$ since $\hellp{\bar{\eta}} = \frac{1}{\bar{\eta}} \E \left[ 1 - e^{-\bar{\eta} L'_f} \right] \geq 0$, where the inequality is implied by the strong $\bar{\eta}$-central condition (i.e.~$\E \left[ e^{-\bar{\eta} L'_f} \right] \leq 1$). 
Therefore, since it always holds that $\bar{C} \geq 1$ we have
\begin{align}\label{eq:h0}
\E [ L'_f ] 
&\leq \bar{C} \hellp{\eta}  
          + \E [ S_{f,0} \cdot \ind{L'_f > u} ] .
\end{align}
Next, we claim that $\E [ S_{f,0} \cdot \ind{L'_f> u} ] \leq \E [ L'_f \cdot \ind{L'_f > u} ]$. To see this, observe that $S_{f,0} = L'_f + \frac{1}{\bar{\eta}} \left( e^{-\bar{\eta} L'_f} - 1 \right)$, and that the second term is negative on the event $L'_f > u$. We thus have
\begin{align*}
\E [ L'_f ] - \E [ L'_f \cdot \ind{L'_f > u} ] 
&\leq \bar{C} \hellp{\eta} ,
\end{align*}
which can be rewritten as
\begin{align}\label{eq:commonality}
\E [ L'_f \cdot \ind{L'_f \leq u} ] 
&\leq \bar{C} \hellp{\eta} ,
\end{align}
Now, since we assume (\ref{eq:witnessagain}), 
the first inequality in \eqref{eq:outwitnessed} is proved, and the second then follows from (\ref{eq:genrenc}):
\begin{align*}
\xsrisk{f} 
\leq \frac{\bar{C}}{c} \hellp{\eta} .
\end{align*}
{\em Proof of \eqref{eq:rkltaub}}. 
Fix arbitrary $f \in \cF$. 
We know that for this particular $f$, either $\xsrisk{f} \leq \epsilon$ in which case there is nothing to prove, or $\xsrisk{f} > \epsilon$. Then for this $f$, 
the $(u,c)$-witness condition holds with $u = \tau(\xsrisk{f}) \leq \tau(\epsilon)$. But then the result follows as above. 
\end{Proof}

\begin{lemma}[``Bounded Part'' Lemma] \label{lemma:bounded-part}
For $u, \bar{\eta} > 0$ and $\eta \in [0, \bar{\eta})$, we have
\begin{align*}
\E [ S_{f,0} \cdot \ind{\loss_f - \loss_{\phi(f)} \leq u} ] 
\leq C_{\bar{\eta},\eta,u} \E [ S_{f,\eta} \cdot \ind{\loss_f - \loss_{\phi(f)} \leq u} ] ,
\end{align*}
where $C_{\bar{\eta},\eta,u} := \frac{\eta u + 1}{1 - \frac{\eta}{\bar{\eta}}}$.
\end{lemma}

\begin{Proof}
It is sufficient to show that on the set $\{\loss_f - \loss_{\phi(f)} \leq u\}$, it holds that $S_{f,0} \leq C S_{f,\eta}$ for some constant $C$. This may be rewritten as wanting to show, for $\eta_0 \rightarrow 0$:
\begin{align*}
\frac{1}{\eta_0} (1 - e^{-\eta_0 (\loss_f - \loss_{\phi(f)})}) - \frac{1}{\bar{\eta}} (1 - e^{-\bar{\eta} (\loss_f - \loss_{\phi(f)})})
\leq C \left( \frac{1}{\eta} (1 - e^{-\eta (\loss_f - \loss_{\phi(f)})}) - \frac{1}{\bar{\eta}} (1 - e^{-\bar{\eta} (\loss_f - \loss_{\phi(f)})}) \right) .
\end{align*}
Letting $r = e^{-\bar{\eta} (\loss_f - \loss_{\phi(f)})}$, this is equivalent to showing that
\begin{align*}
\frac{1}{\bar{\eta}} \left( \frac{1}{\eta_0 / \bar{\eta}} (1 - r^{\eta_0/\bar{\eta}}) - (1 - r) \right)
\leq \frac{C}{\bar{\eta}} \left( \frac{1}{\eta / \bar{\eta}} (1 - r^{\eta/\bar{\eta}}) - (1 -r) \right) .
\end{align*}
Now, for any $\eta \geq 0$, define\footnote{Note that the $g_\eta$ used here is \emph{not} a GRIP.} the function $g_\eta$ as $g_\eta(r) = \frac{1}{\eta} (1 - r^\eta) - (1 - r)$. From Lemma \ref{lemma:ratio}, for any $\eta' \geq 0$, if $r \geq \frac{1}{V}$ for some $V > 1$ then
$g_0(r) \leq \frac{1}{1 - \eta'} (\eta' \log V + 1) g_{\eta'}(r)$.

Applying this inequality, taking $\eta_0 \rightarrow 0$ and $\eta' := \frac{\eta}{\bar{\eta}}$, and observing that on the set $\{\loss_f - \loss_{\phi(f)} \leq u\}$ we may take $V = e^{\bar{\eta} u} > 1$, we see that whenever $\loss_f - \loss_{\phi(f)} \leq u$,
\begin{align*}
\left( \frac{1}{\eta_0} (1 - r^{\eta_0}) - (1 - r) \right)
\leq \frac{1}{1 - \eta'} (\eta' \bar{\eta} u + 1) \left( \frac{1}{\eta'} (1 - r^{\eta'}) - (1 -r) \right) .
\end{align*}
Thus, $S_{f,0} \leq C_{\bar{\eta},\eta,u} S_{f,\eta}$ indeed holds for $C_{\bar{\eta},\eta,u} = \frac{\eta u + 1}{1 - \frac{\eta}{\bar{\eta}}}$.
\end{Proof}

\begin{figure}[ht]
\hspace*{0.25\textwidth}\includegraphics[width=0.5\textwidth]{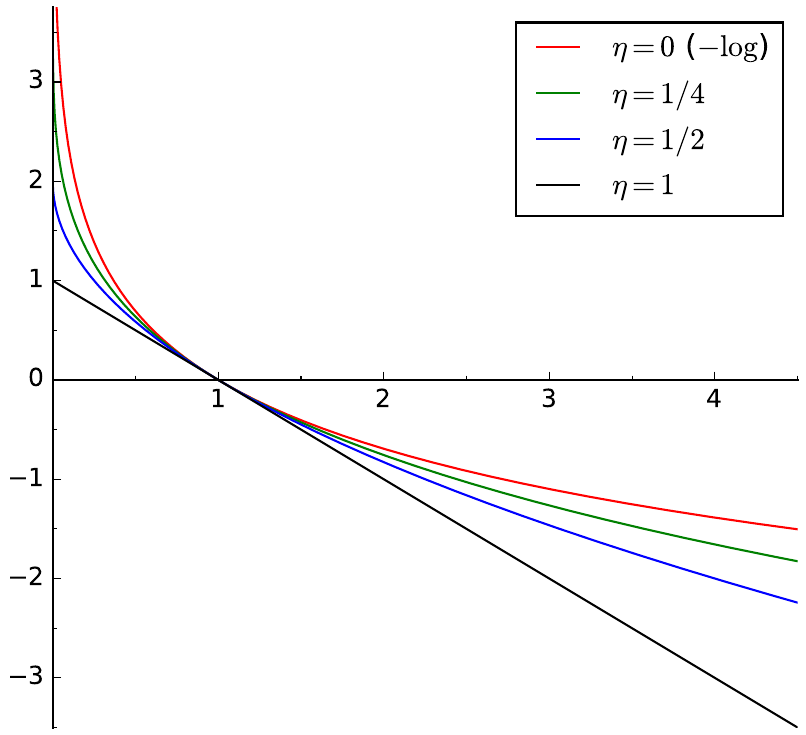}
\caption{\label{fig:moeilijk} The function $r: \rightarrow \eta^{-1}(1-r^{\eta})$ for various values of $r$. $g_{\eta}(r)$ is the difference of the line for $\eta$ at $r$ and the line for $\eta = 1$ at $r$, which is always positive.}
\end{figure}

\begin{lemma}\label{lemma:ratio}
Let $0 \leq  \eta' < \eta < 1$ and $1 < V < \infty$. 
Define $g_{\eta}(r) := \eta^{-1} \left(1-r^\eta \right) - \left(1- r\right)$, a positive function. 
Then for $\eta' > 0$ and $r \geq \frac{1}{V}$:
\begin{align*}
g_{\eta'}(r)
\leq C_{\eta' \leftarrow \eta}(V) g_{\eta}(r) ,
\end{align*}
where 
$C_{\eta' \leftarrow \eta}(V) \leq  ((\eta')^{-1} - 1)/(\eta^{-1} -1)$, 
and
\begin{align*}
\lim_{\eta' \downarrow 0} g_{\eta'}(r)
\leq C_{0 \leftarrow \eta}(V) g_{\eta}(r) ,
\end{align*}
where 
$C_{0 \leftarrow \eta}(V)= \frac{ \log V -(1-V^{-1})}{\frac{1}{\eta}(1- V^{-\eta})
  -(1- V^{-1})} \leq \frac{\eta}{1- \eta} \log V + \frac{1}{1-\eta}$.
\end{lemma} 

\begin{Proof}
Let $0 \leq \eta' < \eta$. 
We will prove that, for all $r \geq \frac{1}{V}$, we have $g_{\eta'}(r) \leq C \cdot g_{\eta}(r)$ for some constant $C$.  Hence it suffices to bound
$$h_{\eta',\eta}(r) := \frac{g_{\eta'}(r)}{g_{\eta}(r)} =
\frac{(\eta')^{-1}(1- r^{\eta'})- (1-r)}{ \eta^{-1}(1- r^{\eta})- (1-
  r)}.$$ 
We can extend the definition of this function to $\eta'=0$ and $r=1$ so that it becomes well-defined for all $r > 0$, $0 \leq \eta' < \eta < 1$: 
$(0)^{-1} (1- r^0)$ is defined as
$\lim_{\eta'\downarrow 0} (\eta')^{-1}(1- r^{\eta'}) = - \log r$.
$h_{\eta',\eta}(1)$ is set to $\lim_{r \uparrow 1} h_{\eta',\eta}(r) = \lim_{r \downarrow 1} h_{\eta',\eta}(r)$ which is calculated using  L'H\^opital's
rule twice, together with the fact that for $0 \leq \eta \leq 1$ (note
$\eta = 0$ is allowed), $g_{\eta}'(r) = -r^{\eta - 1}+1, g_{\eta}''(r)
= (1-\eta) r^{\eta-2}$. Then, because $g_{\eta}(1) = g_{0}(1) =
g'_{\eta}(1)= g'_{0}(1)= 0$, we get:
$$h_{\eta',\eta}(1) := \lim_{r \downarrow 1} g_{\eta'}(r)/g_{\eta}(r) = \lim_{r
  \downarrow 1} g'_{\eta'}(r)/g'_{\eta}(r) = \lim_{r \downarrow 1}
g''_{\eta'}(r)/g''_{\eta}(r) = \frac{1- \eta'}{1-\eta}.$$ We have
$\lim_{r \rightarrow \infty} h_{\eta',\eta}(r) = 1$, and we show below
that $h_{\eta',\eta}(r)$ is strictly decreasing in $r$ for each $0
\leq \eta'< \eta< 1$, so the maximum value is achieved for the minimum
$r = 1/V$.  We have $h_{\eta',\eta}(1/V) \leq h_{\eta', \eta}(0)=
(\eta'^{-1} -1)/(\eta^{-1} -1)$ and $h_{0,\eta}(1/V) = (\log V -
(1-V^{-1}))/(\eta^{-1}(1- V^{-\eta}) - (1- V^{-1}))$.  The result
follows by defining $C_{\eta'\leftarrow \eta}(V) = h_{\eta',
  \eta}(1/V)$.  It only remains to show that $h_{\eta',\eta}(r)$ is
decreasing in $r$ and that the upper bound on $C_{0 \leftarrow
  \eta}(V)$ stated in the lemma holds.

\emph{Proof that $h$ is decreasing}:
The derivative of $h \equiv h_{\eta',\eta}$ for fixed $0 \leq \eta' <
\eta < 1$ is given by $ h'_{\eta',\eta}(r) = r^{-1} \cdot s(r), $
where
\begin{equation}\label{eq:woutertim}
s(r) = \frac{(-r^{\eta'} +r) \cdot g_{\eta}(r) + (r^{\eta} - r) \cdot g_{\eta'}(r)}{g_{\eta}(r)^2}.
\end{equation}
Although we tried hard, we found neither a direct argument that $h'\leq
0$ or that $h'' > 0$ (which would also imply the result in a
straightforward manner). We resolve the issue
by relating $h$ to a function $f$ which is easier to analyze. \eqref{eq:woutertim} shows that
for $r > 0, r \neq 1$, $h'(r) = 0$, i.e., $h$ reaches an
extremum, iff $s(r) = 0$, i.e., iff the numerator in
(\ref{eq:woutertim}) is $0$, i.e., iff $
\frac{g_{\eta'}(r)}{g_{\eta}(r)} = \frac{r^{\eta'}-r}{r^{\eta}- r}, $
i.e., iff
$$
h(r) = f(r),\ \ \  \text{where $f(r) := \frac{r^{\eta'-1}-1}{r^{\eta-1}- 1}$}.  
$$
We can extend  $f$ to its discontinuity point $r=1$ by using
L'H\^opital's rule similar to its use above, and then we find that
$f(1) = h(1)$; similarly, we find that the discontinuities of $f'(r)$
and $h'(r)$ at $r=1$ are also removable, again by aggressively using
L'H\^opital, which gives
\begin{equation}\label{eq:schmid}
f'(1) = \frac{1}{2} \cdot \frac{1- \eta'}{1- \eta} \left(\eta'- \eta\right)\ ,\ h'(1) = \frac{1}{3} \cdot \frac{1- \eta'}{1- \eta} \left(\eta'- \eta\right),
\end{equation}
and we note that both derivatives are $< 0$ and also that there is $L
< 1, R > 1$ such that
\begin{equation}\label{eq:johannes}
\text{$h < f$ on $(L,1)$ \ \ ; \ \ 
$h > f$ on $(1,R)$}.
\end{equation} 
Below we show that $f$ is strictly decreasing on $(0, \infty)$. But then $h$ cannot
have an extremum on $(0,1)$; for if it had,
there would be a point $0 < r_0 < 1$ with $h'(r_0) = 0$ and therefore  $h(r_0) =
f(r_0)$, so that, since $f'(r_0) < 0$, $h$ lies under $f$ in an open
interval to the left of $r_0$ and above $f$ to the right of $r_0$. But
by (\ref{eq:johannes}), this means that there is another point $r_1$
with $r_0 < r_1 < 1$ at which $h$ and $f$ intersect such that $h$ lies
{\em above\/} $f$ directly to the left of $r_1$. But we already showed
that at any intersection, in particular at $r_1$, $h'(r_1) = 0$. Since
$f'(r_1) < 0$, this implies that $h$ must lie {\em below\/} $f$
directly to the left of $r_1$, and we have reached a contradiction.
It follows that $h$ has no extrema on $(0,1)$; entirely analogously,
one shows that $h$ cannot have any extrema on $(1,\infty)$. By
(\ref{eq:schmid}), $h'(r)$ is negative in an open interval containing $1$,
so it follows that $h$ is decreasing on $(0,\infty)$.

It thus only remains to be shown that $f$ is strictly decreasing on
$(0,\infty)$. To this end we consider a monotonic variable
transformation, setting $y = r^{\eta-1}$ so that $r^{\eta'-1} =
y^{(1-\eta')/(1-\eta)}$ and, for $a > 1$, define $f_a(y) =
(y^a-1)/(y-1)$. Note that with $a = (1- \eta')/(1-\eta)$,
$f_a(r^{\eta-1}) = f(r)$. Since $0 < \eta < 1$, $y$ is strictly decreasing in
$r$, so it is sufficient to prove that, for all $a$ corresponding to
some choice of $ 0 \leq \eta'< \eta < 1$, i.e., for all $a > 1$,
$f_a$ is strictly increasing on $y > 0$. 
Differentiation with respect to $y$ gives that $f_a$ is strictly
increasing on interval $(a,b)$ if, for all $y \in (a,b)$,
$$
u_a(y) \equiv a y^a - y^a +1 - a y^{a-1} > 0.
$$
Straightforward differentiation and simplification gives that $u'_a(y)
= a y^{a-1} (a-1) (1- y^{-1})$ which is strictly negative for all
$y< 1$ and strictly positive for $y > 1$. Since trivially, $u_a(1) =
0$, it follows that $u_a(y) > 0$ on $(0,1)$ and $u_a(y) > 0$ on
$(1,\infty)$, so that $f_a$ is strictly increasing on $(0,1)$ and on
$(1,\infty)$. But then $f_a$ must also be strictly increasing at
$r=1$, so $f_a$ is strictly increasing on $(0,\infty)$, which is what
we had to prove.

\emph{Proof of upper bound on $C_{0 \leftarrow \eta}(V)$}: 
The right term in $s(r)$ as given by (\ref{eq:woutertim}) is positive for $r < 1$, and $g_{\eta'}(x) >
g_{\eta}(x)$, so setting $t(r)$ to $s(r)$, but with
$g_{\eta'}(r)$ in the right term in the numerator replaced by
$g_{\eta}(r)$, i.e.,  
$$t(r) : = \frac{(-r^{\eta'} +r) \cdot g_{\eta}(r) + (r^{\eta} - r) \cdot g_{\eta}(r)}{g_{\eta}(r)^2} = \frac{-r^{\eta'} + r^{\eta}}{g_{\eta}(r)} ,
$$ 
we have $t(r) \leq s(r)$ for all $r \leq 1$. We already know that $h_{\eta', \eta}$ is decreasing, so  that
$s(r) \leq 0$ for all $r$, so we have $t(r) \leq s(r) \leq 0$ for all
$r \leq 1$.  In particular, this holds for the case $\eta'= 0$, for
which $t(r)$ simplifies to $t(r) = (-1 + r^{\eta})/g_{\eta}(r) = -(1 -
r^{\eta})/(\eta^{-1} (1 - r^{\eta}) -(1-r))$. A simple calculation
shows that (a) $\lim_{r \downarrow 0} t(r) = -1 /(\eta^{-1} - 1) =
-\eta/(1-\eta)$ and (b) $t(r)$ is increasing on $0 < r < 1$ for all $0
< \eta < 1$.

Now define $\tilde{h}$ by setting $\tilde{h}(r) = (1/(1-\eta))\cdot
(1-\eta \log r)$ for $0 < r \leq 1$. Then $\tilde{h}'(r) = -
(\eta/(1-\eta)) r^{-1} \leq t(r) r^{-1} \leq s(r) r^{-1} =
h'_{0,\eta}(r)\leq 0$ by all the above together. Since $\tilde{h}(1) =
h_{0,\eta}(1)$, and for $r < 1$, $h_{0,\eta}$ is decreasing but
$\tilde{h}$ is decreasing even faster, we must have
$\tilde{h}(r) \geq h_{0,\eta}(r)$ for $0 < r < 1$. We can thus 
bound $h_{0,\eta}(1/V)$ by $\tilde{h}(1/V)$, and the result follows.
\end{Proof}

\subsection{Proof of Lemma~\ref{lem:kl-hell-exp-tails}}
\begin{Proof}
  Markov's inequality implies that for all $f \in \cF$, $\Pr(e^{\delta
    \xsloss{f}} > u) < \frac{M_\delta}{u}$ for any $u \geq
  0$. Therefore, for some map $\tau: \reals_+ \to \reals_+$ to be set
  later:
\begin{align}
\E \left[ \xsloss{f} \cdot \ind{\xsloss{f} > \tau(\xsrisk{f})} \right] 
&= \int_0^\infty \Pr(\xsloss{f} \cdot \ind{\xsloss{f} > \tau(\xsrisk{f})} > t) dt= \nonumber \\
\int_{\tau(\xsrisk{f})}^\infty \Pr(\xsloss{f} > t) dt 
&= \int_{\tau(\xsrisk{f})}^\infty \Pr(e^{\delta \xsloss{f}} > e^{\delta t}) dt \leq 
\int_{\tau(\xsrisk{f})}^\infty M_\delta e^{-\delta t} dt 
&= \frac{M_\delta}{\delta} e^{-\delta \tau(\xsrisk{f})} . \label{eqn:witness-bound-final}
\end{align}
Taking $\tau: x \mapsto 1 \opmax \frac{\log \frac{2 M_\delta}{\delta x}}{\delta}$, the last line above is bounded by $\frac{1}{2} \xsrisk{f}$, and so the $(\tau,c)$-witness condition holds with $c=  1/2$. 
\end{Proof}

\subsection{Proofs related to heavy-tailed regression}
\label{app:heavy-tailed-regression-example}

We start with some general facts. For squared loss, the excess
loss can be written as (abbreviating $f(X)$ and $\fopt(X)$ to $f$ and
$\fopt$, resp.),
\begin{align}\label{eq:bernsteinhandy}
\xsloss{f} & = (f(X) - \fopt(X)) \cdot (-2 Y + f(X) + \fopt(X)) \\
& = (f- \fopt) \cdot (\/(f-\fopt) + 2 (\fopt - Y)\/) \label{eq:bernsteinhandyb} \\
& \label{eq:smallballhandy}
= (f-\fopt)^2 + 2 (\fopt-Y)(f-\fopt) .
\end{align}
Now, recall that in both Examples~\ref{ex:heavy-tailed-regression} and~\ref{ex:smallballagain}, we assumed that the risk minimizer $\fopt$ over $\cF$ continues to be a minimizer when taking the minimum risk over the convex hull of $\cF$. This implies that for all $f \in \cF$,
\begin{align} \label{eqn:convexity-inequality}
\E \left (\fopt(X) - Y) (f(X) - \fopt(X)) \right] \geq 0 ,
\end{align}
To see this, we observe that if we instead consider the function class $\convhull(\cF)$, then $\fopt$ is still a minimizer and \eqref{eqn:convexity-inequality} holds for all $f \in \convhull(\cF)$ from \cite{mendelson2017aggregation} (see the text around equation (1.3) therein).  

But now \eqref{eqn:convexity-inequality} with (\ref{eq:smallballhandy}) implies that, under our assumptions,
\begin{align} \label{eq:l2excess}
\E \left[ (f(X) - \fopt(X))^2 \right] \leq \xsrisk{f}.
\end{align}

\begin{proof}{\bf (of Proposition~\ref{prop:bernstein-witness})}
Let $u > 0$ be a to-be-determined constant. Then
\begin{align*}
& \E \left[ \xsloss{f} \cdot \ind{\xsloss{f} >   \tau(\E[\xsloss{f}])} \right] 
\leq \E \left[ \xsloss{f} \cdot \frac{\xsloss{f}}{\tau(\E[\xsloss{f}])} \cdot \ind{\xsloss{f} \geq 0} \right] 
= \\
& \frac{1}{\tau(\E[\xsloss{f}])} \E \left[ \xsloss{f}^2 \cdot \ind{\xsloss{f} \geq 0} \right] 
\leq \frac{1}{\tau(\E[\xsloss{f}])} \E \left[ \xsloss{f}^2 \right] 
\leq \frac{B}{u} \frac{\left(\E \left[ \xsloss{f} \right]\right)^{\beta}}{
(\E[\xsloss{f}])^{\beta-1}} = \frac{B}{u} \E[\xsloss{f}],
\end{align*}
and the result follows.
\end{proof}

\begin{Proof}{\bf (of Proposition~\ref{prop:awkwardwithout})}
To see that a Bernstein condition holds if $\E [ Y^2 \mid X ] \leq C$ a.s.~and $|f(X)| \leq r$ almost surely, observe that from \eqref{eq:bernsteinhandy}, 
\begin{align*}
\xsloss{f}^2 
&\leq 2 (f(X) - \fopt(X))^2 \left( 4 Y^2 + (f(X) - \fopt(X))^2 \right) ,
\end{align*}
and hence
\begin{align*}
\E \left[ \xsloss{f}^2 \right] 
&\leq 8 \left( 
                 \E \left[ (f(X) - \fopt(X))^2 \E [ Y^2 \mid X ] \right]
                 + r^2 \E \left[ (f(X) - \fopt(X))^2 \right] 
          \right) \\
&\leq 8 (C + r^2) \E \left[ (f(X) - \fopt(X))^2 \right] ,
\end{align*}

Invoking (\ref{eq:l2excess}), we see that a Bernstein condition does indeed hold:
 \begin{align*}
\E \left[ \xsloss{f}^2 \right] 
\leq 8 (C + r^2) \xsrisk{f} .
\end{align*}
\end{Proof}

\begin{Proof}[Proof (of claim in Example \ref{ex:smallballagain})]
From (\ref{eq:bernsteinhandyb}), Cauchy-Schwarz, and our assumption, 
\begin{align}\label{eq:bernsteinisthere}
\E [ \xsloss{f}^2] 
\leq \sqrt{\E[(f(X) - \fopt(X))^4]}  \cdot \sqrt{C} 
\leq A \E[(f(X) - \fopt(X))^2]  \cdot \sqrt{C}
\leq A \xsrisk{f} \cdot \sqrt{C},
\end{align}
where the final inequality follows from \eqref{eq:l2excess} and
\begin{align*}
C & = \E[((f-\fopt) + 2 (Y-\fopt))^4] 
\leq \E [ (2(f-\fopt)^2 + 8 (Y-\fopt)^2)^2] \\ & \leq  \E [ 8(f-\fopt)^4 + 32 (Y-\fopt)^4] 
\leq 8 A^2 \E[(f-\fopt)^2]^2 + 32 \E[\loss_{\fopt}^2] \\
& \leq 8 A^2 \E[\xsloss{f}]^2 + 32 \E[\loss_{\fopt}^2] \leq 8 A^2 c_0^2 + 32 \E[\loss_{\fopt}^2],
\end{align*}
where the third and fifth inequality follow from our assumptions and the fourth follows from \eqref{eq:l2excess}. 
This quantity is bounded, so \eqref{eq:bernsteinisthere} implies the Bernstein condition. 
\end{Proof}

\section{Proofs for Section~\ref{sec:vcentral}}
\label{app:proofs-vcentral}
\subsection{Proof of Lemma~\ref{lem:renyi-to-kl-grip-a}}
We first prove (\ref{eq:unusual}) from the main text: suppose that $\smtuple$ satisfies the $v$-central condition. We then have 
for all $f \in \cF$, 
$$\Exp \left[e^{v(\epsilon) \cdot
(\ell_{f^*_{\epsilon}}-\ell_{f})} \right] 
=  \Exp \left[e^{v(\epsilon) \cdot
(\ell_{f^*}-\ell_{f})} \right] \cdot e^{- v(\epsilon)\epsilon} 
\leq 1,$$
where the inequality follows because $\smtuple$ satisfies the $v$-central condition.
Now suppose further that $(P,\ell, \{f \} \cup \{f^* \})$ satisfies the $(u,c)$-witness condition. This gives:
\begin{align*}
c \E[L_{f}] & \leq \E[( \ell_f - \ell_{f^*}) \cdot \ind{\ell_f - \ell_{f^*} \leq u}]  = \E[( \ell_f - \ell_{f^*}) \cdot \ind{\ell_f - \ell_{f^*_{\epsilon}} \leq u+ \epsilon}] \nonumber \\
& = \E[( \ell_f - (\ell_{f^*_{\epsilon}} + \epsilon)) \cdot \ind{\ell_f - \ell_{f^*_{\epsilon}} \leq u+ \epsilon}] \leq  \E[( \ell_f - \ell_{f^*_{\epsilon}}) \cdot \ind{\ell_f - \ell_{f^*_{\epsilon}} \leq u+ \epsilon}], 
\end{align*}
whence the $(u+\epsilon,c)$ witness condition holds for
$(P,\ell, \{f, f^*_{\epsilon} \})$. By this fact {and
  (\ref{eq:unusual}) (proven above)}, we can apply
Lemma~\ref{lem:renyi-to-kl} (our extension of
Lemma~\ref{lem:renyi-to-kl-simple} from the main text), with $\phi(f)$
set to $f^*_{\epsilon}$ (i.e.~$\phi(f)$ does not depend on $f$). The
result, (\ref{eq:rklsimplegrip-a}), follows.

\section{Proofs for Section~\ref{sec:vppc}}
\label{app:proofs-vcentral}
\subsection{Proof of Propositions~\ref{prop:grip-central}--\ref{prop:ppcprop}}
\begin{proof}(of Proposition~\ref{prop:grip-central})
Consider the learning problem $(P, \tilde{\loss}, \tilde{\cF})$ with
\begin{align*}
  \tilde{\cF} := \left\{ \mixloss{\eta}{Q} : Q \in \Delta(\cF) \right\} \cup \{ \grip{\eta} \}
\end{align*}
and $\tilde{\loss}_{\tilde{f}} := \tilde{f}$ for $\tilde{f} \in \tilde{\cF}$.

We will show that the {strong} $\eta$-PPC condition \Citep{erven2015fast} holds for this problem with $\grip{\eta}$ taking the role of the optimal action. That is,
\begin{align} 
\E \left[ \grip{\eta} \right] 
\leq \inf_{\tilde{Q} \in \Delta(\tilde{\cF})} \E \left[ -\frac{1}{\eta} \log \E_{\tilde{f} \sim \tilde{Q}} \left[ e^{-\eta \tilde{\loss}_{\tilde{f}}} \right] \right] . \label{eqn:tilde-ppc}
\end{align}
In one of their main results, \Citep[Theorem 3.10 and Corollary 3.11]{erven2015fast}, again extending an argument of \cite{li1999estimation}, show that the strong $\eta$-PPC condition implies the strong $\eta$-central condition for any tuple $(P,\ell, \tilde{\cF})$ under the sole assumption that $\tilde{\cF}$ contains a risk minimizer, i.e., there exists $f' \in \tilde{\cF}$ with  $\min_{f \in \tilde{\cF}}\E[ \ell_f] = \E[\ell_{f'}]$. But we construct $\tilde{\cF}$ so that this holds, since it contains $\grip{\eta}$. Thus, 
if \eqref{eqn:tilde-ppc} indeed holds (as we will soon show), then $(P, \tilde{\loss}, \tilde{\cF})$ also satisfies the 
the strong $\eta$-central condition. But this implies that, for all $\tilde{f} \in \tilde{\cF}$,
\begin{align*}
\E \left[ e^{-\eta ( \tilde{\loss}_{\tilde{f}} - \grip{\eta} )} \right] \leq 1 .
\end{align*}
The statement above holds in particular for any $\tilde{f} = \mixloss{\eta}{Q}$, which includes the special case of the Dirac mix losses of the form $\mixloss{\eta}{\delta_f} = \loss_f$ for any $f \in \cF$, and hence we have, for all $f \in \cF$,
\begin{align*}
\E \left[ e^{-\eta ( \loss_f - \grip{\eta} )} \right] \leq 1 \qquad \text{ for all } f \in \cF ,
\end{align*}
which is what we wanted.

Let us now prove inequality \eqref{eqn:tilde-ppc}. 
We start with the RHS of \eqref{eqn:tilde-ppc} and, via a sequence of lower bounds, will arrive at the LHS. 
First, observe that the RHS can be rewritten as
\begin{align*}
&\inf_{\alpha \in [0, 1]} \, \inf_{\tilde{Q} \in \Delta(\Delta(\cF))} 
\E \left[ -\frac{1}{\eta} \log \left( \alpha e^{-\eta \grip{\eta}} + (1 - \alpha) \E_{Q \sim \tilde{Q}} \left[ e^{-\eta \mixloss{\eta}{Q}} \right] \right) \right] \\
&= \inf_{\alpha \in [0, 1]} \, \inf_{Q \in \Delta(\cF)} 
\E \left[ -\frac{1}{\eta} \log \left( \alpha e^{-\eta \grip{\eta}} + (1 - \alpha) \mixloss{\eta}{Q} \right) \right] .
\end{align*}

Next, for each $\alpha$ and $Q$, we introduce a function $\Gamma_{\alpha,Q} \colon \reals \rightarrow \reals$, defined as
\begin{align*}
\Gamma_{\alpha,Q}(x) = -\frac{1}{\eta} \log \left( \alpha e^{-\eta x} + (1 - \alpha) \mixloss{\eta}{Q} \right) ,
\end{align*}
so that the last line in the above display may be rewritten as
\begin{align*}
\inf_{\alpha \in [0, 1]} \, \inf_{Q \in \Delta(\cF)} \E \left[ \Gamma_{\alpha,Q}(\grip{\eta}) \right] .
\end{align*}

Now, as we show in Appendix~\ref{app:grip}, there exists a sequence $(Q_n)_{n \geq 1}$ such that $\mixloss{\eta}{Q_n}$ converges to $\grip{\eta}$ in $L_1(P)$. 
For any $n \geq 1$, we have
\begin{align}
\E \left[ \Gamma_{\alpha,Q}(\grip{\eta}) \right] 
= \E \left[ \Gamma_{\alpha,Q}(\mixloss{\eta}{Q_n}) \right] 
    + \E \left[ \Gamma_{\alpha,Q}(\grip{\eta}) - \Gamma_{\alpha,Q}(\mixloss{\eta}{Q_n}) \right] \label{eqn:gamma-everywhere}
\end{align}
Note that $\Gamma_{\alpha,Q}$ is 1-Lipschitz, since (for any choice of $\alpha$ and $Q$),
\begin{align*}
\frac{d \Gamma_{\alpha,Q}}{dx} \Gamma_{\alpha,Q}(x) 
= -\frac{1}{\eta} \frac{-\eta \alpha e^{-\eta x}}{\alpha e^{-\eta x} + (1 - \alpha) e^{-\eta \mixloss{\eta}{Q}}} 
= \frac{\alpha e^{-\eta x}}{\alpha e^{-\eta x} + (1 - \alpha) e^{-\eta \mixloss{\eta}{Q}}} 
\in [0, 1] .
\end{align*}
Consequently, it holds that \eqref{eqn:gamma-everywhere} is lower bounded by
\begin{align*}
\E \left[ \Gamma_{\alpha,Q}(\mixloss{\eta}{Q_n}) \right] 
    - \E \left[ \left| \Gamma_{\alpha,Q}(\grip{\eta}) - \Gamma_{\alpha,Q}(\mixloss{\eta}{Q_n} \right| \right]
\geq \E \left[ \Gamma_{\alpha,Q}(\mixloss{\eta}{Q_n}) \right] 
    - \E \left[ \left| \grip{\eta} - \mixloss{\eta}{Q_n} \right| \right] .
\end{align*}
Next, since $\mixloss{\eta}{Q_n}$ converges to $\grip{\eta}$ in $L_1(P)$, taking the limit as $n \rightarrow \infty$, the RHS of the last line above converges to $\E \left[ \Gamma_{\alpha,Q}(\mixloss{\eta}{Q_n}) \right]$. 
Thus, we have shown that
\begin{align*}
\E \left[ \Gamma_{\alpha,Q}(\grip{\eta}) \right] 
\geq \lim_{n \rightarrow \infty} \E \left[ \Gamma_{\alpha,Q}(\mixloss{\eta}{Q_n}) \right] ,
\end{align*}
and so:
\begin{align*}
&\inf_{\alpha \in [0, 1]} \inf_{Q \in \Delta(\cF)} 
\E \left[ -\frac{1}{\eta} \log \left( \alpha e^{-\eta \grip{\eta}} + (1 - \alpha) e^{-\eta \mixloss{\eta}{Q}} \right) \right] \\
&\geq \inf_{\alpha \in [0, 1]} \inf_{Q \in \Delta(\cF)} 
\lim_{n \rightarrow \infty} \E \left[ -\frac{1}{\eta} \log \left( \alpha e^{-\eta \mixloss{\eta}{Q_n}} + (1 - \alpha) e^{-\eta \mixloss{\eta}{Q}} \right) \right] \\
&= \inf_{\alpha \in [0, 1]} \inf_{Q \in \Delta(\cF)} 
\lim_{n \rightarrow \infty} \E \left[ \mixloss{\eta}{\alpha Q_n + (1 - \alpha) Q} \right] \\
&\geq  \inf_{\alpha \in [0, 1]} \inf_{Q \in \Delta(\cF)} 
\lim_{n \rightarrow \infty} \E \left[ \grip{\eta} \right] \\
&= \E \left[ \grip{\eta} \right] ,
\end{align*}
where we used that the quantity inside $\lim_{n\rightarrow \infty}$ is equal
to $\E[ m^{\eta}_{Q'}]$ for some $Q' \in \Delta(\cF)$, and hence by
definition not smaller than $\E[\grip{\eta}]$. Thus, inequality
\eqref{eqn:tilde-ppc} indeed holds.
\end{proof}

\begin{proof}(of Proposition~\ref{prop:slowrate})
  Fix $\eta > 0$ and let $u$ be as in (\ref{eq:positivethinking}). 
For each $f \in \cF$,
  let $f'$ be defined by $\loss_{f'} = \loss_{f}$ if
  $\loss_f \leq \loss_{\fopt} + u$ and $\loss_{f'} = \loss_{\fopt}$ otherwise
  and let $\cF'$ be the resulting model. Then $\gripgen{\eta}{\cF'}$ is the GRIP relative to $\eta$ and the class $\cF'$; from Appendix~\ref{app:grip} this GRIP is guaranteed to exist. 
By definition, for every $\delta > 0$ there is a distribution $Q'$ on $\cF'$ such that 
$\E_{Z \sim P}[ \mixloss{\eta}{Q'} - \gripgen{\eta}{\cF'} ] \leq \delta$. 
Define $f^{\circ}$ such that it has constant loss, 
i.e., for all $z \in \cZ$, $\loss_{f^{\circ}}(z) := \E[ \loss_{\fopt}]$. 
By using $- \log x \geq 1-x $ and  we have for each $z \in \cZ$, for some $\eta' \in (0,\eta)$:
\begin{align*}
\mixloss{\eta}{Q'} - \loss_{f^{\circ}} 
= -\frac{1}{\eta} \log \E_{f' \sim Q'} e^{- \eta (\loss_{f'} - \loss_{f^{\circ}})} 
&\geq \frac{1}{\eta} \left(1 - \E_{f' \sim Q'} e^{- \eta (\loss_{f'} - \loss_{f^{\circ}})} \right) \\
&= \E_{f' \sim Q'}\left[ \loss_{f'}- \loss_{f^{\circ}}\right] 
- \frac{1}{2} \eta \E(\loss_{f'} - \loss_{f^{\circ}})^2 e^{- \eta' (\loss_{f'} - \loss_{f^{\circ}})} \\
&\geq \E_{f' \sim Q'}\left[ \loss_{f'}- \loss_{f^{\circ}}\right] 
- \frac{1}{2} e^{\eta \loss_{f^{\circ}}} \cdot \eta \E_{f' \sim Q'}(\loss_{f'} - \loss_{f^{\circ}})^2 .
\end{align*}
Now use that 
\begin{align*}
\E_{f' \sim Q'} \left[ (\loss_{f'} - \loss_{f^{\circ}})^2 \right] 
& = \E_{f' \sim Q'} \left[ \left( (\loss_{f'} - \loss_{\fopt}) + (\loss_{\fopt} - \loss_{f^{\circ}}) \right)^2 \right] \nonumber \\  & \leq  2\left( \E_{f' \sim Q'} \left[ (\loss_{f'} - \loss_{\fopt})^2 \right] + (\loss_{\fopt} - \loss_{f^{\circ}})^2 \right) \nonumber 
\\ & \leq 2\left( \E_{f' \sim Q'} \left[\ind{\loss_{f'} > \loss_{\fopt}}
(\loss_{f'} - \loss_{\fopt})^2 + \ind{\loss_{f'} \leq \loss_{\fopt}} (\loss_{f'} - \loss_{\fopt})^2 \right] + (\loss_{\fopt} - \loss_{f^{\circ}})^2 \right) \nonumber \\ &  \leq 2 u^2 + 2 \loss_{\fopt}^2 + (\loss_{\fopt}- \loss_{f^{\circ}})^2 .
\end{align*}
Combining this with the previous inequality and taking the expectation with respect to $Z$ yields
\begin{align*}
\E_{Z \sim P} \left[ \gripgen{\eta}{\cF'} - \loss_{\fopt} \right] 
&= \E_{Z \sim P} \left[ \mixloss{\eta}{Q'} - \loss_{f^{\circ}} \right] - \delta \\
&\geq \E_{Z \sim P} \E_{f' \sim Q'}\left[ \loss_{f'} - \loss_{\fopt}\right] 
           - \frac{1}{2} \eta e^{\eta \loss_{f^{\circ}}} 
              \cdot \left( 2 u^2 + \E_{Z \sim P} \left[ 2 \loss_{\fopt}^2 + (\loss_{\fopt} - \loss_{f^{\circ}})^2 \right] \right) 
           - \delta \\
&\geq \E_{Z \sim P} \E_{f' \sim Q'}\left[ (\loss_{f'} - \loss_{\fopt} ) \cdot  \ind{\loss_{f'} - \loss_{\fopt} \leq u} \right] 
            - \frac{1}{2} \eta e^{\eta \E[\loss_{\fopt}]} \cdot \left( 2 u^2 + 3 \E[\loss_{\fopt}^2]\right) 
            - \delta \\
&= \E_{f \sim Q} \E_{Z \sim P} \left[ (\loss_f - \loss_{\fopt} ) \cdot  \ind{\loss_f - \loss_{\fopt} \leq u} \right] 
            - \frac{1}{2} \eta e^{\eta \E[\loss_{\fopt}]} \cdot \left( 2 u^2 + 3 \E[\loss_{\fopt}^2]\right) 
            - \delta \\
& \geq -\frac{1}{2} \eta e^{\eta \E[\loss_{\fopt}]} \cdot \left( 2 u^2 + 3 \E[\loss_{\fopt}^2]\right) 
            - \delta ,
\end{align*}
where $Q \in \Delta(\cF)$ is the distribution defined by taking $dQ(f) = dQ'(f')$ (where we make use of the bijection between $\cF$ and $\cF'$ from the definition of $\loss_{f'}$ in terms of $f$, for all $f' \in \cF$), and the final inequality invokes (\ref{eq:positivethinking}). 
We now take $\eta \leq  1/\E[\loss_{\fopt}]$, yielding
\begin{align}\label{eq:laborate}
\E_{Z \sim P} \left[ \loss_{\fopt} - \gripgen{\eta}{\cF'} \right] 
\leq \eta \cdot e \cdot \left(u^2 + \frac{3}{2} \E[\loss_{\fopt}^2] \right) + \delta .
\end{align}
The result now follows from Proposition~\ref{prop:ppcprop}, using that  
the reasoning above holds for every $\delta > 0$.
\end{proof}

\begin{proof}{\bf (of Proposition~\ref{prop:ppcprop})}
Define the set $\cF'$ such that for each $f \in \cF$, there is an $f' \in \cF$ with $\loss'_f = \loss_{f'}$ and vice versa. 
Note that we must have:
\begin{align}\label{eq:aborate}
\E_{Z \sim P} \left[ \gripgen{\eta}{\cF'} \right] \leq \E_{Z \sim P} \left[ \grip{\eta} \right] .
\end{align}
To see this, assume for contradiction that there exists some $\varepsilon > 0$ such that $\E_{Z \sim P} \left[ \grip{\eta} \right] \leq \E_{Z \sim P} \left[ \gripgen{\eta}{\cF'} \right] - \varepsilon$. 
Let $(Q_j)_{j \geq 1}$ be a sequence for which $ \E_{Z \sim P} [ \mixloss{\eta}{Q_j} ] \leq \E_{Z \sim P} [ \grip{\eta} ] + \frac{\varepsilon}{2}$. We will make use of the fact that, for each $Q' \in \Delta(\cF')$, $\mixloss{\eta}{Q'} \leq \mixloss{\eta}{Q}$ since for each $f'$ the corresponding $f$ has, on all $z$, either the same or larger loss. This setup then implies the following contradiction:
\begin{align*}
\E_{Z \sim P} [ \gripgen{\eta}{\cF'} ] 
\leq \E_{Z \sim P} [ \mixloss{\eta}{Q'_j} ]
\leq \E_{Z \sim P} [ \mixloss{\eta}{Q_j} ]
\leq \gripgen{\eta}{\cF} + \frac{\varepsilon}{2} 
\leq \gripgen{\eta}{\cF'} - \frac{\varepsilon}{2} .
\end{align*}
Now, since by assumption 
$\loss_{\fopt} \equiv \loss_{(\fopt)'}$,  (\ref{eq:aborate}) implies that 
\begin{align*}
\E_{Z \sim P} \left[ \loss_{\fopt} - \gripgen{\eta}{\cF} \right] 
\leq \E_{Z \sim P} \left[ \loss_{\fopt} - \gripgen{\eta}{\cF'} \right] 
\end{align*}
which implies the statement of the proposition.
\end{proof}

\subsection{Proof of Lemma~\ref{lem:renyi-to-kl-grip-b}}
The proof of Lemma~\ref{lem:renyi-to-kl-grip-b} is based on relating
the loss $\gripgen{\bar\eta}{\cal F}$ of the GRIP comparator appearing
in that lemma to the loss of a related ``dynamic'' comparator
$\gripgen{\bar\eta}{f}$ (which we will call ``mini-GRIP'') that {\em
  varies\/} with $f$. This requires us to first re-define the witness
condition for such dynamic comparators, relate this dynamic witness
condition to the standard witness condition, and relate the GRIP loss
to the mini-GRIP loss; this is all achieved in the following
subsection.
\subsubsection{Witness Protection and Mini-grip}
\label{app:mini-grip}

\begin{assumption}[Advanced Empirical Witness of Badness] \label{ass:adv-emp-witness-badness}
Let $M \geq 1$ be a parameter of the assumption.
\glsaddcond{advanced-witness} We say that $\smtuple$ satisfies the \emph{empirical witness of badness condition} (abbreviated as \emph{witness condition}) with respect to \emph{dynamic} comparator $\phi$ if there exist constants $u > 0$ and $c \in (0,1]$ such that for all $f \in \cF$,
\begin{align} \label{eqn:adv-emp-witness-badness}
\E \left[ (\loss_f - \loss_{\phi(f)}) \cdot \ind{\loss_f - \loss_{\phi(f)} \leq u (1 \opmax (M^{-1} \xsrisk{f}))} \right] \geq c \E [ \loss_f - \loss_{\phi(f)} ] .
\end{align}
\glsaddcond{weak-advanced-witness}
If we modify the RHS of \eqref{eqn:adv-emp-witness-badness} so that the term $\E [ \loss_f - \loss_{\phi(f)} ]$ is replaced by the potentially smaller $\xsrisklong{f}$, then we call the condition the \emph{weak empirical witness of badness condition} (abbreviated as \emph{weak witness condition}).
\end{assumption}

In practice, we will assume only that the witness condition holds for the \emph{static} comparator $\psi: f \mapsto \fopt$ (so named because the comparator does not vary with $f$), as can already be handled through the simpler witness condition of Definition~\ref{def:witness}. However, because the central condition may not necessarily be satisfied with comparator $\fopt$, it is beneficial if a witness condition holds for a suitably-related comparator for which the central condition \emph{does} hold. The ideal candidate for this comparator turns out to be an $f$-dependent pseudo-loss, $\gripgen{\eta}{f}$, an instance of a GRIP (see Definition \ref{def:grip}).

The main motivation for our introducing the GRIP is that $\smtuple$ with comparator $\grip{\eta}$ satisfies the $\eta$-central condition (from Proposition~\ref{prop:grip-central}). 
The GRIP arises as a generalization of the reversed information projection of \cite{li1999estimation}, which is the special case of the above with $\eta = 1$, log loss, and $\cF$ a class of probability distributions. In this case, the GRIP, now a reversed information projection, is the (limiting) distribution $P^*$ which minimizes the KL divergence $\KL(\Pt \pipes P^*)$ over the convex hull of $\mathcal{P}$; note that $P^*$ is not necessarily in $\convhull(\mathcal{P})$. \citet[Theorem 4.3]{li1999estimation} proved the existence of the reversed information projection; for completeness, in Appendix~\ref{app:grip} we present a lightly modified proof of the existence of the GRIP.

As mentioned above, in our technical results exploiting both the central and witness conditions, we will need not only the ``full'' GRIP but also a ``mini-grip'' $\gripgen{\eta}{f}$, for each $f$, defined by replacing $\cF$ with $\{\fopt, f\}$ in Definition \ref{def:grip}. \glsadd{mini-grip} The mini-grip with respect to $f$ then has the simple, characterizing property of satisfying
\begin{align*}
\E [ \gripgen{\eta}{f} ] = \inf_{\alpha \in [0, 1]} \E \left[ -\frac{1}{\eta} \log \left( (1 - \alpha) e^{-\eta \loss_{\fopt}} + \alpha e^{-\eta \loss_f} \right) \right] .
\end{align*}
Also, as will be used to critical effect in the application of Lemma~\ref{lem:renyi-to-kl}, for each $f$ the learning problem $(\Pt, \{\fopt, f\}, \loss)$ with comparator $\gripgen{\eta}{f}$ satisfies the $\eta$-central condition.

Although up until now it has sufficed to refer to GRIPs only via their loss, for convenience of notation we now let \glsadd{pseudo-actions-for-grips} $\gripbase{\eta}$ denote the pseudo-action obtaining the GRIP loss $\grip{\eta}$, and we let $\gripbasegen{\eta}{f}$ denote the pseudo-action obtaining the mini-GRIP loss $\gripgen{\eta}{f}$. It should be emphasized that neither $\gripbase{\eta}$ nor $\gripbasegen{\eta}{f}$ need be well-defined; this is of no consequence, however, as we will use both only via their losses $\grip{\eta}$ and $\gripgen{\eta}{f}$, which \emph{are} well-defined.

We now show that if the witness condition holds with respect to the static comparator $\psi: f \mapsto \fopt$, then the weak witness condition holds with respect to the comparator $\phi: f \mapsto \gripbasegen{\eta}{f}$.

\begin{lemma}[Witness Protection Lemma] \label{lemma:witness-protection} 
Assume that $\smtuple$ satisfies the witness condition with static comparator $\psi: f \mapsto \fopt$ and constants $(M, u, c)$. Then, for any $\eta > 0$, $\smtuple$ satisfies the weak witness condition with dynamic comparator 
$\phi: f \mapsto \gripbasegen{\eta}{f}$ 
with the same constants $(M, u, c)$.
\end{lemma}
\begin{Proof}[Proof (of Lemma \ref{lemma:witness-protection} (Witness Protection Lemma))]
Let $f$ be arbitrary. For brevity we define $u' := u (1 \opmax (M^{-1} \xsrisk{f}))$. 
Observe that
\begin{align*}
\E \left[ (\loss_f - \gripgen{\eta}{f}) \cdot \ind{\loss_f - \gripgen{\eta}{f} > u'} \right] 
\leq \E \left[ (\loss_f - \loss_{\fopt}) \cdot \ind{\loss_f - \loss_{\fopt} > u'} \right] .
\end{align*}
Rewriting, we have
\begin{align*}
  \E [ \loss_f - \gripgen{\eta}{f} ] - \E \left[ (\loss_f - \gripgen{\eta}{f}) \cdot \ind{\loss_f - \gripgen{\eta}{f} \leq u'} \right]
  \leq \xsrisk{f} - \E \left[ (\loss_f - \loss_{\fopt}) \cdot \ind{\loss_f - \loss_{\fopt} \leq u'} \right] ,
\end{align*}
which we rearrange as
\begin{align*}
\E \left[ (\loss_f - \gripgen{\eta}{f}) \cdot \ind{\loss_f - \gripgen{\eta}{f} \leq u'} \right] 
&\geq
    \E \left[ (\loss_f - \loss_{\fopt}) \cdot \ind{\loss_f - \loss_{\fopt} \leq u'} \right]
    + \E [ \loss_f - \gripgen{\eta}{f} ] 
    - \xsrisk{f} \\
&=
    \E \left[ (\loss_f - \loss_{\fopt}) \cdot \ind{\loss_f - \loss_{\fopt} \leq u'} \right]
    + \E [ \loss_{\fopt} - \gripgen{\eta}{f} ] \\
&\geq
    \E \left[ (\loss_f - \loss_{\fopt}) \cdot \ind{\loss_f - \loss_{\fopt} \leq u'} \right] .
\end{align*}
From the assumed witness condition with static comparator $\psi: f \mapsto \fopt$, the RHS is lower bounded by
$c \xsrisk{f}$, 
and so we have established the weak witness condition with dynamic comparator $\phi$ and the same constants $(M, u, c)$.
\end{Proof}

\paragraph{From Hellinger mini-grip to GRIP}
\begin{lemma} \label{lemma:minigrip-to-grip}
For any $\eta > 0$ and $f \in \cF$,
\begin{align} \label{eqn:hell-grip-to-minigrip}
\Exphel{\eta} \left[\loss_f - \gripgen{\eta}{f} \right] 
\leq \Exphel{\eta/2} \left[\loss_f - \grip{\eta} \right].
\end{align}
\end{lemma}
\begin{proof}
Observe that 
\begin{align*}
\frac{1}{\eta / 2} \left( 1 - \E \left[ e^{-\frac{\eta}{2} (\loss_f - \grip{\eta})} \right] \right)  
&= \frac{1}{\eta / 2} \left( 1 - \E \left[ e^{-\frac{\eta}{2} (\loss_f - \gripgen{\eta}{f} + \gripgen{\eta}{f} - \grip{\eta})} \right] \right) \\
&\geq \frac{1}{\eta / 2} \left( 1 - \frac{1}{2} \E \left[ e^{-\eta (\loss_f - \gripgen{\eta}{f})} \right] - \frac{1}{2} \E \left[ e^{-\eta (\gripgen{\eta}{f} - \grip{\eta})} \right] \right) \\
&\geq \frac{1}{\eta / 2} \left( \frac{1}{2} - \frac{1}{2} \E \left[ e^{-\eta (\loss_f - \gripgen{\eta}{f})} \right] \right) = \frac{1}{\eta} \left( 1 - \E \left[ e^{-\eta (\loss_f - \gripgen{\eta}{f})} \right] \right) ,
\end{align*}
where the first inequality follows from Jensen's and for the second
inequality we use that, as we will now show,
$\E \left[ e^{-\eta (\gripgen{\eta}{f} - \grip{\eta})} \right] \leq 1$.
To show that this is indeed the case, recall that 
$\gripgen{\eta}{f} = -\frac{1}{\eta} \log \left( (1 - \alpha) e^{-\eta \loss_{\fopt}} + \alpha e^{-\eta \loss_f} \right)$.
Using this representation we find: 
\begin{align*}
\E \left[ e^{-\eta (\gripgen{\eta}{f} - \grip{\eta})} \right] 
= (1 - \alpha) \E \left[ e^{-\eta (\loss_\fopt - \grip{\eta})} \right] + \alpha \E \left[ \alpha e^{-\eta (\loss_f - \grip{\eta})} \right] \leq 1 .
\end{align*}
\end{proof}

Next, we chain $1 - x \leq -\log x$, Lemma \ref{lemma:minigrip-to-grip}, 
and Lemma~\ref{lem:renyi-to-kl} to obtain a bound that we will use in the proofs of Theorems \ref{thm:main-bounded-b} and \ref{thm:main-unbounded}.

\subsection{Actual Proof of Lemma~\ref{lem:renyi-to-kl-grip-b}}
Let $f \in \cF$. 
Let $u > 0$ and $c \in (0, 1]$ be constants for which 
$\E \left[ \xsloss{f} \cdot \ind{\xsloss{f} \leq u} \right] \geq c \xsrisk{f}$, 
i.e., the $(u,c)$-witness condition holds. Below we  show that 
for all $\eta \in (0, \frac{\bar{\eta}}{2})$
\begin{align} \label{eq:uberwitness}
\xsrisk{f} \leq c'_{2u} \Expann{\eta} \left[ \loss_f - \grip{\bar{\eta}} \right] ,
\end{align}
with $c'_{2u} = \frac{1}{c} \frac{2 \eta u + 1}{1 - \frac{2 \eta}{\bar{\eta}}}$.

\emph{Proof of \eqref{eq:uberwitness}.}
We have from \eqref{eq:genrenc} and Lemma \ref{lemma:minigrip-to-grip} 
that 
\begin{align*}
\Expann{\eta} \left[ \loss_f - \grip{\bar{\eta}} \right] 
\geq \Exphel{\eta}\left[ \loss_f - \grip{\bar{\eta}} \right].
\end{align*}
Now Lemma \ref{lemma:witness-protection} establishes the weak witness condition with respect to comparator $\gripbasegen{\bar{\eta}}{f}$, and from Proposition~\ref{prop:grip-central} this comparator further satisfies 
$\E \left[ e^{-\bar{\eta} (\loss_f - \gripgen{\bar{\eta}}{f})} \right] \leq 1$, 
so that we may apply Lemma~\ref{lem:renyi-to-kl} with $\phi(f) = \gripbasegen{\bar{\eta}}{f}$ to further lower bound the above by $\frac{1}{c'_{2u}} \xsrisk{f}$.

\subsection{Proof of Theorem \ref{thm:main-bounded-b}}
\label{app:proof-main-bounded-b}
Theorem \ref{thm:main-bounded-b} now follows easily from Lemma~\ref{lem:renyi-to-kl-grip-b}:
fix some $\epsilon \geq 0$. 
First, Lemma~\ref{lem:rascbound-complex} (our extension of Lemma~\ref{lem:rascbound-simple} from the main text) states for our particular choice of $\eta$ that
\begin{align}
\E_{\rv{f} \sim \dol_n} \left[ -\frac{1}{\eta} \log \E \left[ e^{-\eta (\loss_{\rv{f}} - \grip{v(\epsilon)})} \right] \right] 
\stochleq_{\eta \cdot n} \E_{\rv{f} \sim \dol_n} \left[ \frac{1}{n} \sum_{j=1}^n (\loss_{\rv{f}}(Z_j) - \grip{v(\epsilon)}(Z_j)) \right] + \frac{\KL(\dol_n \pipes \Prior)}{\eta n} . \label{eqn:step-1}
\end{align}
Weakening this to an in-expectation statement via part (i) of Proposition \ref{prop:drop}, and combining the in-expectation version with  
Lemma~\ref{lem:renyi-to-kl-grip-b}, \eqref{eq:rklsimplegrip-b}  implies that, for $c'_{2u} = \frac{1}{c} \frac{2 \eta u + 1}{1 - \frac{2 \eta}{v(\epsilon)}}$,
\begin{align}\label{eq:neededlater}
\E_{Z_1^n} \left[ \E_{\rv{f} \sim \dol_n} \left[ \xsrisk{\rv{f}} \right] \right] 
\,\, \leq \,\, c'_{2u} \E_{Z_1^n} \left[ \E_{\rv{f} \sim \dol_n} \left[ \frac{1}{n} \sum_{j=1}^n (\loss_{\rv{f}}(Z_j) - \grip{v(\epsilon)}(Z_j)) \right] + \frac{\KL(\dol_n \pipes \Prior)}{\eta n} \right] .
\end{align}
Now, the $v$-PPC condition implies that $\E [ \loss_{\fopt} ] \leq \E
[ \grip{v(\epsilon)} ] + \epsilon$, implying the result
\eqref{eqn:v-ppc-bound}.

\section{Proofs for Section~\ref{sec:main-unbounded}}
\subsection{Proof of Proposition~\ref{prop:smallballagain}}
We first state another proposition that is of independent interest,
relating generalized ``small-ball'' assumptions to weakenings thereof
which resemble the witness condition.
\begin{definition} We say that a collection of nonnegative random
  variables $ \{S_a: a \in \cA \}$ satisfies the \emph{generalized
    small-ball condition} if there exist constants $C_1, C_2$ with
  for all $a \in \cA$, $P(S_a \geq C_1 \E[S_a]) \geq C_2$ (Mendelson's
  (\citeyear{mendelson2014learning}) small-ball assumption in
  Example~\ref{ex:smallball} and~\ref{ex:smallballagain} is the case
  with $\cA = \cF \times \cF$, $S_{f,g}:= (f(X)-g(X))^2$,
  $C_1 = \kappa^2, C_2 = \epsilon$). We say that
  $ \{S_a: a \in \cA \}$ satisfies the \emph{generalized weakened small-ball condition} if there exist constants $C'_1, C'_2$ with for
  all $a \in \cA$,
  $\E[ \ind{S_a < C'_1 \E[S_a]} \cdot S_a] \geq C'_2 \E[S_a]$.
\end{definition}
The term ``weakened'' comes from the following proposition:
\begin{proposition}\label{prop:smallballfun}
Suppose that the generalized small-ball condition holds with constants $C_1$ and $C_2$. Then the generalized weakened small-ball condition holds with constants $C'_1 = 2/C_2$ and $C'_2 = (C_1 C_2)/2$. 
\end{proposition}
\begin{proof}
From Markov's inequality, we have for all $a \in \cA$, 
$P(S_a < (2/C_2) \E[S_a]) \geq 1- C_2/2$. 
In combination with the small-ball assumption, this implies 
\begin{align*}
P\left(C_1 \E[S_a] \leq S_a < \frac{2}{C_2} \E[S_a] \right) \geq \frac{C_2}{2},
\end{align*}
and so, since $S_a \geq 0$, 
\begin{align*}
\E\left[ \ind{S_a < (2/C_2) \E[S_a]} \cdot S_a\right] \geq
\E\left[ \ind{C_1 \E[S_a]  \leq  S_a  <  (2/C_2) \E[S_a]} \cdot S_a\right] \geq 
\frac{C_2}{2} \cdot C_1 \cdot \E[S_a] ,
\end{align*}
and the result follows. 
\end{proof}
\begin{proof}{\bf (of Proposition~\ref{prop:smallballagain})}
  Take some $c_0 > b$, with a precise value to be established
  later. First consider the set $\{ f \in \cF: \xsrisk{f} >c_0 \}$.
  Define the random variable $S_f := (f(X)-\fopt(X))^2$ and
  $T_f := 2(\fopt(X) - Y) (f-\fopt)$. From (\ref{eq:smallballhandy}) we
  see that $\xsloss{f} = S_f + T_f$.  Hence for every $c> 0$,
\begin{align}
& \E[\xsloss{f} \cdot \ind{\xsloss{f} \geq c \E[\xsloss{f}]}] \nonumber \\ 
&\leq \E[ S_f \cdot \ind{S_f \geq T_f} \cdot \ind{S_f + T_f \geq c \E[ \xsloss{f} ]} ] 
          + \E[ S_f \cdot \ind{S_f < T_f} \cdot \ind{S_f + T_f \geq c \E[ \xsloss{f} ]} ] 
          + \E [ |T_f| ] \nonumber \\
&\leq \E[ S_f \cdot \ind{S_f \geq T_f} \cdot \ind{2 S_f \geq c \E[ \xsloss{f} ]} ] 
          + \E[ T_f \cdot \ind{S_f < T_f} \cdot \ind{S_f + T_f \geq c \E[ \xsloss{f} ]} ] 
          + \E [ |T_f| ] \nonumber \\
&\leq \E[ S_f \cdot \ind{S_f \geq T_f} \cdot \ind{S_f \geq (c/2) \E[ S_f ]} ] 
          + 2 \E [ |T_f| ] \label{eq:twoterms} ,
\end{align}
where the last inequality follows since $\E [ S_f ] \leq \E [ \xsloss{f} ]$, owing to \eqref{eqn:convexity-inequality}.

We now bound both terms further. 
By Cauchy-Schwarz, the second term satisfies
\begin{align*}  
2 \E[|T_f|] 
&= 4 \E[|Y - \fopt||f-\fopt|] \\
&\leq 4 \sqrt{\E[(Y- \fopt)^2] \cdot \E[S_f^2]} 
\leq 4 \sqrt{\frac{\E[\loss_{\fopt}]}{\xsrisk{f}}} \cdot \xsrisk{f} 
< 4 \sqrt{\frac{\E[\loss_{\fopt}]}{c_0}} \cdot \xsrisk{f} .
\end{align*}
Plugging in $c':= (c/2) = 2/ \epsilon$, the first term can be rewritten, by Proposition~\ref{prop:smallballfun} and our assumption that the small-ball assumption holds, as
\begin{align*}
\E[S_f] - \E[S_f \cdot \ind{S_f < (c/2) \E[S_f]}] 
\leq \E[S_f] - \frac{\kappa^2 \epsilon}{2} \E[S_f] 
= (1 - \frac{\kappa^2 \epsilon}{2}) \E[S_f] 
\leq (1- \frac{\kappa^2 \epsilon}{2}) \E[\xsloss{f}] ,
\end{align*}
so that with (\ref{eq:twoterms}) we get 
\begin{align*}
\E[\xsloss{f} \cdot \ind{\xsloss{f} \geq c' \E[\xsloss{f}]}] 
\leq C' \E[\xsloss{f}] ,
\end{align*}
for
$C'= \left( \left(1 - \frac{\kappa^2 \epsilon }{2} \right) + 4
  \sqrt{\frac{\E[\loss_{\fopt}]}{c_0}} \right)$.  We now pick $c_0$ large
enough such that $C' < 1$. It then follows by the
characterization (\ref{eq:reversewitness}) of the witness condition that the 
set $\{f \in \cF: \xsrisk{f} \geq c_0 \}$ satisfies the $(\tau,c)$-witness 
condition with $\tau(x) = c' x$ for $c'= 2 / \epsilon$ and constant $c = 1 - C'$.

For the set $\{ f \in \cF: \xsrisk{f} <  c_0 \}$, note that we have
already shown (Example~\ref{ex:heavy-tailed-regression}) that the
Bernstein condition implies the basic witness condition. This implies
that there exists $u > 0$ such that $\{ f \in \cF: \xsrisk{f} \leq c_0
\}$ satisfies the $(u,c)$-witness condition for $c = \frac{1}{2}$.

Putting the two statements for both subsets of $\cF$ together, 
it follows that $\cF$ satisfies the $(\tau,c)$-witness condition 
with any $\tau$ such that $\tau(x) \geq u \opmax \frac{2 x}{\epsilon}$ for all $x$ 
and with $c = (1 - C') \opmin \frac{1}{2}$; the result follows.
\end{proof}

\subsection{Proof of Theorem \ref{thm:main-unbounded}}
\label{app:proof-main-unbounded}
We will need the following lemma, whose proof is a straightforward extension of the proofs of Theorem \ref{thm:main-bounded-a} and Theorem \ref{thm:main-bounded-b}:
\begin{lemma}
\label{lem:preliminaryunbounded}
With $\tau$ as in the statement of Theorem~\ref{thm:main-unbounded},
we get for any $\epsilon \geq 0$, any
$0 < \eta < \frac{v(\epsilon)}{2}$:
\begin{align} 
\label{eqn:v-central-boundb}
& \text{under $v$-central: }\ \E_{\rv{f} \sim \dol_n} \left[ \xi(
\xsrisk{\rv{f}}) \right]
\ \ \stochleq_{\frac{\eta \cdot n}{2 c_{u + \epsilon}}}
\ \ c_{u + \epsilon} \left( \rsc_{n,\eta}(\nocomparator \dolest) + \epsilon \right)\\
& \text{under $v$-PPC: }\  \label{eqn:v-ppc-boundb}
\E_{Z_1^n} \left[ \E_{\rv{f} \sim \dol_n} \left[ \xi( \xsrisk{\rv{f}} ) \right] \right]
\leq
c'_{2 u} \left( \E_{Z_1^n} \left[ \rsc_{n,\eta}(\nocomparator \dolest) \right] + \epsilon \right),
\end{align}
where $c_u := \frac{u}{c} \frac{\eta + 1}{1 - \frac{ \eta}{v(\epsilon)}}$ and 
$c'_{2u} := \frac{u}{c} \frac{2 \eta + 1}{1 - \frac{2 \eta}{v(\epsilon)}}$ 
and $\xi( \xsrisk{ {f}}) = 1 \opmin \xsrisk{{f}}$. 
\end{lemma}
\begin{proof}
(\ref{eqn:v-central-boundb})  
follows by following essentially the same steps as in the proof of Theorem~\ref{thm:main-bounded-a}, but splitting the expectation in two parts:
\begin{align}\label{eq:ronald}
  \E_{\rv{f} \sim \dol_n} \left[ \xi( \xsrisk{\rv{f}}) \right] =
  \E_{\rv{f} \sim \dol_n} \left[ \ind{\xsrisk{\rv{f}} < 1} \cdot
    \xsrisk{\rv{f}}\right] + \E_{\rv{f} \sim \dol_n} \left[
    \ind{\xsrisk{\rv{f}} \geq 1} \cdot 1 \right].\end{align} Fix
some $\epsilon \geq 0$.  The first term on the right of (\ref{eq:ronald})  can be bounded as follows,
using Lemma~\ref{lem:renyi-to-kl-grip-a} and  the fact that a $(u,c)$-witness condition is assumed for $f$ with $\xsrisk{f} < 1$ in combination with
(\ref{eqn:step-1}) and the fact that for $c> 0$ and general random
variables $U,V$, we have $U \stochleq_a V \Leftrightarrow c U
\stochleq_{a/c} c V$:
\begin{align*}
\E_{\rv{f} \sim \dol_n} \left[ \ind{\xsrisk{\rv{f}} < 1} \cdot 
\xsrisk{\rv{f}}\right] \stochleq_{\eta n / c_{u+\epsilon} } \,\, c_{u+\epsilon} \cdot \left(
\E_{\rv{f} \sim \dol_n} \left[ \frac{1}{n} \sum_{j=1}^n (\loss_{\rv{f}}(Z_j) - \ell_{f^*_{\epsilon}}(Z_j)) \right] + \frac{\KL(\dol_n \pipes \Prior)}{\eta n} \right).
\end{align*}
The second term on the right of (\ref{eq:ronald}) can similarly be bounded, using that $\tau(\E[L_f]) = u \E[L_f]$ for all $f$ with $\E[L_f] \geq 1$: 
\begin{multline}
\E_{\rv{f} \sim \dol_n} \left[ \ind{\xsrisk{\rv{f}} \geq 1} \cdot 
\frac{\xsrisk{\rv{f}}}{\xsrisk{\rv{f}}} 
\right] \stochleq_{\eta n / B} \\ B \cdot \left( 
\E_{\rv{f} \sim \dol_n} \left[ \frac{1}{n} \sum_{j=1}^n (\loss_{\rv{f}}(Z_j) - \ell_{f^*_{\epsilon}}(Z_j)) \right]  + \frac{\KL(\dol_n \pipes \Prior)}{\eta n} \right),
\nonumber
\end{multline}
where $B= \sup_{f: \E[L_f] \geq 1} {c_{u \E[L_f]+ \epsilon}/\E[L_f]}$.
The result (\ref{eqn:v-central-boundb}) now follows by adding the two
terms using Proposition~\ref{prop:drop} and bounding $B$ by using that
$c_{u \cdot a+\epsilon}/a \leq c_{u+\epsilon}$ for $a \geq 1$.

(\ref{eqn:v-ppc-boundb}) follows in similar fashion, by repeating the
proof of Theorem~\ref{thm:main-bounded-b}, but again splitting the
expectation of $\xi(L_{f})$ in two parts, just like above; we omit the
details.
\end{proof}
\begin{Proof}[Proof (of Theorem \ref{thm:main-unbounded})]
We start by  establishing the key
  inequality (\ref{eq:almostthere}) below both under the $v$-central
  and the $v$-PPC condition, but with different values for $r_n$ in
  (\ref{eq:almostthere}).  For this, we invoke Lemma~\ref{lem:preliminaryunbounded}.  This gives that  
the $v$-PPC condition implies, via (\ref{eqn:v-ppc-boundb}) and Markov's
inequality, that for all $\delta \geq 0$, with probability at least
$1-\delta$,
\begin{equation}
\label{eq:prepwork} \E_{\rv{f} \sim \dol_n} \left[ \xi( \xsrisk{\rv{f}} ) \right] 
\leq r_n,
\end{equation}
where 
$r_n = \frac{c'_{2u}}{\delta} \cdot \left( \E\left[{\rsc}_{n,\eta_n}\right] + \epsilon_n \right)$.

On the other hand, under the $v$-central condition, \eqref{eqn:v-central-boundb} holds and via
  Proposition~\ref{prop:drop} we can turn it into a high probability
  bound. Combining this bound with (\ref{eq:rscupperbound}) via a
  standard union bound argument gives that, for all $\delta > 0$, with
  probability at least $1- \delta$, (\ref{eq:prepwork}) holds, with $\xi$ as before but now with
$r_n = 
c_{u+\epsilon_n} C_{n,\delta} \left( \E\left[\overline{\rsc}_{n,\eta_n}\right] + \epsilon_n 
  + \frac{2}{n\eta_n} \right).$ 
Rewriting (\ref{eq:prepwork}) gives that, with probability at least $1- \delta$,
\begin{align}\label{eq:almostthere}
\dol_n\left( \left\{ \rv{f}: \xsrisk{\rv{f}}\geq 1 \right\} \right) +
\E_{\rv{f} \sim \dol_n} \left[ \ind{
\xsrisk{\rv{f}} <  1} \cdot \xsrisk{\rv{f}}) \right] 
\leq  r_n .
\end{align}
\emph{Part 1, Deterministic Estimators.} 
For deterministic $\dolest \equiv (\estim{f},\Prior)$, \\ (\ref{eq:almostthere}) simplifies to
$
\ind{\xsrisk{\estim{f}} \geq 1}  
+ \ind{\xsrisk{\estim{f}} <  1} \cdot \xsrisk{\estim{f}}  
\leq r_n ,
$
which further implies that with probability at least $1- \delta$, simultaneously,
\begin{align}\label{eq:hungry}
\ind{\xsrisk{\estim{f}} \geq 1} \leq r_n 
\text{ and } 
\ind{\xsrisk{\estim{f}} <  1} \cdot \xsrisk{\estim{f}} 
\leq r_n ,
\end{align}
and both the result for the $v$-PPC condition \eqref{eq:deterministicboundppc} and
the $v$-central condition \eqref{eq:deterministicboundcentral} follow 
by noting that
we may assume $n$ large enough so that $r_n < 1$, so that (\ref{eq:hungry}) is logically equivalent to
\begin{align*}
\xsrisk{\estim{f}} <  1 
\text{ and } 
\ind{\xsrisk{\estim{f}} <  1} \cdot \xsrisk{\estim{f}} 
\leq r_n ,
\end{align*}
which in turn is equivalent to $ \xsrisk{\estim{f}} \leq r_n$, and
thus the results are implied. 

\emph{Part 2, General Learning Algorithms.} 
Here we assume the $v$-PPC condition, so we can use \eqref{eq:almostthere} with $r_n$ as in the $v$-PPC case. 

By
Markov's inequality, for any sequence
$b_1, b_2, \ldots$ of positive numbers tending to $\infty$,
\begin{align*}
\dol_n \left( \left\{ f\in \cF: 1 >  \xsrisk{{f}} > b_n r_n \right \}\right) 
&= \dol_n \left( \ind{ \xsrisk{\rv{f}} < 1} \cdot \xsrisk{\rv{f}} > b_n r_n \right) \\
&\leq \frac{ \E_{\rv{f} \sim \dol_n} \left[ 
                        \ind{ \xsrisk{\rv{f}} < 1} \cdot \xsrisk{\rv{f}} 
                    \right]}
                  {b_n r_n} .
\end{align*}
Combining this with (\ref{eq:almostthere}) (dropping the
leftmost term in that inequality) gives that with probability at least
$1-\delta$,
$$
\dol_n \left( \left\{ f\in \cF: 1 >  \xsrisk{{f}} >
  b_n r_n \right \}\right) \leq \frac{1}{b_n}.
$$
Combining this again with (\ref{eq:almostthere}), 
now dropping the second term in the inequality and using a standard union bound, gives that with probability at least $1- 2 \delta$,
\begin{align*}
\dol_n \left( \left\{ f\in \cF: \xsrisk{{f}} >
  b_n r_n \right \}\right) \leq \frac{1}{{b_n}} + r_n,
\end{align*}
which, plugging in the definition of $r_n$ and $\overline{\rsc}_{n,\eta}$ on the left, can be rewritten as, for each $n$, each $\delta$, with $a_n$ as in the theorem statement:
\begin{equation}\label{eq:hehe}
\text{With probability $\geq 1- 2 \delta$:}\ \ 
\dol_n \left( \left\{ 
f\in \cF: 
\xsrisk{{f}} >  \frac{b_n}{a_n} \cdot \frac{c'_{2 u}}{\delta} \cdot 
\left( \E[\overline{\rsc}_{n,\eta} + \epsilon_n]  \right) \right\}\right) \leq \frac{1}{{b_n}} + r_n.
\end{equation}
Now choose $\delta = 1/\sqrt{a_n} \rightarrow 0$ as a function of $n$, and choose $b_n= \sqrt{a_n} \rightarrow \infty$. Then (\ref{eq:hehe}) implies the result.
\end{Proof}

\subsection{Proof of Proposition~\ref{prop:slowerrate}}
\begin{Proof}
Let $c$, $u$
  and $\tau$ be as in the statement of the proposition.  For each
  $f \in \cF$, we will define modified predictors $f'$, defined in
  terms of their losses $\ell_{f'}$ so that for all such $f'$, we have
\begin{align}\label{eq:etrangeflic}
 \E [ \left(  \xslosslong{f'} \right) \cdot \ind{\xslosslong{f'} \leq u'} ] 
\geq 0,
\text{\ for\ }  u' =  u \cdot
\left(\frac{\E[\ell_{f^*}]}{c} \opmax 1 \right),
\end{align}
which allows us to apply Proposition~\ref{prop:slowrate} to the set of
$f'$; we will also ensure that for all $z \in \cZ$,
\begin{equation}\label{eq:superflic}
\ell_{f'}(z) \leq \ell_{f}(z)  \text{\ \ and\ \ }
\ell_{(f^*)'}(z) = \ell_{f^*}(z),
\end{equation}
which will allow us to
apply Proposition~\ref{prop:ppcprop} so that results for $f'$
transfer to the original $f$. Once we have shown (\ref{eq:etrangeflic}) and (\ref{eq:superflic}), the result follows.

\paragraph{Case 1: $\xsrisk{f} \leq (\E[\ell_{f^*}]/c) \opmax 1$.} 
For all $f$ with $\xsrisk{f} \leq (\E[\ell_{f^*}]/c) \opmax 1$
(including $f^*$), we simply set $f'= f$. Then (\ref{eq:superflic})
holds trivially. To see that (\ref{eq:etrangeflic}) holds, note that
the assumed $\tau$-witness condition holds for $\tau(\xsrisk{f}) = u (1
\opmax \xsrisk{f}) \leq u (1 \opmax (\E[\ell_{f^*}]/c \opmax 1))$,
which is no larger than the $u'$ mentioned in (\ref{eq:etrangeflic}),
which then immediately follows by the assumed witness condition. 

\paragraph{Case 2: $\xsrisk{f} > (\E[\ell_{f^*}]/c) \opmax 1$.}
For these $f$, we define
\begin{equation*}
\ell_{f'}(z) = \begin{cases} \ell_{f}(z) & \text{if $\ell_f(z) \leq \ell_{f^*}(z)$}  \\
\frac{\ell_{f}(z) - \ell_{f^*}(z)}{c'} + \ell_{f^*}(z) &  \text{if $\ell_f(z) >  \ell_{f^*}(z)$}, 
\end{cases} 
\end{equation*}
with $c' := \xsrisk{f}/(\E[\ell_{f^*}/c] \opmax 1)$, which by construction must satisfy $c'> 1$. This implies after rearranging
terms that  (\ref{eq:superflic}) holds. It thus
remains to prove (\ref{eq:etrangeflic}). 
To see that it holds, first note that 
$\ell_{f'} > \ell_{f^*} \Leftrightarrow \ell_{f} > \ell_{f^*}$ 
and that $\ell_{f} \geq 0$ on all $z$. 
Using these facts we find that:
 \begin{align}
& \E [ \left(  \xslosslong{f'} \right) \cdot \ind{\xslosslong{f'} \leq u'} ] 
\nonumber \\
\geq & - \E[\ind{\xslosslong{f'} \leq 0} \ell_{f^*}] + 
\E [ \left(  \xslosslong{f'} \right) \cdot 
\ind{\ell_{f'} -  \ell_{f^*} > 0} \cdot
\ind{\xslosslong{f'} \leq u'} ] 
\nonumber \\
\geq & - \E[\ell_{f^*}] + 
\E [ \left(  \xslosslong{f'} \right) \cdot 
\ind{\ell_{f} > \ell_{f^*}} \cdot
\ind{\xslosslong{f'} \leq u'} ] 
\nonumber \\
= & - \E[\ell_{f^*}] + 
\E \left[ \left(  \frac{\xslosslong{f}}{c'} \right) \cdot 
\ind{\ell_{f} > \ell_{f^*}} \cdot
\ind{\xslosslong{f} \leq u'c'} \right] 
\nonumber \\
=  & - \E[\ell_{f^*}] + 
\frac{1}{c'}\E [ \left(  {\xslosslong{f}} \right) \cdot 
\ind{\ell_{f} > \ell_{f^*}} \cdot
\ind{\xslosslong{f} \leq u' 
\xsrisk{f}/((\E[\ell_{f^*}]/c) \opmax 1)} ] 
\nonumber \\
=  & - \E[\ell_{f^*}] + 
\frac{1}{c'}\E [ \left(  {\xslosslong{f}} \right) \cdot 
\ind{\ell_{f} > \ell_{f^*}} \cdot
\ind{\xslosslong{f} \leq u \xsrisk{f}} ]. 
\nonumber \\
=  & - \E[\ell_{f^*}] + 
\frac{1}{c'}\E [ \left(  {\xslosslong{f}} \right) \cdot 
\ind{\ell_{f} > \ell_{f^*}} \cdot
\ind{\xslosslong{f} \leq u (\xsrisk{f} \opmax 1)} ]. 
\label{eq:checkthis} \\
\geq   & - \E[\ell_{f^*}] + 
\frac{1}{c'}\E [ \left(  {\xslosslong{f}} \right) \cdot 
\ind{\xslosslong{f} \leq u (\xsrisk{f} \opmax 1)} ] \nonumber \\
\geq & - \E[\ell_{f^*}] + 
\frac{1}{c'} \cdot c \xsrisk{f} 
\geq  - \E[\ell_{f^*}] + \E[\ell_{f^*}] = 0, \label{eq:checkthat}
\end{align}
where (\ref{eq:checkthis}) follows because all $f$'s we consider here have $\xsrisk{f} > 1$ and (\ref{eq:checkthat}) follows by our assumption of the $\tau$-witness condition.

\end{Proof}

\section{The Existence of the Generalized Reversed Information Projection}
\label{app:grip}

Recall that $\mathcal{E}_{\cF,\eta}$ is the the entropification-induced set $\left\{ e^{-\eta \loss_f} : f \in \cF \right\}$. 
In this section, we prove the existence of the generalized reversed information projection $\grip{\eta}$ of $\Pt$ onto $\convhull(\mathcal{E}_{\cF,\eta})$. 
Because $\cF$ and $\eta$ are fixed throughout, we adopt the notation $\mathcal{E} := \mathcal{E}_{\cF,\eta}$ and $\mathcal{C} := \convhull(\mathcal{E}_{\cF,\eta})$.

Formally, we will show that there exists $q^*$ (not necessarily in $\mathcal{C}$) satisfying
\begin{align*}
\E [ - \log q^*(Z) ] 
= \inf_{q \in \mathcal{C}} \E [ - \log q(Z) ] .
\end{align*}
One might think that there is an easy proof by simply taking $q^*$ to
lie in the closure of $\mathcal{C}$ under some appropriate topology,
but it is not evident what topology to take. For example, even in the
simple case with $\eta=1$ and $\ell_f$ is the log-loss so that
$\mathcal{E}$ and $\mathcal{C}$ are sets of probability densities, it
may happen that $q^*$ is a sub-density (integrating to less than 1)
\citep{li1999estimation} so that it would not lie in the closure of
any standard topology which we may impose on $\mathcal{C}$. We thus
follow a different approach. We first rewrite the above in the
language of information geometry. To provide easier comparison to
\cite{li1999estimation} we use the following modified $\KL$ notation
here for a generalized KL divergence, which in particular makes the
underlying distribution $P$ explicit:
\begin{align*}
\KL(p; q_0 \pipes q) := \E_{Z \sim P} \left [ \log \frac{q_0(Z)}{q(Z)} \right] ,
\end{align*}
where $q_0$ and $q$ are nonnegative but neither need be a normalized probability density. 
Then the existence question above is equivalent to the existence of $q^*$ such that
\begin{align*}
\KL(p; q_0 \pipes q^*) 
= \inf_{q \in \mathcal{C}} \KL(p; q_0 \pipes q) ;
\end{align*}
here, the only restriction on $q_0$ is that $\E_{Z \sim P} [ \log q_0 ]$ be finite.

Now, \cite{li1999estimation} already showed the above in the case of density estimation with log loss, $\eta = 1$, and $q_0 = p$; in that setting, we have $e^{-\eta \loss_f} = f$, and so mixtures of elements of $\mathcal{E}$ correspond to mixtures of probability distributions in $\cF$. Hence, our setting is more general, yet Li's argument (with minor adaptations) still works. To be sure, we go through his argument step-by-step and show that it all still works in our setting.

In the remainder of this section, we treat two cases simultaneously unless a separate treatment is indicated: the case when the loss is uniformly bounded from below (as in Appendix~\ref{sec:infinities-ulb}) and the case of log loss (with the loss not uniformly bounded from below, as in Appendix~\ref{sec:infinities-log-loss}). In the former case, we always take $q_0 = e^{-\eta \loss_{\fopt}}$. In the latter case, we always take $q_0 = p$.

\subsection{Proving $q^*$ exists}

Throughout, we will need to assume the existence of a certain sequence $(q_n)_{n \geq 1}$ in $\mathcal{C}$, satisfying $\KL(p; q_0 \pipes q_n) \rightarrow \inf_{q \in \mathcal{C}} \KL(p; q_0 \pipes q)$, for which $\KL(p; q_0 \pipes q_n)$ is finite for all $n$. This is not problematic, as we now explain. We treat separately the case of losses uniformly bounded from below and the case of log loss without a uniform lower bound on the loss.

\paragraph{Losses uniformly bounded from below.}

First, observe that for any $q_n \in \mathcal{C}$,
\begin{align*}
\KL(p; q_0 \pipes q_n) \geq -\|\loss_-\|_\infty - \E [ \loss_{\fopt} ] > -\infty .  
\end{align*}
To see this, observe that $q_n = \E_{f \sim R_n} [ e^{-\eta \loss_f} ]$ for some distribution $R_n \in \Delta(\cF)$; then assumption \eqref{eqn:bounded-below} gives the first inequality. The second inequality holds because we only deal with non-trivial learning problems, and so $\fopt$ obtains risk less than $+\infty$. 
Next, since the particular choice $q_n = e^{-\eta \loss_{f^*}}$ yields $\KL(p; q_0 \pipes q_n) = 0$, we may always restrict to sequences for which we have $\KL(p; q_0 \pipes q_n) < \infty$ for all $n$. Hence, we indeed can take the sequence satisfying the finiteness requirement.

\paragraph{Log loss.}

First, we show for any $q_n$ that $\KL(p; q_0 \pipes q_n)$ is well-defined; its well-definedness is not immediately clear since each $q_n$ need not be a probablity density. For convenience, we introduce the notation that, for any $n$, the distribution $R_n$ satisfies $q_n = \E_{f \sim R_n} [ e^{-\eta \loss_f} ]$. Therefore, $-\log q_n = \mixloss{\eta}{R_n}$. 

Now, defining the pseudo-loss $\loss_p(Z) = -\log p(Z)$ corresponding to playing the pseudo-action $p$, our present goal is to show that
$\E \left[ \mixloss{\eta}{R_n} - \loss_P \right]$
is well-defined for each $j$. 
To this end, we make the following claim:
\begin{align} \label{eqn:not-that-negative}
\E_{Z \sim P} \left[ 
\left( \mixloss{\eta}{R_n}(Z) - \loss_p(Z) \right)^-
\right] > -\frac{1}{\eta} \log 2 .
\end{align}
To see the claim, define for $f \in \cF$ the excess loss $\loss_{f,p}(Z) := \loss_f(Z) - \loss_p(Z)$ and observe that (we simplify by writing $R$ instead of $R_n$)
\begin{align*}
&\E_{Z \sim P} \left[ 
\left( \mixloss{\eta}{R} - \loss_p \right)^-
\right] \\
&= \E_{Z \sim P} \left[ -\frac{1}{\eta} \log \E_{f \sim R} \left[ e^{-\eta \loss_{f,p(Z)}} \right] \cdot \ind{\E_{f \sim R} \left[ e^{-\eta \loss_{f,p}(Z)} \right] > e} 
\right] \\
&= \frac{1}{\eta} \E_{Z \sim P} \left[ 
         -\log \left( 
             \E_{f \sim R} \left[ e^{-\eta \loss_{f,p(Z)}} \right] 
             \cdot \ind{\E_{f \sim R} \left[ e^{-\eta \loss_{f,p}(Z)} \right] > e} 
             + \ind{\E_{f \sim R} \left[ e^{-\eta \loss_{f,p}(Z)} \right] \leq e} 
        \right) 
      \right] \\
&\geq -\frac{1}{\eta} \log 
               \E_{Z \sim P} \left[ 
                   \E_{f \sim R} \left[ e^{-\eta \loss_{f,p(Z)}} \right] 
                   \cdot \ind{\E_{f \sim R} \left[ e^{-\eta \loss_{f,p}(Z)} \right] > e} 
                   + \ind{\E_{f \sim R} \left[ e^{-\eta \loss_{f,p}(Z)} \right] \leq e} 
      \right] \\
&\geq -\frac{1}{\eta} \log 
               \E_{Z \sim P} \left[ 
                   \E_{f \sim R} \left[ e^{-\eta \loss_{f,p(Z)}} \right] 
                   \cdot \ind{\E_{f \sim R} \left[ e^{-\eta \loss_{f,p}(Z)} \right] > e} 
                   + 1
      \right] ,
\end{align*}
where Jensen's inequality was applied for the first inequality. 
It remains to show that 
\begin{align*}
\E_{Z \sim P} \left[ 
    \E_{f \sim R} \left[ e^{-\eta \loss_{f,p(Z)}} \right] 
    \cdot \ind{\E_{f \sim R} \left[ e^{-\eta \loss_{f,p}(Z)} \right] > e} 
\right] < \infty .
\end{align*}
Rewriting the LHS, we have
\begin{align*}
\E_{Z \sim P} \left[ 
    \E_{f \sim R} \left[ \left( \frac{p_f}{p} \right)^\eta \right] 
    \cdot \ind{\left( \frac{p_f}{p} \right)^\eta > e} 
\right] 
&\leq \E_{Z \sim P} \left[ 
            \E_{f \sim R} \left[ \left( \frac{p_f}{p} \right)^\eta \right] 
        \right] \\
&\leq \left( \E_{Z \sim P} \left[ 
            \E_{f \sim R} \left[ \frac{p_f}{p} \right] 
        \right] \right)^\eta \\
&= 1 ,
\end{align*}
where the inequality follows from $\eta \leq 1$, the concavity of the map $x \mapsto x^\eta$, and Jensen's inequality. The claim thus follows.

Now that we have shown that $\KL(p; q_0 \pipes q_n)$ is well-defined for all $n$, we also conclude from assumption \eqref{eqn:f-star-finite} that we may always take a sequence such that $\KL(p; q_0 \pipes q_n) < \infty$ for all $n$. 
Moreover, from \eqref{eqn:not-that-negative}, this can be strengthened to 
$\KL(p; q_0 \pipes q_n) \in [ -\eta^{-1} \log 2, \infty)$, and so this quantity is finite as desired.

In the remainder of this section, the two cases of loss assumptions are treated simulataneously (recall that $q_0$ is defined differently for each).

\subsubsection*{Step 1: Existence of minimizer $\bar{q}_n$ in convex hull of finite sequence}

Let $(q_n)_{n \geq 1}$ be a sequence in $\mathcal{C}$ for which $\KL(p; q_0 \pipes q_n) \rightarrow \inf_{q \in \mathcal{C}} \KL(p; q_0 \pipes q)$. 
From the argument above we may restrict the sequence to one for which $\KL(p; q_0 \pipes q_n)$ is finite for all $n$. 
Take $\mathcal{C}_n$ to be $\convhull( \{q_1, \ldots, q_n\} )$.

We introduce the representation $D(t): \Delta^{n-1} \rightarrow \reals_+$, where $D(t) = \KL(p; q_0 \pipes q_t)$ with $q_t = \sum_{j=1}^n t_j q_j$.

The first claim is that $t \mapsto D(t)$ is a continuous function. Li's Lemma 4.2 proves continuity of $D$ when $q_0 = p$, $\KL(p \pipes q_i) < \infty$ for $i \in [n]$ and each $q_i$ is a probability distribution. However, inspection of the proof reveals that the result still holds for general $q_0$ and when both $q_0$ and $q_i$ are only pseudoprobability densities, as long as we still have $\KL(p; q_0 \pipes q_i) < \infty$ for $i \in [n]$. But we already have established the latter requirement, and so $D$ is indeed continuous. Since $D$ also has compact domain, it follows that $D$ is globally minimized by an element in $\mathcal{C}_n$. Call this element $\bar{q}_n$.

\subsubsection*{Step 2: Beneficial properties of minimizer $\bar{q}_n$}

We claim for all $q \in \mathcal{C}_n$ that $\int p \frac{q}{\bar{q}_n} \leq 1$. This follows from a suitably adapted version of Li's Lemma 4.1. First, we observe that even though Li's Lemma 4.1 is for the case of the KL divergence $\KL(p \pipes q) = \int p \log \frac{p}{q}$, changing the $\log p$ term to $\log q_0$ has no effect on the proof. Therefore, this result also works for $\KL(p; q_0 \pipes q)$. Next, the proof works without modification even when its $q^*$ and $q$ are only pseudoprobability densities. To apply Li's Lemma 4.1, \emph{mutatis mutandis}, we instantiate its $\mathcal{C}$ as $\mathcal{C}_n$, its $p$ as $p$, its $q$ as $q$, and its $q^*$ as $\bar{q}_n$.

\subsubsection*{Step 3: $(\log \bar{q}_n)_n$ is Cauchy sequence in $L_1(P)$}

We can find a sequence $(\bar{q}_n)_{n \geq 1}$ such that $\{\KL(p; q_0 \pipes \bar{q}_n)\}$ both is non-increasing and converges to $\inf_{q \in \mathcal{C}} \KL(p \pipes q)$.

Next, let $n \leq m$ throughout the rest of this step and observe that
\begin{align*}
\KL(p; q_0 \pipes \bar{q}_n) - \KL(p; q_0 \pipes \bar{q}_m) 
= \int p \log \frac{p}{\frac{p \bar{q}_n}{\bar{q}_m} / c_{m,n}} + \log \frac{1}{c_{m,n}}
\end{align*}
with $c_{m,n} := \int \frac{p \bar{q}_n}{\bar{q}_m}$.

Now, due to the normalization by $c_{m,n}$ the first term on the RHS is a KL divergence and hence nonnegative. Also, since $c_{m,n} \leq 1$, the second term also is nonnegative.

Next, observe that $\KL(p; q_0 \pipes \bar{q}_n) - \KL(p; q_0 \pipes \bar{q}_m) \rightarrow 0$ as $n,m \rightarrow \infty$, and so we have
\begin{align*}
\int p \log \frac{p}{\frac{p \bar{q}_n}{\bar{q}_m} / c_{m,n}}
= \KL\left( p \pipes \frac{p \bar{q}_n}{\bar{q}_m} / c_{m,n} \right) \rightarrow 0
\end{align*}
as well as
\begin{align*}
\log \frac{1}{c_{m,n}} \rightarrow 0 
\quad \Rightarrow \quad 
c_{m,n} \rightarrow 1 .
\end{align*}

Next, we apply the following inequality due to Barron/Pinsker, holding for any probability distributions $p_1$ and $p_2$:
\begin{align*}
\int p_1 | \log(p_1) - \log(p_2) | \leq \KL(p_1 \pipes p_2)  + \sqrt{2 \KL(p_1 \pipes p_2)} .
\end{align*}
This yields
\begin{align*}
\int p \left| \log \frac{p}{\frac{p \bar{q}_n}{\bar{q}_m} / c_{m,n}} \right| \rightarrow 0.
\end{align*}
Since $c_{m,n} \rightarrow 1$, it therefore follows that
\begin{align*}
\int p | \log(\bar{q}_n) - \log(\bar{q}_m) | \rightarrow 0 .
\end{align*}

Therefore $(\log(\bar{q}_n))_{n \geq 1}$ is a Cauchy sequence in $L_1(P)$, and from the completeness of this space, $\log(\bar{q}_n)$ converges to some $\log(q^*) \in L_1(P)$.

Finally, we observe that $\KL(p; q_0 \pipes q^*) = \lim_{n \rightarrow \infty} \KL(p; q_0 \pipes \bar{q}_n)$ since
\begin{align*}
\KL(p; q_0 \pipes q^*) - \lim_{n \rightarrow \infty} \KL(p; q_0 \pipes \bar{q}_n)
&= \lim_{n \rightarrow \infty} \int p (\log \bar{q}_n - \log q^*) \\
&\leq \lim_{n \rightarrow \infty} \int p |\log \bar{q}_n - \log q^*| \\
&= 0 .
\end{align*}

\section{Definitions and conventions concerning $\infty$ and $- \infty$} 
\label{app:infinity-new}

For general losses we allow the loss to take on the value $\infty$, 
and for density estimation under log loss we allow the loss to take on the value $\infty$ and to be unbounded from below; see Appendix~\ref{sec:infinities-log-loss} for a full description of our assumptions in this latter setting. 
We thus need to take care to avoid ambiguous expressions such as $\infty - \infty$; here we follow the approach of \cite{grunwald2004game}. We generally permit operations on
the extended real line $[- \infty,\infty]$, with definitions and
exceptions as in \cite[Section 4]{rockafellar1970convex}. For a given
distribution $P$ on some space $\cU$ with associated $\sigma$-algebra,
we define the {\em extended random variable\/} $U$ as any measurable
function $U: \cU \rightarrow \reals \cup \{-\infty, \infty\}$. We
say that $U$ is {\em well-defined\/} if either $P(U = \infty) = 0$
or $P(U = - \infty) = 0$. Now let $U$ be a well-defined extended random variable.
For any function $f: [ - \infty, \infty ] \rightarrow  [ -
\infty,\infty]$, we say that $f(U)$ is well-defined if either  $P(f(U) = \infty) = 0$
or $P(f(U) = - \infty) = 0$ and we
abbreviate the expectation $\Exp_{U \sim P} [f(U)]$ to $\Exp[f]$,
hence we think of $f$ as an extended random variable itself.  If $f$ is bounded
from below and above $\Exp[f]$ is defined in the usual manner.
Otherwise we interpret $\Exp[f]$ as $\Exp[f^+] + \Exp[f^-]$ where
$f^+(u) := \max \{f(u),0 \}$ and $f^-(u) := \min \{f(u),0\}$, allowing
either $\Exp[f^+]= \infty$ or $\Exp[f^-] = - \infty$, but not both.
In the first case, we say that $\Exp[f]$ is well-defined; in the
latter case, $\Exp[f]$ is undefined. In the remainder of this section
we introduce conditions under which all extended random variables and
all expectations occurring in the main text are always well-defined.

The quantities which we need to show to be well-defined, both in the case of general losses and log loss, are (i) the risk for deterministic estimators; (ii) the risk for randomized estimators; (iii) the excess risk for either deterministic or randomized estimators; 
and (iv) certain ESIs and posterior expectations of annealed expectations. The GRIP is handled separately in Appendix~\ref{app:grip}.

\subsection{When the loss is uniformly bounded from below (general losses)}
\label{sec:infinities-ulb}

Here, we show that the relevant expressions are well-defined when the loss is uniformly bounded from below.

\subsubsection*{Risk for deterministic/randomized estimators and relevant comparators}

We first show that the risk of any deterministic estimator is well-defined. Our assumption that the loss is uniformly bounded from below is equivalent to the existence of a finite constant $\|\loss_-\|_\infty$ for which
\begin{align}
\inf_{f \in \cF} \inf_{z \in \cZ} \loss_f(Z) \geq -\|\loss_-\|_\infty . \label{eqn:bounded-below}
\end{align}
We thus have for any $f \in \cF$  that $\E_{Z \sim P}[ (\loss_f(Z))^{-}] > -  \infty$, and so the risk $\E_{Z \sim P} [ \loss_f(Z) ]$ is well-defined. Moreover, since $\inf_{f \in \cF} \E [ \loss_f(Z) ] > -\infty$, we also have that for any distribution $\dol$ on $\cF$ that $\E_{f \sim \dol} \left[ \E_{Z \sim P} [ \loss_f(Z) ] \right]$ is well-defined.

For all comparators $\tilde{f}$ used in this paper, assumption \eqref{eqn:bounded-below} also implies that
\begin{align*}
\inf_{z \in \cZ} \loss_{\tilde{f}}(Z) > -\infty .
\end{align*}
To see this, observe that the only comparators we use from the set $\allactionset \Setminus \cF$ are GRIPs (which for a given $z \in \cZ$ cannot obtain loss lower than $\inf_{f \in \cF} \loss_f(z)$) and versions of the loss of a GRIP or some $f \in \cF$ that are shifted by a finite constant. Thus, the risk is well-defined for all comparators used in this paper.

\subsubsection*{Excess risk for randomized estimators}

Next, the excess risk of any randomized estimator relative to a non-trivial comparator also is well-defined, since, by definition of a non-trivial comparator $\tilde{f}$ and the uniformly-bounded-below assumption, 
we have $-\infty < \E_{Z \sim P} \left[ \loss_{\tilde{f}}(Z) \right] < \infty$.

\subsubsection*{ESI / Posterior-expectation of annealed expectations}

Finally, we verify that all ESIs and annealed expectations of excess losses also are well-defined. The relevant quantities are (for all non-trivial comparators $\tilde{f}$)
\begin{align}
\E_{Z \sim P} \left[ e^{\eta \left( \loss_{\tilde{f}}(Z) - \loss_f(Z) \right)} \right] \quad \text{for all } f \in \cF \label{eqn:well-defined-esi}
\end{align}
and
\begin{align}
\E_{f \sim Q} \left[ -\frac{1}{\eta} \log \E_{Z \sim P} \left[ e^{\eta \left( \loss_{\tilde{f}}(Z) - \loss_f(Z) \right)} \right] \right] \quad \text{for all } Q \in \Delta(\cF) . \label{eqn:well-defined-annealed}
\end{align}

A potential issue with the ESI \eqref{eqn:well-defined-esi} being well-defined is that we can have both $\loss_{\tilde{f}}(z) = +\infty$ and $\loss_f(z) = +\infty$ for all $z$ in some set $A \subset \cZ$ of $P$-measure zero. To show that the expectation is well-defined, we define for $j = 1, 2, \ldots$ the random variable
\begin{align*}
g_j(Z) = \exp \left( \eta \left( \bigl(j \opmin \loss_{\tilde{f}}(Z) \bigr) - \loss_f(Z) \right) \right) .
\end{align*}
Now, for each $j = 1, 2, \ldots$, the expectation $\E [ g_j(Z) ]$ is
well-defined. Moreover, letting $A$ be precisely the subset of $\cZ$
for which $\loss_{\tilde{f}}(z) = +\infty$, it holds that $\{g_j\}$
converges to $\exp \left( \eta(\loss_{\tilde{f}} - \loss_f) \right)$
pointwise on $\cZ \Setminus A$. Hence, from Levi's
monotone convergence theorem, $\E_{Z \sim P} \left[ e^{\eta \left(
      \loss_{\tilde{f}}(Z) - \loss_f(Z) \right)} \right]$ is
well-defined.

Next, we show that annealed expectations of the form \eqref{eqn:well-defined-annealed} also are well-defined. From H\"older's inequality,
\begin{align*}
\E \left[ e^{\eta (\loss_{\tilde{f}}(Z) - \loss_f(Z))} \right] 
&= \E \left[ e^{\eta \loss_{\tilde{f}}(Z)} e^{-\eta \loss_f(Z)} \right] \\
&\leq e^{\|\loss_-\|_\infty} \E \left[ e^{\eta \loss_{\tilde{f}}(Z)} \right] \\
&< \infty ,
\end{align*}
where the final inequality follows because $\loss_{\tilde{f}}(Z) < \infty$ with probability 1. 
Therefore, the negative logarithm of the above is lower bounded by a finite negative constant that is independent of $f \in \cF$. It follows that \eqref{eqn:well-defined-annealed} is well-defined.

\subsection{Log loss}
\label{sec:infinities-log-loss}
In the common case of log loss with uncountable sample spaces, the
loss is not always uniformly bounded from below; see
Example~\ref{ex:densagain} below for a concrete illustration.  To
allow for this case while avoiding issues with infinities we need to
make the alternative assumptions of Section~\ref{sec:extended-intro},
which we now discuss. Recall that we assumed for all
$f \in \cF$ that $p_f$ is absolutely continuous with respect to a
common dominating measure $\mu$, and that furthermore we have (\ref{eqn:f-star-finite}) and 
(\ref{eqn:entropy-bounded-below}). 
To
motivate these assumptions, observe that $H(P)$ is the Bayes risk with
respect to all possible probability measures, whereas $\KL(P \pipes
P_{\fopt})$ is the approximation error due to playing the optimal
in-model predictor $\fopt$ rather than $P$. Now,
\eqref{eqn:f-star-finite} is a reasonable requirement, as it simply
means that the approximation error is finite; this is discussed
further in Example~\ref{ex:densagain}. 
Now, if we have $H(P) = -\infty$, then in light of
\eqref{eqn:f-star-finite}, we would also have to have 
$\E_{Z \sim P} \left[ \loss_{\fopt}(Z) \right] = -\infty$, 
which would imply that for any $f \in \cF$ with $\E[\loss_f] \neq \E [\loss_{\fopt}]$, 
the excess risk is infinite; this would make learning meaningless. 
We thus\footnote{A referee asked the natural question why we do not simply impose the more standard condition that $P \ll P_f$ for all $f \in \cF$, thus avoiding use of differential entropy. But this is not sufficient, as explained below \eqref{eq:refereeasks}.}

\subsubsection*{Risk for deterministic estimators}

Because for log loss we do not assume that losses are bounded from below, we need to ensure that the risk is well-defined.

We do this in two steps. First, we show that $\KL(P \pipes Q)$ is well-defined for any probability distribution $Q$ with density $q$ (with respect to $\mu$). 
We do this by showing that $\E \left[ \left( \log \frac{p}{q} \right)^- \right] > -\infty$:
\begin{align*}
\E[\ind{q/p > 1}(-\log q + \log p)] 
&= \E[-\log (\ind{q/p > 1} \cdot (q/p) + \ind{q/p \leq 1} \cdot 1) ] \\
&\geq - \log \E [ \ind{q/p > 1} \cdot (q/p) + \ind{q/p \leq 1} \cdot 1 ] \\
&\geq - \log 2 ,
\end{align*}
where the application of Jensen's inequality for the first inequality is legitimate because the expectation is of a nonpositive quantity. The above holds in particular for $q$ set to any $p_f$ (for $f \in \cF$). 
Next, we use the decomposition
\begin{align}\label{eq:refereeasks}
\E [ \loss_f ] = \E [ -\log p_f + \log p - \log p ] = \KL(P \pipes Q) + H(P) .
\end{align}
Since the KL divergence term is nonnegative and $H(P) < -\infty$ (recall assumption \eqref{eqn:entropy-bounded-below}), the above is well-defined.

We note that it is not sufficient to replace
(\ref{eqn:entropy-bounded-below}) by the standard requirement that
$P \ll P_f$ for all $f \in \cF$, for then (\ref{eq:refereeasks}) may
become undefined. To see this, note that, for two probability measures
$P$ and $R$, we may have $\KL(P \pipes R)= \infty$ even if $P \ll R$
(take, for example, $P$ a distribution on $\naturals$ with mass
function $p(i) \propto i^{-1 - \alpha}$ for $0 < \alpha \leq 1$ and
$R$ with mass function $r(i) = 2^{-i}$). Since $H(P)$ defined relative
to base measure $R$ is equal to $- \KL(P \pipes R)$ we may in general
also have $H(P) = -\infty$ even if $P$ has a density relative to
$R$. Thus, without the requirement \eqref{eqn:entropy-bounded-below} we could
have $\KL(P \pipes Q) + H(P) = \infty - \infty$ which is undefined.

\subsubsection*{Risk for randomized estimators}

The above argument can be trivially modified (adding an outer expectation over $f \sim \dol$ everywhere) to show that the risk of any randomized estimator $\dol$ is also well-defined.

\subsubsection*{Excess risk with respect to randomized estimators}

Finally, because we only consider situations in this paper for which the GRIP obtains risk less than positive infinity, the excess risk of any $\dol$ with respect to the GRIP is well-defined; the same is true for the excess risk with respect to the comparator $\fopt$, since we only consider situations where the risk of $\fopt$ is close to the risk of the GRIP.

\subsubsection*{ESI / Posterior-expectation of annealed expectations}

Finally, we verify that all ESIs and annealed expectations of excess losses also are well-defined. The relevant quantities are (for all non-trivial comparators $\tilde{f}$)
\begin{align}
\E_{Z \sim P} \left[ e^{\eta \left( \loss_{\tilde{f}}(Z) - \loss_f(Z) \right)} \right] \quad \text{for all } f \in \cF \label{eqn:well-defined-esi-ii}
\end{align}
and, taking the comparator to be the GRIP $\grip{\eta}$ as this is all that we require for annealed expectations in this paper,
\begin{align}
\E_{f \sim Q} \left[ -\frac{1}{\eta} \log \E_{Z \sim P} \left[ e^{\eta \left( \grip{\eta}(Z) - \loss_f(Z) \right)} \right] \right] \quad \text{for all } Q \in \Delta(\cF) . \label{eqn:well-defined-annealed-ii}
\end{align}

A potential issue with the ESI \eqref{eqn:well-defined-esi-ii} being well-defined is that we can have $\loss_{\tilde{f}}(z) = \loss_f(z) = +\infty$ or $\loss_{\tilde{f}}(z) = \loss_f(z) = -\infty$ for all $z$ in some set $A \subset \cZ$ of $P$-measure zero. To show that the expectation is well-defined, we define for $j = 1, 2, \ldots$ the random variable
\begin{align*}
g_j(Z) = \exp \left( \eta \left( \bigl[j \opmin \loss_{\tilde{f}}(Z) \bigr] - \bigl[ (-j) \opmax \loss_f(Z) \bigr] \right) \right) .
\end{align*}
Now, for each $j = 1, 2, \ldots$, the expectation $\E [ g_j(Z) ]$ is well-defined. Moreover, letting $A$ be precisely the subset of $\cZ$ for which either $\loss_{\tilde{f}}(z) = +\infty$ or $\loss_f(z) = -\infty$, it holds that $\{g_j\}$ converges to $\exp \left( \eta(\loss_{\tilde{f}} - \loss_f) \right)$ pointwise on $\cZ \Setminus A$. Hence, from Beppo Levi's monotone convergence theorem, $\E_{Z \sim P} \left[ e^{\eta \left( \loss_{\tilde{f}}(Z) - \loss_f(Z) \right)} \right]$ is well-defined.

Finally, we verify that \eqref{eqn:well-defined-annealed-ii} is well-defined. Indeed, it is well-defined as a trivial consequence of $\E_{Z \sim P} \left[ e^{\eta \left( \grip{\eta}(Z) - \loss_f(Z) \right)} \right] \leq 1$ which holds by virtue of the comparator being the GRIP.

\begin{myexample}{Density Estimation} \label{ex:densagain} Consider
  the Gaussian scale family with $\cZ = \reals$ and $\{ p_f \mid f \in
  \cF \}$ where $\cF = \reals^+$ and $p_f(y) \propto \exp(-y^2 /2 f)$,
  i.e., $p_f$ is the density, relative to standard Lebesgue measure, of
  the normal distribution with mean $0$ and variance $\sigma^2 :=
  f$. Then under log loss we have $\loss_f(y) = \frac{y^2}{f} +
  \frac{1}{2} \log (\pi(f))$. Obviously, we do not want to rule out a
  model as standard like this, yet the loss is unbounded from below,
  which illustrates the need for treating log-loss separately from
  other loss functions.  The requirements (\ref{eqn:f-star-finite})
  and (\ref{eqn:entropy-bounded-below}) above do allow for this model,
  as long as the underlying distribution $P$ (a) has a density
  relative to Lebesgue measure (otherwise
  (\ref{eqn:entropy-bounded-below}) does not hold); (b) is not
  too-heavy tailed (it needs to have a second moment, otherwise
  (\ref{eqn:f-star-finite}) does not hold), and (c) is not excessively
  peaked at $0$ (for example, the probability distribution $P$ on
  $(0,1/\exp(1))$ with density $p(x) = 1/(x \cdot \log^2 x)$ has $H(P)
  = -\infty$, but distribution $P'$ with density $p'(x) = 3/(x \cdot
  \log^4 x)$ has finite $H(P')$. If one restricts the model to contain
  only $f \geq \sigma^2_0$ for some $\sigma^2_0 > 0$, then the log
  loss is bounded from below, and the requirements
  (\ref{eqn:f-star-finite}) and (\ref{eqn:entropy-bounded-below}) do
  not need to be imposed; in that situation, one could allow for an
  underlying distribution $P$ with a point mass at some outcome, so
  that $P$ does not have a density relative to Lebesgue measure and
  $D(P \| P_{f^*}) = \infty$, yet all our concepts remain well-defined. 
\end{myexample}

\section{Comparative examples}
\label{app:examples}

\begin{myexample}{Bernstein condition does not hold, bounded excess risk} \label{ex:no-bernstein-bounded} 
Consider regression with squared loss, so that $\cZ = \cX \times \cY$. Select $\Pt$ such that $X$ and $Y$ are independent. 
Let $X$ follow the law $\Pt$ such that $\Pt(X = 0) = \Pt(X = 1) = \frac{a}{2}$, for $a := 2 - \frac{\pi^2}{6} \in (0, 1)$, 
and, for $j = 2, 3, \ldots$, $\Pt(X = j) = \frac{1}{j^2}$. 
Let $Y = 0$ surely. 
Take as $\cF$ the countable class $\{f_1, f_2, \ldots\}$ such that $f_1(1) = 0.5$ and $f_1$ is identically 0 for all other values of $x \in \cX$; for each $j = 2, 3, \ldots$, the function $f_j$ is defined as $f_j(0) = 1$, $f_j(j) = j$, and $f_j$ takes the value 0 otherwise.

It follows that $\fopt = f_1$, and for every $j > 1$ we have
$\xsrisk{f_j} = \frac{3 a}{8} + 1$. Thus, the excess risk is bounded
for all $f_j$. The witness condition holds because for all $j > 1$ we
have $\Pr(\xsloss{f_j} = 1) = a$ and
$\E [ \xsloss{f_j} \cdot \ind{\xsloss{f_j} \leq 1} ] \geq \frac{3
  a}{8}$. Also, it is easy to verify that the strong central condition
holds with $\eta = 2$. On the other hand, the Bernstein condition
fails to hold in this example because
$\E [ \xsloss{f_j}^2 ] = a + j^2 \rightarrow \infty$ as
$j \rightarrow \infty$, while the excess risk is finite. In fact, even
the variance of the excess risk is unbounded as
$j \rightarrow \infty$, precluding the use of a weaker variance-based
Bernstein condition as in equation (5.3) of
\cite{koltchinskii2006local}.  Therefore, Theorem
\ref{thm:main-bounded-a} still applies while, e.g., the results of
\cite{zhang2006information} and \cite{audibert2009fast} do not (see
Section \ref{sec:related-work}).
\end{myexample}

\begin{myexample}{Bernstein condition does not hold, unbounded excess risk } \label{ex:no-bernstein-unbounded} 
The setup of this example was presented in Example 5.7 of \Citet{erven2015fast} and is reproduced here for convenience. For $f_\mu$ the univariate normal density with mean $\mu$ and variance 1, let $\mathcal{P}$ be the normal location family and let $\cF = \{f_\mu : \mu \in \reals\}$ be the set of densities of the distributions in $\mathcal{P}$. Then, since the model is well-specified, for any $P \in \mathcal{P}$ with density $f_\nu$ we have $\fopt = f_\nu$. 
As shown in \Citet{erven2015fast}, the Bernstein condition does not hold in this example, although we note that the weaker, variance-based Bernstein condition of \cite[equation (5.3)]{koltchinskii2006local} does hold.  However, we are not aware of any analyses that make use of the variance-based Bernstein condition in the unbounded excess losses regime. 

Since the model is well-specified, the strong central condition holds with $\eta = 1$. Next, we show that the witness condition holds with $M = 2$, $u = 4$, and $c = 1 - \sqrt{\frac{2}{\pi}}$. 
From location-invariance, we assume $\nu > \mu = 0$ without loss of generality.

First, observe that the excess risk is equal to 
$\xsrisk{f_\mu} = \frac{1}{2} \nu^2$.

As $M = 2 < \infty$, the witness condition has two cases: the case of excess risk at least 2 and the case of excess risk below 2. We begin with the first case, in which $\nu \geq 1$. Then the contribution to the excess risk from the upper tail is
\begin{align*}
& \E \left[ \xsloss{f_\mu} \cdot \ind{\xsloss{f_\mu} > u \xsrisk{f_\mu}} \right] 
= \E \left[ \left( -\frac{\nu^2}{2} + X \nu \right) \cdot \ind{-\frac{\nu^2}{2} + X \nu > u \frac{\nu^2}{2}} \right] \\
= & \E \left[ \left( -\frac{\nu^2}{2} + X \nu \right) \cdot \ind{X > \frac{u \nu}{2} + \frac{\nu}{2}} \right] \leq \nu \E \left[ X \cdot \ind{X > \frac{u \nu}{2}} \right] ,
\end{align*}
which is at most
\begin{align*}
\nu \E \left[ X \cdot \ind{X - \nu > \left( \frac{u}{2} - 1 \right) \nu} \right] 
&= \nu \int_0^\infty \Pr(X \cdot \ind{X - \nu > (\frac{u}{2} - 1 ) \nu} > t) dt \\
&\leq \nu \frac{1}{\sqrt{2 \pi}} \frac{e^{-(\frac{u}{2} - 1)^2 \nu^2 / 2}}{(\frac{u}{2} - 1) \nu} = \frac{1}{\sqrt{2 \pi}} \frac{e^{-(\frac{u}{2} - 1)^2 \nu^2 / 2}}{(\frac{u}{2} - 1)} .
\end{align*}

Since $u = 4$, the above is at most $\frac{1}{\sqrt{2 \pi}}$ and so, in this regime, the witness condition indeed is satisfied with $c = 1 - \sqrt{2 / \pi}$.

Consider now the case of $\nu < 1$. In this case, the threshold simplifies to the constant $u$ and the upper tail's contribution to the excess risk is
\begin{align*}
\E \left[ \xsloss{f_\mu} \cdot \ind{\xsloss{f_\mu} > u} \right] 
&= \E \left[ \left( -\frac{\nu^2}{2} + X \nu \right) \cdot \ind{-\frac{\nu^2}{2} + X \nu > u} \right] \\
&= \E \left[ \left( -\frac{\nu^2}{2} + X \nu \right) \cdot \ind{X > \frac{u}{\nu} + \frac{\nu}{2}} \right] \leq \nu \E \left[ X \cdot \ind{X > \frac{u}{\nu}} \right] ,
\end{align*}
which is at most
\begin{align*}
\nu \E \left[ X \cdot \ind{X - \nu > \frac{u}{\nu} - \nu} \right] 
&= \nu \int_0^\infty \Pr(X \cdot \ind{X - \nu > \frac{u}{\nu} - \nu} > t) dt \\
&\leq \nu \frac{1}{\sqrt{2 \pi}} \frac{e^{-(\frac{u}{\nu} - \nu)^2 / 2}}{\frac{u}{\nu} - \nu} = \nu^2 \frac{1}{\sqrt{2 \pi}} \frac{e^{-(\frac{u}{\nu} - \nu)^2 / 2}}{u - \nu^2} .
\end{align*}
Since $u = 4$ and $\nu < 1$, the above is at most $\frac{\nu^2}{\sqrt{18 \pi}}$, and so the value of $c$ from before still works and the witness condition holds in this regime as well.
\end{myexample}

\begin{myexample}{Small-ball assumption violated} \label{ex:no-small-ball}
To properly compare to the small-ball assumption of \cite{mendelson2014learning}, we consider regression with squared loss in the well-specified setting, so that the parameter estimation error bounds of \cite{mendelson2014learning} directly transfer to excess loss bounds for squared loss. Take $X$ and $Y$ be independent. The distribution of $X$ is defined as, for $j = 1, 2, \ldots$, $\Pt(X = j) = p_j := \frac{1}{a} \cdot \frac{1}{j^2}$ for $a = \frac{\pi^2}{6}$. Let the distribution of $Y$ be zero-mean Gaussian with unit variance. For the class $\cF$, we take the following countable class of indicator functions: for each $j = 0, 1, 2, \ldots$, define $f_j(i) = \ind{i = j}$, for any positive integer $i$. Since $f_0(x) = \E [ Y \mid X = x ] = 0$ for all $x \in \{1, 2, \ldots\}$, we have $\fopt = f_0$. 

The small-ball assumption fails in this setting, since, for any constant $\kappa > 0$ and for all $j = 1, 2, \ldots$:
\begin{align*}
\Pr \left( |f_j - \fopt| > \kappa \|f_j - \fopt\|_{L_2(\Pt)} \right) 
\,\, \leq \,\, \Pr \left( |f_j - \fopt| > 0 \right) 
\,= \, p_j 
= \frac{1}{a j^2} \rightarrow 0 \text{ as } j \rightarrow \infty .
\end{align*}

On the other hand, the strong central condition holds with $\eta = \frac{1}{2}$, since, for all $j = 1, 2, \ldots$ and all $x$:
\begin{align*}
\E \left[ e^{-\eta \xsloss{f_j}} \right] 
= \E \left[ \frac{e^{-\eta (f_j(x) - Y)^2}}{e^{-\eta Y^2}} \right] 
= \int \frac{\frac{1}{\sqrt{2 \pi \eta^{-1}}} e^{-\eta (f_j(x) - Y)^2}}
                   {\frac{1}{\sqrt{2 \pi \eta^{-1}}}{e^{-\eta Y^2}}}
            p(Y) dy 
\end{align*}
which is equal to 1 for $\eta = \frac{1}{2}$, since $Y \sim \mathcal{N}(0, 1)$.

It remains to check the witness condition. 
Observe that, for each $j$, we have $\xsrisk{f_j} = p_j$. 

Next, we study how much of the excess risk comes from the upper tail, above some threshold $u$:
\begin{align}
\E \left[ \xsloss{f_j} \cdot \ind{\xsloss{f_j} > u} \right] 
&= \E \left[ \left( f_j^2(X) - 2 f_j(X) Y \right) \cdot \ind{f_j^2(X) - 2 f_j(X) Y > u} \right] \nonumber \\
&= p_j \E \left[ \left( 1 - 2 Y \right) \cdot \ind{1 - 2 Y > u} \right] \nonumber \\
&= p_j \left( \Pr \left( Y < \frac{1 - u}{2} \right) - 2 \E \left[ Y \cdot \ind{Y < \frac{1 - u}{2}} \right] \right) . \label{eqn:no-sb-yes-witness}
\end{align}
Now, let $K := \frac{u - 1}{2}$. It is easy to show that
\begin{align*}
\Pr \left( Y > K \right) 
\leq \frac{1}{\sqrt{2 \pi}} \frac{e^{-K^2 / 2}}{K} .
\end{align*}
In addition, for $u \geq 3$ (and hence $K \geq 1$), we have
\begin{align*}
\E \left[ Y \cdot \ind{Y > K} \right] 
&= \int_0^\infty \Pr(Y \cdot \ind{Y > K} > t) dt = \int_K^\infty \Pr(Y > t) dt \\
&\leq \int_K^\infty \frac{1}{\sqrt{2 \pi}} \frac{e^{-t^2 / 2}}{t} dt 
\leq \int_K^\infty \frac{1}{\sqrt{2 \pi}} e^{-t^2 / 2} dt \leq \frac{1}{\sqrt{2 \pi}} \frac{e^{-K^2 / 2}}{K} dt .
\end{align*}
Thus, taking $u = 3$, we see that \eqref{eqn:no-sb-yes-witness} is at most $p_j \sqrt{\frac{2}{\pi}} e^{-1/2} \leq \frac{p_j}{2}$, the witness condition therefore holds, and so we may apply the first part of Theorem \ref{thm:main-bounded-a}.
\end{myexample}

\end{document}